\documentclass[manuscript,sigconf]{acmart}

\AtBeginDocument{%
  \providecommand\BibTeX{{%
    \normalfont B\kern-0.5em{\scshape i\kern-0.25em b}\kern-0.8em\TeX}}}

\copyrightyear{2022}
\acmYear{2022}
\setcopyright{licensedusgovmixed}\acmConference[FAccT '22]{2022 ACM Conference on Fairness, Accountability, and Transparency}{June 21--24, 2022}{Seoul, Republic of Korea}
\acmBooktitle{2022 ACM Conference on Fairness, Accountability, and Transparency (FAccT '22), June 21--24, 2022, Seoul, Republic of Korea}
\acmPrice{15.00}
\acmDOI{10.1145/3531146.3533236}
\acmISBN{978-1-4503-9352-2/22/06}

\newcommand{\cal}[1]{\mathcal{#1}}

\usepackage{hyperref}
\hypersetup{
    colorlinks=true,
    linkcolor=blue,
    citecolor=blue,
    urlcolor=blue,
    filecolor=black
    }

\usepackage{graphicx}
\DeclareGraphicsExtensions{.eps,.pdf,.jpeg,.png}
\usepackage{xcolor}

\usepackage{bm}
\usepackage{amsmath,amsfonts,amsthm,amssymb}
\usepackage{systeme,mathtools}
\usepackage{eurosym}
\theoremstyle{definition}

\newtheorem*{example*}{Example}

\usepackage{tikz}
\usepackage{wrapfig}
\usetikzlibrary{bayesnet}
\usetikzlibrary{fit,positioning,calc,backgrounds}

\usetikzlibrary{positioning,arrows.meta,quotes}
\usetikzlibrary{shapes,snakes}
\tikzset{>=latex}
\tikzstyle{plate caption} = [caption, node distance=0, inner sep=0pt,
below left=5pt and 0pt of #1.south]

\usepackage{enumitem}
\setlist[enumerate]{nosep}
\setlist[itemize]{nosep}

\usepackage[labelfont=bf]{caption}
\usepackage{subcaption}

\usepackage{comment}
\excludecomment{yaircomment}
\excludecomment{przemekcomment}

\newtheoremstyle{exampstyle}
    {2pt} 
  {2pt} 
  {} 
  {} 
  {\bfseries} 
  {.} 
  {.5em} 
  {} 
\theoremstyle{exampstyle} 

\theoremstyle{exampstyle}
\newtheorem{definition}{Definition}

\newtheorem{proposition}{Proposition} 

\graphicspath{{figs/}}

\renewcommand{\u}{u}
\newcommand{\y}{y}

\newcommand{\Y}{Y}
\newcommand{\x}{\bm{x}}
\newcommand{\X}{\bm{X}}
\newcommand{\w}{\bm{w}}
\newcommand{\W}{\bm{W}}
\newcommand{\z}{\bm{z}}
\newcommand{\Z}{\bm{Z}}

\newcommand{\E}{\mathbb{E}}
\DeclareMathOperator*{\V}{\mathbb{V}}
\DeclareMathOperator*{\p}{\textit{P}}

\newcommand{\overbar}[1]{\mkern 1.5mu\overline{\mkern-1.5mu#1\mkern-1.5mu}\mkern 1.5mu}

\renewcommand{\cal}[1]{\mathcal{#1}}

\newcommand{\mathintitle}[1]{\texorpdfstring{#1}{TEXT}}

\usepackage{bbm}
\begin{document}

\title{Marrying Fairness and Explainability in Supervised Learning}

\author{Przemyslaw Grabowicz}
\affiliation{%
  \institution{University of Massachusetts Amherst}
  \city{Amherst}
  \state{Massachusetts}
  \country{USA}}
\email{grabowicz@cs.umass.edu}

\author{Nicholas Perello}
\affiliation{%
  \institution{University of Massachusetts Amherst}
  \city{Amherst}
  \state{Massachusetts}
  \country{USA}}
\email{nperello@umass.edu}

\author{Aarshee Mishra}
\affiliation{%
  \institution{University of Massachusetts Amherst}
  \city{Amherst}
  \state{Massachusetts}
  \country{USA}}
\email{aarsheemishr@umass.edu}

\renewcommand{\shortauthors}{Grabowicz, et al.}

\begin{abstract}
  Machine learning algorithms that aid human decision-making may inadvertently discriminate against certain protected groups. Therefore, we formalize direct discrimination as a direct causal effect of the protected attributes on the decisions, while \textit{induced} discrimination as a change in the causal influence of non-protected features associated with the protected attributes. The measurements of marginal direct effect (MDE) and SHapley Additive exPlanations (SHAP) reveal that state-of-the-art fair learning methods can induce discrimination via association or reverse discrimination in synthetic and real-world datasets. To inhibit discrimination in algorithmic systems, we propose to nullify the influence of the protected attribute on the output of the system, while preserving the influence of remaining features. We introduce and study post-processing methods achieving such objectives, finding that they yield relatively high model accuracy, prevent direct discrimination, and diminishes various disparity measures, e.g., demographic disparity.
\end{abstract}



\maketitle


\section{Introduction}
Discrimination consists of treating somebody unfavorably because of their membership to a particular group, characterized by a \textit{protected attribute}, such as race or gender. 
Freedom from discrimination is outlined as a basic human right by the Universal Declaration of Human Rights.
In the legal~\cite{titlevii, fairhousing} and social science~\cite{Ture1968Black,Altman2016Discrimination,LippertRasmussen2012Badness} contexts, a key consideration serving as the basis for identifying discrimination is whether there is a disparate treatment or unjustified disparate impact on the members of some protected group.
To prevent disparate treatment, the law often forbids the use of certain protected attributes, $Z$, such as race or gender, in decision-making, e.g., in hiring. 
Thus, these decisions, $Y$, shall be based on a set of relevant attributes, $\X$, and should not depend on the protected attribute, $Z$, i.e., $\p(y|\x,z) = \p(y|\x,z')$ for any $z, z'$, ensuring that there is no \textit{disparate treatment}.\footnote{Throughout the manuscript we use a shorthand notation for probability: $\p(y|\x,z) \equiv \p(Y=y|\X=\x,Z=z)$, where $\X,Y,Z$ are random variables, $\x,y,z$ are their instances, and $\p$ is a probability distribution or density.}
We refer to this kind of discrimination as \textit{direct discrimination} (or lack of thereof), because of the direct use of the protected attribute~$Z$.

Historically, the prohibition of direct discrimination was sometimes circumvented by the use of variables correlated with the protected attribute as proxies. 
For instance, some banks systematically denied loans and services, intentionally or unintentionally, to certain racial groups based on the areas they lived in~\cite{Zenou2000Racial,Hernandez2009Redlining}, which is known as the phenomenon of ``redlining''.
In order to prevent such \textit{inducement of discrimination}, the legal system of the United States has established that the impact of a decision-making process should be the same across groups differing in protected attributes~\cite{LippertRasmussen2012Badness,Altman2016Discrimination}, that is $\p(y|z) = \p(y|z')$, unless there is a ``justified reason'' or ``business necessity'' for this \textit{disparate impact}~\cite{titlevii, fairhousing}. 

\begin{figure}[t]
\centering
 \centering
    \scalebox{0.7}{
    \begin{tikzpicture}
    \node[obs]                   		(x1b)      {$x_1$} ; %
    \node[obs, right=of x1b]        		(x2b)      {$x_2$} ; %
    \node[obs, right=of x2b]       		(zb)      {$z$} ; %
    \node[obs, below=of x2b]	(yb)      {$y$} ; %
    
    \draw [->] (3.5,2) -- (2.5,1);
    
    \node[obs, right=of zb, xshift=1.5cm]     	(x1a)      {$x_1$} ; %
    \node[obs, right=of x1a]        	(x2a)      {$x_2$} ; %
    \node[obs, right=of x2a]       		(za)      {$z$} ; %
    \node[obs, below=of x2a]	(ya)      {$y$} ; %
    
    \draw [->] (6.5,2) -- (7.5,1);
    
    \node[obs, above=of $(zb)!0.5!(x1a)$, yshift=0.5cm]	(y)      {$y$} ; %
    \node[obs, above=of y]        		(x2)      {$x_2$} ; %
    \node[obs, left=of x2]        		(x1)      {$x_1$} ; %
    \node[obs, right=of x2]       		(z)      {$z$} ; %
    
    \edge[-,dashed] {x2} {z} ; %
    \edge[-,dashed] {x2a} {za} ; %
    \edge[-,dashed] {x2b} {zb} ; %
    \edge {x1,z} {y} ; %
    \edge {x1b,x2b} {yb} ; %
    \edge {x1a} {ya} ; %
    
    \node[draw] at (5.1,4.5) {discriminatory process generating training data};
    \node[draw] at (1.8,-2.7) {standard learning dropping $Z$};
    \node[draw] at (8.5,-2.7) {infl. of $Z$ removed, infl. of $X_1$ preserved};
    
    \end{tikzpicture}
    }
 \caption{An illustration of the graphical models that result from applying different learning methods to the example scenario: standard learning dropping $\Z$ (bottom left), the removal of influence of $\Z$ while preserving the influence of $\X$ (bottom right). The directed edges correspond to causal relations, while the dashed edge to a potentially unknown relationship, e.g., a non-causal association.}
\label{fig:graphs}
\end{figure}
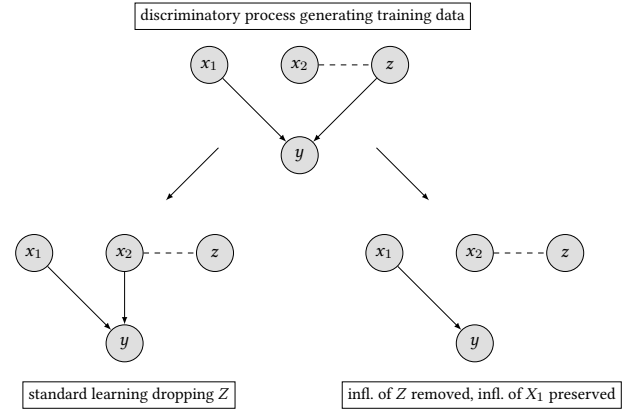
\textbf{Example.}
The following example runs through the manuscript.
Consider a hypothetical linear model of loan interest rate, $\Y$. Prior works suggest that interest rates differ by race, $Z$~\cite{Turner1999Mortgage, Bartlett2019Consumer}. Some loan-granting clerks may produce non-discriminatory decisions, $y=\beta_0-x_1$, while other clerks may discriminate directly, $y_\text{dir}=\beta_0-x_1-z$ (see the graphical model in the top of Figure~\ref{fig:graphs}), where $\beta_0$ is a fixed base interest rate, $x_1$ is a relative salary of a loan applicant, $x_2$ is an encoding of the zip code (positive for wealth neighbourhoods, negative otherwise), while $z$ encodes race and takes some positive (negative) value for White (non-White) applicants. 
If the protected attribute is not available (e.g., loan applications are submitted online), then a discriminating clerk may induce discrimination in the interest rate, by using a proxy for race, $y_\text{ind}=\beta_0-x_1-x_2$, where $x_2$ is the proxy. This case corresponds to the aforementioned real-world phenomenon of redlining.
%
If we trained a model on the dataset $D = \{( x_1, x_2, z, y_\text{dir} )\}$ without using the protected attribute, since it is prohibited by law, then we would induce indirect discrimination in the same way as redlining.
To see this point, assume for simplicity that all variables have a zero mean and 
there's no correlation between $X_1$ and $Z$ and a positive correlation, $r>0$, between $X_2$ and $Z$. If we applied standard supervised learning under the quadratic loss, then we would learn the model $\hat{y}_1 = \beta_0 - x_1 -z$. If we dropped the protected attribute, $Z$, before regressing $Y_\text{dir}$ on the attributes $X_1$ and $X_2$, then we would learn the model $\hat{y}_2 = \beta_0 - x_1 - r x_2$, that induces via $X_2$ indirect discrimination growing in proportion to $r$.

\textbf{Interdisciplinary challenge.}
There is a substantial and quickly growing literature on fairness in machine learning. However, its connection to the legal literature is underdeveloped, e.g., legal background is missing in the first textbook on fair machine learning (as of May 2022)~\cite{solonbarocas2019fairness}, and business necessity is often neglected, which may be slowing down the widespread adoption of fair machine learning methods~\cite{Madaio2021Assesing}.
In supervised learning, potentially any feature that improves model predictiveness on deployment could be claimed to fulfil a business necessity. 
However, how does one prevent such features from being used for unintentional inducement of discrimination?
This is a particularly acute problem for data-rich machine learning systems, since they often can find surprisingly accurate surrogates for protected attributes when a large enough set of legitimate-looking variables is available, resulting in discrimination via association~\cite{Wachter2019Affinity}. 
Causality-based research offers so-called \textit{path-specific counterfactual fairness} that enables designation of fair paths for business necessity~\cite{Nabi2019Learning, Chiappa2019Path,Wu2019PC}, but these approaches rely on causal assumptions, arbitrary reference interventions, achieve sub-optimal model accuracy, and do not formally prevent induced discrimination via fair paths.
Our study brings the concepts inspired by legal systems to supervised learning, which necessitates less assumptions and is used more widely than causal discovery, e.g., we make no assumptions about the relationship between $X_2$ and $Z$ (dashed line in Figure~\ref{fig:graphs}).
The big challenge in introducing non-discriminatory supervised learning algorithms is preventing direct discrimination without inducing indirect discrimination while enabling the necessity of businesses to maximizing model accuracy.

\textbf{Contributions.}
To the best of our knowledge, this is the first study that fills the gap between fair supervised learning and legal systems by bridging causal notions of fairness with the literature on explainable artificial intelligence. 
%
We propose methods for removing direct discrimination from models that allows a limited use of features that prevents their use as a proxy for the protected attribute (the bottom right part of Figure~\ref{fig:graphs}).
Specifically, first we define the concepts of direct, indirect, and induced discrimination via the measures of causal influence and tie them to legal instruments.
While doing so, we point out that induced discrimination can happen both for causal models of real-world decision-making processes and any other models that approximate such processes.
Second, we construct loss functions that aim to remove the influence of the protected attribute, $\Z$, while maintaining the influence of the remaining features, $\X$, using a novel measure of marginal direct effect (MDE) and a well-known input influence measure (SHAP).
Third, we show that dropping the protected attribute before training in standard supervised learning would result in increased influence 
of features associated with the protected attribute. 
Fourth, we introduce marginal interventional mixture models that drop $\Z$ while minimizing the inducement of discrimination through $\X$. We show that this method keeps influence of $\X$ and~$\Z$ close to the target values and, in addition, decreases popular disparity measures, while keeping high model accuracy.
Our methods are released publicly via an easy-to-use \texttt{FaX-AI} Python library (\url{https://github.com/social-info-lab/FaX-AI}).

\section{Related works}
In machine learning, discrimination is typically defined based on statistical independence~\cite{Pedreshi2008Discrimination, Feldman2014Certifying, Zafar2015Fairness,  Zafar2017Fairnessa, Zafar2017Fairness, Hardt2016Equality, Zafar2017Parity, Woodworth2017Learning, Pleiss2017Fairness, Donini2018Empirical, Aswani2019Optimization, Oneto2020General} or causal relations~\cite{Kilbertus2017Avoiding, Kusner2017Counterfactual, Zhang2018Fairness, Salimi2019Capuchin}.
Well-known fairness objectives, such as parity of impact and equalized odds, correspond or are related to the statistical independence between $\Z$ and $\Y$~\cite{Aswani2019Optimization}.
However, legal systems allow for exceptions from this independence through the business necessity clause, which permits usage of an attribute $\X$ associated with $\Z$ and results in the decisions $\Y$ depending on $\Z$ through $\X$ if it fulfils certain business necessity. 
Hence, the notions of discrimination based on the statistical independence between $Y$ and $Z$ are misaligned with their legal counterparts~\cite{Lipton2019Troubling}, which results in shortcomings. For instance, the algorithms that put constraints on the aforementioned disparities in treatment and impact~\cite{Pedreshi2008Discrimination,Feldman2014Certifying,Zafar2015Fairness} could negatively affect females with short hair and/or programming skills, because of those features' (fair or unfair) association with males~\cite{Lipton2018Does}.

A relevant line of research proposes to define direct and indirect discrimination as direct and indirect causal influence of $\Z$ on~$\Y$, respectively~\cite{Zhang2017causal,Zhang2018Fairness}. 
While this notion of direct discrimination is consistent with the concept of disparate treatment in legal systems, the corresponding indirect discrimination is not consistent with them, since the business necessity clause allows the use of an attribute that depends on the protected feature (causally or otherwise), if the attribute is judged relevant to the decisions made. For instance, the majority's view in the Supreme Court case of Ricci v.~DeStefano~\cite{ricci} argued that the defendants could not argue that the disputed promotion examinations results were inconsistent with business necessity.
\textit{Path-specific} notions of causal fairness address this issue to a limited extent~\cite{Nabi2019Learning, Chiappa2019Path,Wu2019PC}.
These methods introduce \textit{fair causal paths}, i.e., the paths through which the impact of the protected attribute is permitted, hence enabling business necessity. 
However, if there is no limit on the influence that can pass through such a path, then the path can be used for discrimination, as in the aforementioned case of \textit{redlining}. 
This limit is not a focus of prior works~\cite{Kilbertus2017Avoiding, Kusner2017Counterfactual, Zhang2018Fairness, Salimi2019Capuchin, Nabi2019Learning, Chiappa2019Path,Wu2019PC}, but it is crucial to prevent induced discrimination in machine learning.
In addition, for the removal of protected attributes these works rely on causal assumptions and a reference intervention, which is a standard technique in causality literature, but the reference intervention is arbitrary and may decrease model accuracy, as we show in Section~\ref{sec:vs-chiappa}.
To the best of our knowledge, this work is the first to define and inhibit induced discrimination in supervised learning on the grounds of causality and explainability research. 

\section{Problem formulation of fair and explainable learning}
\label{sec:fair}

Consider decisions $\Y$ that are outcomes of a process acting on non-protected variables $\X$ and protected variables $\Z$, where $\x\in \mathcal{X}$, $\z\in \mathcal{Z}$, $\y\in \mathcal{Y}$, i.e., the variables can take values from any set, e.g., binary or real.
Protected and non-protected features are indexed, e.g., $X_i$ corresponds to the $i$'th feature (component).
The decisions are generated via a function $\y=\y(\x,\z, \epsilon)$, where $\epsilon$ is an exogenous noise variable. Since the exogenous noise is unpredictable, we focus on the de-noised function $\y(\x,\z)=\E_{\epsilon} y(\x,\z,\epsilon)$ for notational simplicity.
The process generating decisions $\Y$ corresponds either to a real-world causal mechanism or its model,
while the inducement of indirect discrimination shall be prevented on legal grounds in either case (see Subsection~\ref{sec:legal-responsibility}). 
These decisions can represent any decision-making process, e.g.: i)~estimates of recidivism risk for a crime suspect, given some information about their prior offenses $\x$ and their race~$\z$, or ii)~credit score assignments for a customer, given their financial record $\x$ and their gender $z$.

The goal of standard supervised learning is to obtain a function $\hat{\y}: \mathcal{X} \times \mathcal{Z} \to \mathcal{Y}$ that minimizes an expected loss, e.g., $\E[\ell( \Y, \hat{\y}(\X,\Z) )]$, where the expectation is over the set of training samples $(\x,\z,\y)$ and $\ell$ is a loss function such as quadratic loss, $\ell(\y,\hat{y})=(\y-\hat{y})^2$.
If the dataset is tainted by discrimination, then a data science practitioner may desire, and, in principle, be obliged by law, to apply an algorithm that does not perpetuate this discrimination.
For example, $\Y$ could correspond to past hiring decisions, which we now want to automate with model $\hat{\Y}$. If historical discrimination in hiring took place, then $\Y$ would be tainted, and a suitable fair machine learning algorithm would be needed.
In this setting, $\hat{\Y}$ can be \textit{altered} w.r.t. the model of the original decisions $\Y$ to prevent discrimination. 
The crucial question is how to drop~$\Z$ from the model without inducing discrimination, that is, without increasing the impact of attributes $\X$ associated with $\Z$ in an unjustified and discriminatory way.


We propose that a non-discriminatory model shall remove the influence of the protected features $\Z$ on $\Y$, while preserving the influence of the remaining attributes $\X$ on $\Y$. This method allows addition of features to the model that increase model predictiveness, while preventing them from being used inadvertently as proxies for the protected features. 
%
To preserve influence of non-protected attributes, we define and minimize special loss functions. Such losses can be constructed on the grounds of causal influence (CDE, MDE), or model input influence or feature relevance measures (SHAP).
If there are many non-protected attributes, then the influence can be preserved for each of them separately or all of them together; we study both cases.

\subsection{Legal notions and responsibility for decision-making models}
Before we deep dive into mathematical definitions of respective loss functions, we first define a couple of abstractions of legal instruments by tying them to decision-making models and discuss legal responsibility for a model.

\subsubsection{Legality of the influence of protected features and their relationships with other attributes}
\label{sec:legal-influence}
We define unfair influence and fair relationship between protected attributes $\Z$ and decisions $\Y$ by tying them to legal instruments, i.e., legal terms of art that formally express a legally enforceable act.
\begin{definition}
\textbf{Unfair influence} is an influence of protected feature(s) $\Z$ on specified type of decisions $\Y$ that is judged illegal via some legal instrument. 
\end{definition}
For instance, the U.S. Civil Rights Act of 1968 (Titles VIII and IX, known as Fair Housing Act)~\cite{fairhousing} determines that decisions about sale, rental, and financing of housing shall not be influenced by race, creed, and national origin; the U.S. Civil Rights Act of 1964 (Title VII)~\cite{titlevii} determines that hiring decisions shall not be influenced by race, color, religion, sex, and national origin. 

In the context of making decisions $\Y$ using features $\X$, some of the features may be associated with, or affected by, the protected attribute $\Z$. Some of such features may be legally admissible for use in the decision-making if they are not \textit{unfairly influenced}, are relevant to decisions~$\Y$, and fulfil a business purpose.
\begin{definition}
\textbf{Fair relationship} of protected feature(s) $\Z$ with non-protected feature(s) $\X$ is a relationship in the context of making decisions $\Y$ that is judged legal via some legal instrument, e.g., business necessity clause. 
\end{definition}
For instance, in graduate admissions to University of California Berkeley it was found that females were less often admitted than males~\cite{bickel1975sex}. However, females applied to departments with lower admission rates than males and the overall admission process was judged legal. If we represent department choice with $X$, then we could use this feature in the model of admission decisions $\Y$, despite the fact that $X$ is causally influenced by gender.
Prior research shows that features perceived as fair tend to be volitional~\cite{grgichlaca2018human}, as in the above example.

From the perspective of supervised learning, the definitions of unfair and fair influence are exclusion and inclusion rules, respectively, determining which features are legally admissible in the model of $\Y$. Legal texts typically clearly define unfair influence, but fair relationships are determined on case-by-case basis. It is reasonable to assume that the purpose of business is to develop a model that on deployment is the most predictive possible. 
One could argue that any feature that is predictive of $\Y$ and different than $\Z$ fulfills business necessity and is fair to use. However, some of such features may be affected by \textit{unfair influence}. In such cases, one can remove $\Z$ from the unfairly influenced $\X$ and, then, from $\Y$, without inducing discrimination (see Section~\ref{sec:indirect-removal}).

\subsubsection{Legal responsibility for a decision-making model vs. its causal interpretation}
\label{sec:legal-responsibility}

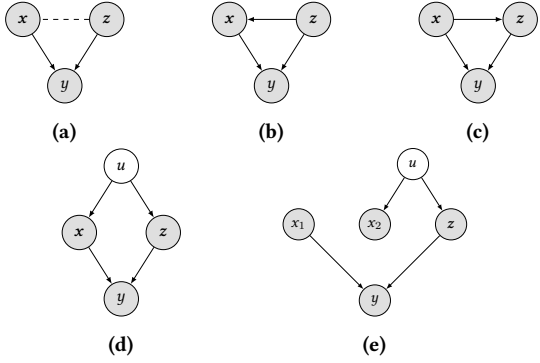
\begin{figure}[t]
\centering
\begin{subfigure}[b]{0.15\textwidth}
 \centering
    \scalebox{0.65}{
    \begin{tikzpicture}
    \node[obs]                   		(x)      {$\x$} ; %
    \node[obs, right=of x]       		(z)      {$\z$} ; %
    \node[obs, below=of $(x)!0.5!(z)$]	(y)      {$\y$} ; %
    \edge[dashed,-] {z} {x} ; %
    \edge {x,z} {y} ; %
    \end{tikzpicture}
    }
 \caption{}
 \label{fig:graph2a}
\end{subfigure}
\begin{subfigure}[b]{0.15\textwidth}
 \centering
    \scalebox{0.65}{
    \begin{tikzpicture}
    \node[obs]                   		(x)      {$\x$} ; %
    \node[obs, right=of x]       		(z)      {$\z$} ; %
    \node[obs, below=of $(x)!0.5!(z)$]	(y)      {$\y$} ; %
    \edge {z} {x} ; %
    \edge {x,z} {y} ; %
    \end{tikzpicture}
    }
 \caption{}
 \label{fig:graph2b}
\end{subfigure}
\begin{subfigure}[b]{0.15\textwidth}
 \centering
    \scalebox{0.65}{
    \begin{tikzpicture}
    \node[obs]                   		(x)      {$\x$} ; %
    \node[obs, right=of x]       		(z)      {$\z$} ; %
    \node[obs, below=of $(x)!0.5!(z)$]	(y)      {$\y$} ; %
    \edge {x} {z} ; %
    \edge {x,z} {y} ; %
    \end{tikzpicture}
    }
 \caption{}
 \label{fig:graph2c}
\end{subfigure}
\begin{subfigure}[b]{0.15\textwidth}
 \centering
    \scalebox{0.65}{
    \begin{tikzpicture}
    \node[obs]                   		(x)      {$\x$} ; %
    \node[obs, right=of x]       		(z)      {$\z$} ; %
    \node[latent, above=of $(x)!0.5!(z)$]                  		(u)      {$\u$} ; %
    \node[obs, below=of $(x)!0.5!(z)$]	(y)      {$\y$} ; %
    \edge {x,z} {y} ; %
    \edge {u} {x,z} ; %
    \end{tikzpicture}
    }
 \caption{}
 \label{fig:graph2d}
\end{subfigure}
\begin{subfigure}[b]{0.22\textwidth}
 \centering
    \scalebox{0.59}{
    \begin{tikzpicture}
    \node[obs]                   		(x1)      {$x_1$} ; %
    \node[obs, right=of x1]        		(x2)      {$x_2$} ; %
    \node[obs, right=of x2]       		(z)      {$\z$} ; %
    \node[obs, below=of x2]	(y)      {$\y$} ; %
    \node[latent, above=of $(x2)!0.5!(z)$]                  		(u)      {$\u$} ; %
    \edge {x1,z} {y} ; %
    \edge {u} {x2,z} ; %
    \end{tikzpicture}
    }
 \caption{}
 \label{fig:graph2e}
\end{subfigure}

\caption{The consi\-dered setting. 
We make no assumptions about the relations between $\x$ and $\z$ (marked with a dashed edge), nor their components. Hence, the graph (a) includes all exemplary cases (b-e). The graph (e) depicts the data-generating process from Example and shows no relationship between the components $x_1$ and $x_2$ of $\x$. The random variable $\u$ is an exogenous noise, i.e., an unmeasured independent random variable.}
\label{fig:graph2}
\end{figure}


%
To determine responsibility for potentially harmful decisions, legal systems consider the epistemic state of decision-makers~\cite{shafer2001causality, chockler2003responsibility}, 
e.g., whether an employer knew about discrimination in company's hiring process, and their intentions~\cite{Altman2016Discrimination}, i.e., the employer may be expected to do their due diligence to identify discrimination and to correct their hiring process given their knowledge.
In the context of decision-making models, the epistemic state corresponds to a potentially discriminatory model~$\Y$ of the respective real-world decision-making, whereas intentions correspond to learning objectives, methods, and feature selection that result in a discriminatory model~$\Y$ and a desired non-discriminatory~$\hat{Y}$.
The first step towards developing non-discriminatory models is finding accurate and robust, potentially causal~\cite{forster1994how, Spirtes2016Causal, Constantinou2020Large-scale}, models of discriminatory decisions in close collaboration with domain experts.
Machine learning models are developed in best faith to maximize accuracy, but often are not causal and not robust to covariate shifts~\cite{koh2020wilds, sagawa2020investigation}, i.e., they constitute an inaccurate epistemic state. 
Unfortunately, in practice it may be impossible to test causal validity of model $y(\x,\z)$, because of limited and unobserved data, privacy concerns, and the infeasibility or prohibitive costs of causal experimentation. 
In such situations, legal systems may acquit model developers if the intentions and reasoning behind the development process of models of $\Y$ and $\hat{Y}$ were legally admissible, despite the incorrect epistemic state.
%
Either way, whether the model at hand does or does not represent causal relations between variables in the real world, the model is causal w.r.t. its own predictions and the parents of these predictions are $\X$ and, possibly, $\Z$, as detailed in the causal explainability literature~\cite{Janzing2019Feature}. That model can suffer the effects of training on discriminatory data. In the remainder of this paper, we use $Y$ to refer both to the causal process and its model, since the two are the same in the former "ideal" causal setting, but our reasoning and approach is applicable to the latter "practical" non-causal settings as well, since the induction of indirect discrimination is questionable on legal grounds, i.e., decision-maker's epistemic state may be incorrect, but their intentions shall be good (to identify and prevent discrimination using reasonable methods).

\subsection{Problem formulation based on causal influence measures}

Formal frameworks for causal models include classic potential outcomes (PO) and structural causal models (SCM)~\cite{pearl2009causality}. Other frameworks, such as segregated graphs~\cite{Shpitser2015Segregated} and extended conditional independence~\cite{Dawid2008Beware} generalize the classic frameworks, e.g., they introduce undirected and bidirectional causal relationships.
The methods proposed here rely only on the notion of intervention, which tends to have a consistent meaning across causal frameworks. 

The following formulas are for the graphs depicted in Figure~\ref{fig:graph2}, where all variables are observed.
We assume that there are direct causal links from $\X$ and $\Z$ to $\Y$. If this assumption does not hold, e.g., because supervised learning is used for nowcasting instead of forecasting, then the following methodology may suffer collider bias (Berkson's paradox)~\cite{Spirtes2016Causal, Constantinou2020Large-scale}. For instance, if the underlying causal graph is $\Y\rightarrow X \leftarrow Z$, then conditioning on $X$ makes $\Y$ and $Z$ depend on each other, despite the fact that $Z$ does not causally influence $Y$, so supervised learning based on samples $(x,z,y)$ would yield a model in which $Z$ unfaithfully (w.r.t. the causal graph) influences the model of $Y$.
We make no assumptions about the relations between $\X$ and $\Z$ and their components (Figure~\ref{fig:graph2a}), e.g., these relations may be direct causal links (Figure~\ref{fig:graph2b}-\ref{fig:graph2d}) or associations (Figure~\ref{fig:graph2e}). Finally, it is assumed that there are no unmeasured confounders. 

%
In the notation of SCM and PO, the potential outcome for variable $\Y$ after intervention $do(\X=\x,\Z=\z)$ is written as $Y_{\x,\z}$, which is the outcome we would have observed had the variables $\X$ and $\Z$ been set to the values $\x$ and $\z$ via an intervention.
The causal \textit{controlled direct effect} on $\Y$ of changing the value of $\Z$ from a reference value $\z$ to $\z'$ given that $\X$ is set to $\x$~\cite{pearl2009causality} is
\begin{align}
    \text{CDE}_{\Y}(\z',\z | \x) = \E[ \Y_{\x,\z'} - \Y_{\x,\z}].
\end{align}
%


Next, we define direct, indirect, and induced discrimination by tying the causal concept of controlled direct effect to the notions of \textit{fair influence} and \textit{unfair relationship}, which are abstractions of respective legal concepts.



%
\begin{definition}
\textbf{Direct discrimination} is an \textit{unfair influence} of protected attribute(s) $\Z$ on the decisions $\Y$ and \mbox{$\exists_{z,z'} \exists_{\x} \text{CDE}_{\Y}(\z,\z' | \x) \neq 0$}.
\end{definition}



\begin{definition}
\textbf{Indirect discrimination} is an influence on the decisions $\Y$ of feature(s)~$\X$ whose \textit{relationship} with $\Z$ is not \textit{fair} and $\exists_{x,x'} \exists_{\z} \text{CDE}_{\Y}(\x,\x' | \z) \neq 0$.
\end{definition}

To remove direct discrimination, one can construct a model~$\hat{Y}$ that does not use $\Z$.
However, the removal of direct discrimination may induce discrimination via the attributes $X_i$ associated with the protected attributes $\Z$, even if there is no causal link from $\Z$ to $X_i$.

\begin{definition}
\textbf{Discrimination induced} via~$X_i$ is a transformation of the process generating $\Y$ not affected by direct and indirect discrimination into a new process $\hat{Y}$ that modifies the influence of certain $X_i$ depending on $\Z$ between the processes $Y$ and $\hat{Y}$ in the sense that 
$\exists_{\z} \exists_{\x,\x'} 
\text{CDE}_{\Y}(\x,\x' | \z) \neq \text{CDE}_{\hat{Y}}(\x,\x' | \z)$ given that 
$P(\x|\z) \neq 
P(\x)$ or $P(\x'|\z) \neq 
P(\x')$.
\end{definition}

\begin{example*}
Consider the aforementioned linear models of loan interest rate, $\hat{y}_1$ and $\hat{y}_2$. Note that $CDE_{\hat{y}_1}(\x,\x' | \z) - CDE_{\hat{y}_2}(\x,\x' | \z) = r*(x_2-x_2') $, since $\x$ has two components $\x_1$ and $\x_2$ and the first component is reduced, so the model~$\hat{y}_2$, that drops the protected attribute, induces indirect discrimination via $X_2$, because $X_2$ serves as a proxy of $Z$.
\end{example*}

Following causal inference literature~\cite{pearl2009causality}, to measure the extent of induced discrimination, we introduce natural direct and indirect effects using nested counterfactuals, $\Y_{\X_{\z},\z'}$, denoting the outcome that would have been observed if $\Z$ were set to $\z'$ and $\X$ were set to the value it would have taken if $\Z$ were set to $\z$.
\textit{Natural direct effect} of changing the value of $\Z$ from a reference value $\z$ to $\z'$ is
\begin{align}
    \text{NDE}_{\Y}(\z',\z ) = \E[ \Y_{\X_{\z},\z'} - \Y_{\X_{\z},\z}].
    \label{eq:nde}
\end{align}
However, the measure NDE faces some challenges: to see this, consider the graphs in Figure~\ref{fig:graph2}. For the graph in Figure~\ref{fig:graph2b} the \textit{causal mediation formula}~\cite{pearl2009causality} yields
\begin{equation}
\begin{split}
    \text{NDE}_{\Y}(\z',\z ) &= 
    \E[ \Y_{\X_{\z},\z'} - \Y_{\X_{\z},\z}]\\&= 
    \E_{\X\sim P(\X|\z)}[\Y|\X,\z'] - \E_{\X\sim P(\X|\z)}[\Y|\X,\z].
\end{split}
\end{equation}
For the graphs in Figures~\ref{fig:graph2c} and~\ref{fig:graph2d} it yields a different value,
\begin{align}
    \text{NDE}_{\Y}(\z',\z ) = 
    \E_{\X\sim P(\X)}[\Y|\X,\z'] - \E_{\X\sim P(\X)}[\Y|\X,\z],
\end{align}
because in this case $\X$ is not causally affected by $\Z$ and, hence, here the expectations are over $P(\X)$ instead of $P(\X|\z)$. These expectations come from the nested interventions, i.e., $\X_{\z}$.
We argue that the direct effect of $\Z$ on $\Y$ shall not depend on the direction of the causal link between $\X$ and $\Z$. Furthermore, the choice to use $\X_{\z}$ as a reference value in the definitions of natural direct effects is arbitrary, e.g., one could use $\X_{\z'}$ instead. To address these two issues, we introduce a \textit{marginal direct effect} as
\begin{equation}
\begin{split}
    \text{MDE}_{\Y}(\z',\z ) 
    &= \E[ \Y_{\X'',\z'} - \Y_{\X''_,\z}]
    = \E_{\X\sim P(\X)} \text{CDE}_{\Y}(\z',\z | \X)\\
    &= \E_{\X\sim P(\X)}[\Y|\X,\z'] - \E_{\X\sim P(\X)}[\Y|\X,\z],
\end{split}
\end{equation}
which takes an expectation over the probabilistic interventions on $\X$, setting its value to random samples of $\X''$, where $\X''$ is a variable independent from all other variables, but has the same marginal distribution as $\X$. 
This measure yields the same value for all graphs in Figure~\ref{fig:graph2}.
Then, to preserve influence of non-protected attributes we can minimize the following loss
\begin{align}
L_\text{MDE}(\X) = 
\E_{\X'',\X} \ell( \text{MDE}_{\Y}(\X,\X''), \text{MDE}_{\hat{\Y}}(\X,\X'') ).
\label{eq:loss-mde}
\end{align}
or its feature-specific version, which computes the loss separately for each component of $\X$, 
\begin{equation}
\begin{split}
L_\text{MDE}^\text{IND}(\X) &=
\sum_i  L_\text{MDE}(X_i) \\ &=
\sum_i  \E_{X_i'',X_i} \ell( \text{MDE}_{\Y}(X_i,X_i''), \text{MDE}_{\hat{\Y}}(X_i,X_i'') ) .
\label{eq:loss-mde-feature-specific}
\end{split}
\end{equation}


A similar loss could be constructed based on the comparison between $\text{CDE}_{\Y}(\X,\X'' | \Z)$ and $\text{CDE}_{\hat{\Y}}(\X,\X'' | \Z)$. In this paper we focus on losses based on MDE or the SHAP input influence measure described next.

\subsection{Problem formulation based on input influence measures}

Alternatively, influence can be measured on the grounds of input influence measures introduced to explain black-box AI models.
%
%
For the purpose of this section, we introduce a concatenation of variables $\X$ and $\Z$ as $\W=\X\Z$, i.e., samples of $\W$ are tuples $\w = (\x,\z)$ and $\w \in \mathcal{X} \times \mathcal{Z} = \mathcal{W}$.
Components of $\W$ are indexed, e.g., $W_i$ is the i-th feature among the set $\mathcal{F}$ of all protected and non-protected features, i.e., $i \in \mathcal{F}$.
To measure the influence of a certain feature $W_i$ prior works suggest to make a probabilistic intervention on that variable by replacing it with an independent random variable~\cite{Datta2016Algorithmic, Lundberg2017unified, Janzing2019Feature}. 
In particular, let primed variables have the same joint distribution as the non-primed variables, $\forall_{\w \in \cal{W}} P(\W'=\w)=P(\W=\w)$, while being independent from them, $\W' \perp \W$. 
Let double primed variables have the same marginal distributions as the non-primed variables, $\forall_{i\in \cal{F}} \forall_{w \in \cal{W}_i} P(W''_i=w)=P(W_i=w)$, and be independent from each other and the non-primed variables, i.e., $\forall_{i\in \cal{F}} \forall_{j \neq i} W''_i \perp W''_j$, $\W'' \perp \W'$ and $\W'' \perp \W$.

For any subset of features $T$ that does not contain $i$, prior works define a marginal influence (MI) using $\W'$ as a random baseline~\cite{Datta2016Algorithmic, Janzing2019Feature},
\begin{align*}
    \text{MI}_{\Y}(W_i | \w, T ) = 
    \E_{\W'} \left[Y_{\w_{T\cup \{i\}} \W'_{-(T\cup \{i\})}} - Y_{\w_T \W'_{-T}} \right],
\end{align*}
where the random variable $\W_T \W'_{-T}$ represents a concatenation of random variables $\W_T$ and $\W'_{-T}=\W'_{\cal{F} \setminus T}$, which amounts to a modified $\W$ with its components $W_i$, for each $i\in {\cal{F} \setminus T}$, replaced by the respective components of $\W'$; likewise $\w_T \W'_{-T}$ is a concatenation of sample $\w_T$ and random variable  $\W'_{-T}$.

A popular measure of the influence of input $w_i$ is based on the Shapley value (SHAP), which averages the marginal influence over all possible subsets $T$ of all features except for~$i$~\cite{Datta2016Algorithmic, Lundberg2017unified},
\begin{equation}
    \text{SHAP}_{\Y}(w_i | \w ) =  \sum_{T\subseteq \cal{F}\setminus \{i\}} \frac{\text{MI}_{\Y}(W_i | \w, T )}{ |\cal{F}| \binom{|\cal{F}|-1}{|T|} }. 
    \label{eq:shap}
\end{equation}

For instance, for the case of two variables,
\begin{equation}
    \text{SHAP}_{\Y}(\x |\x,\z) =  \E_{\X',\Z'} [(
Y_{\x,\z} - Y_{\X',\z} + Y_{\x,\Z'} - Y_{\X',\Z'})/2].
\end{equation}

Then, to preserve influence of non-protected attributes we can minimize the respective loss,
\begin{align}
L_\text{SHAP}(\X) = 
\E_{\X} \ell( \E_{\Z''} \text{SHAP}_{\Y}(\X |\X\Z''), \E_{\Z''} \text{SHAP}_{\hat{\Y}}(\X|\X\Z'') ),
\end{align}
or its feature-specific version,
\begin{equation}
\begin{split}
L_\text{SHAP}^\text{IND}(\X) =
\sum_i L_\text{SHAP}(X_i) \\ =  
\sum_i \E_{\X} \ell( \E_{\Z''} \text{SHAP}_{\Y}(X_i|\X\Z''), \E_{\Z''} \text{SHAP}_{\hat{\Y}}(X_i|\X\Z'') ).
\label{eq:loss-shap-feature-specific}
\end{split}
\end{equation}




While here we have constructed loss functions based on SHAP, other input influence measures, such as PFI or SAGE, can be used as well~\cite{Altmann2010Permutation, Ribeiro2016Why, sundararajan2017axiomatic, Marx2019Disentangling, Covert2020Understanding}. We leave the exploration of other losses for future works.

\section{Learning fair and explainable models}
We seek models $\hat{\Y}$ that remove the influence of the protected attributes $\Z$, while preserving the influence of non-protected attributes $\X$ by minimizing 
$L_\text{MDE}(\X)$ or $L_\text{SHAP}(\X)$, which lead to a simple closed-form solution, or their feature-specific versions, i.e., $L^\text{IND}_\text{MDE}(\X)$ or $L^\text{IND}_\text{SHAP}(\X)$, which we solve via transfer learning.
Either of these approaches can be used to remove direct or indirect discrimination (see example in Subsection~\ref{sec:indirect-removal}).

\subsection{Minimizing \mathintitle{$L_\text{MDE}(\X)$ or $L_\text{SHAP}(\X)$}}

\begin{definition}
\textbf{Interventional mixture} of a model $y(\x,\z)$ w.r.t. attribute $\Z$ is a model $\hat{y}_\pi(\x) = \E_{\tilde{\Z}} [\hat{y}(\x,\tilde{\Z})]$, where $\tilde{\Z}$ is a random variable independent from all other variables, has the same support as $\Z$, and a distribution $\pi(\tilde{\Z})$. \\
\textbf{Marginal interventional mixture} (MIM) is  $\hat{y}_\text{MIM}(\x) = \E_{\Z'} [\hat{y}(\x,\Z')]$. 
\end{definition}

\begin{proposition}
For variable $Y$, the objective $L_\text{MDE}(\X)$ is minimized by the MIM.  
\end{proposition}

\begin{proof}
$L_\text{MDE}(\X) = \E_{\X'',\X} \ell( \E_{\Z'} [ Y_{\X,\Z'} - Y_{\X'',\Z'}], \E_{\Z'} [ \hat{\Y}_{\X,\Z'} - \hat{\y}_{\X'',\Z'} ] )$, so for $\hat{y}_\text{MIM}(\x)=\E_{\Z} y(\x,Z)$ it is zero.
\end{proof}

\begin{proposition}
For a real-valued and analytic $y(x,z)$, the MIM is an interventional mixture that minimizes the objective $L_\text{SHAP}(\X)$.
\end{proposition}

\renewcommand*{\proofname}{Proof sketch}
\begin{proof}
Without loss of generality, for simplicity let us consider the case of two variables $X$ and $Z$. 
Let us expand $y(x,z)$ into a Taylor series around the point $x=0,z=0$. The series is a sum of components $C x^k z^l$, where $C$ is a constant and $k$ and $l$ are integers from $1$ to $\infty$. Then, we replace $Y$ in the definition of $L_\text{SHAP}$ with the Taylor series and make a proof by induction. Minimizing this objective gives a potentially infinite set of conditions $\E[\tilde{Z}^l] = \E [Z^l]$ for the respective moments of $\tilde{Z}$. Since $l$ can be any positive integer, these conditions are met if $P(\tilde{Z}=z)=P(Z=z)$. The full proof is in Appendix~A.
\end{proof}

\begin{example*}
In the loan interest rate example, the full model is $y(\x,z) = \beta_0 - x_1 - z$. The MIM is $\hat{y}_\text{MIM} = \beta_0 - x_1 -  \E_Z Z$. 
\end{example*}

\subsection{Minimizing \mathintitle{$L^\text{IND}_\text{MDE}(\X)$} and \mathintitle{$L^\text{IND}_\text{SHAP}(\X)$} via transfer learning}
%
The minimization of the feature-specific losses, \mathintitle{$L^\text{IND}_\text{MDE}(\X)$} and \mathintitle{$L^\text{IND}_\text{SHAP}(\X)$}, does not result in closed-form solutions, so we apply a respective gradient descent.
First, we drop the protected attribute(s) $Z$ from the data. We then obtain the ``Trad. w/o $Z$'' model by minimizing the cross entropy loss, $H(\hat{y}, y) = -\sum_i y_i \log \hat{y_i}$. Next, we optimize for either $L^\text{IND}_\text{MDE}(\X)$ or $L^\text{IND}_\text{SHAP}(\X)$. For both objectives we use $\ell_2$ loss. We refer to these two-stage optimization-based methods as OPT-MDE and OPT-SHAP, respectively. 
The training is done using momentum based\ gradient optimizer ADAM \cite{kingma2017adam} via batch gradient descent. We fine-tune two hyper-parameters: learning rate~($\alpha$) and number of epochs ($N$). During fine-tuning we pick the values for which we get the best performance on the validation set. In our datasets, $\alpha$ is from $10^{-3}$ to  $10^{-2}$ and $N$ is from $20$ to $100$. Our implementations of the methods are released publicly via \texttt{FaX-AI} Python library.

\subsection{Removal of indirect discrimination via nested use of proposed methods}
\label{sec:indirect-removal}
Potentially any feature that is predictive of $\Y$ and different than $\Z$ could fulfill business necessity, as we pointed in Subsection~\ref{sec:legal-influence}.
However, a feature $X_i$ can be unfairly and illegally influenced by $\Z$. If decisions $Y$ used such $X_i$, then $Y$ would be indirectly discriminatory. We have two options to prevent that: i)~not include feature $X_i$ in the model of $\Y$ or, ii)~create a model of $X_i$, remove from it the impact of $\Z$, then use the corrected $\hat{X}_i$ in the model of $\Y$, and finally drop the impact of $\Z$ on $\hat{Y}$, while using either of the proposed methods for removing the impact of $\Z$ from the models of $X_i$ and $\Y$. 
The latter solution can be improved via a counterfactual anaylsis in situations where we know the value of a variable for which we apply MIM to, i.e., $X_i$.

\begin{example*}
In the loan example, the annual salary $x_1$ of a loan applicant could have been affected by discrimination, e.g., $x_1=s+z$, where $s$ stands for job-related skills. In such case, a bank shall first debias the salary, either by developing a model of $X_1$ using available information about $S$ and applying our methods, or by retrieving a debiased $\hat{x}_1$ from another source, e.g., the applicant's employer, who is better positioned (and is obliged by law) to debias the salary. In this case, $\hat{x}_{1,\text{MIM}}=s+\bar{z}$ and $\hat{y}_\text{MIM} = \beta_0 - \hat{x}_{1,\text{MIM}} - \bar{z} = \beta_0 -s - 2 \bar{z}$, where $\bar{z}$ is the mean of $Z$, so skills determine interest rate.
\end{example*}

\subsubsection{Counterfactual mixtures. }
\label{sec:counterfactual}

Causality literature posits a causal hierarchy and distinguishes between interventional and counterfactual estimates~\cite{Pearl2016Causal}. The latter differ from former in that they assume that everything stays the same, including any exogenous noise values, when estimating the effect of an intervention.
Note that the interventional mixture calculates the value of $\hat{X}_i$ had the casual influence of $\Z$ been removed from it given the values of all \textit{observed} variables, but not the values of exogenous noise. 
However, each variable can contain \textit{exogenous} noise, i.e., unobserved intrinsic noise not associated with any other variable. In the situations where we know the value of the variable for which we want to develop a fair model, we can use that value to infer that variable's exogenous noise.
For such situations, we propose an \textit{marginal counterfactual mixture} (MCM), which merges the three canonical counterfactual reasoning steps with the MIM step: (\textit{abduction}) infer exogenous noise for a variable, (\textit{intervention}) apply the MIM to remove the influence of the protected attribute on that variable, and (\textit{counterfactual prediction}) estimate the counterfactual value of the variable given the exogenous noise and intervention.

\begin{figure*}[t]
\centering
\begin{subfigure}{\linewidth}
  \centering
  \includegraphics[width=0.65\linewidth]{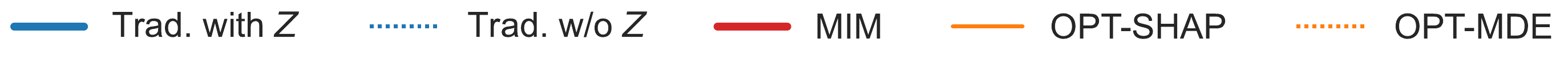}
\end{subfigure}
\begin{subfigure}{\linewidth}
  \centering
  \includegraphics[width=.195\linewidth]{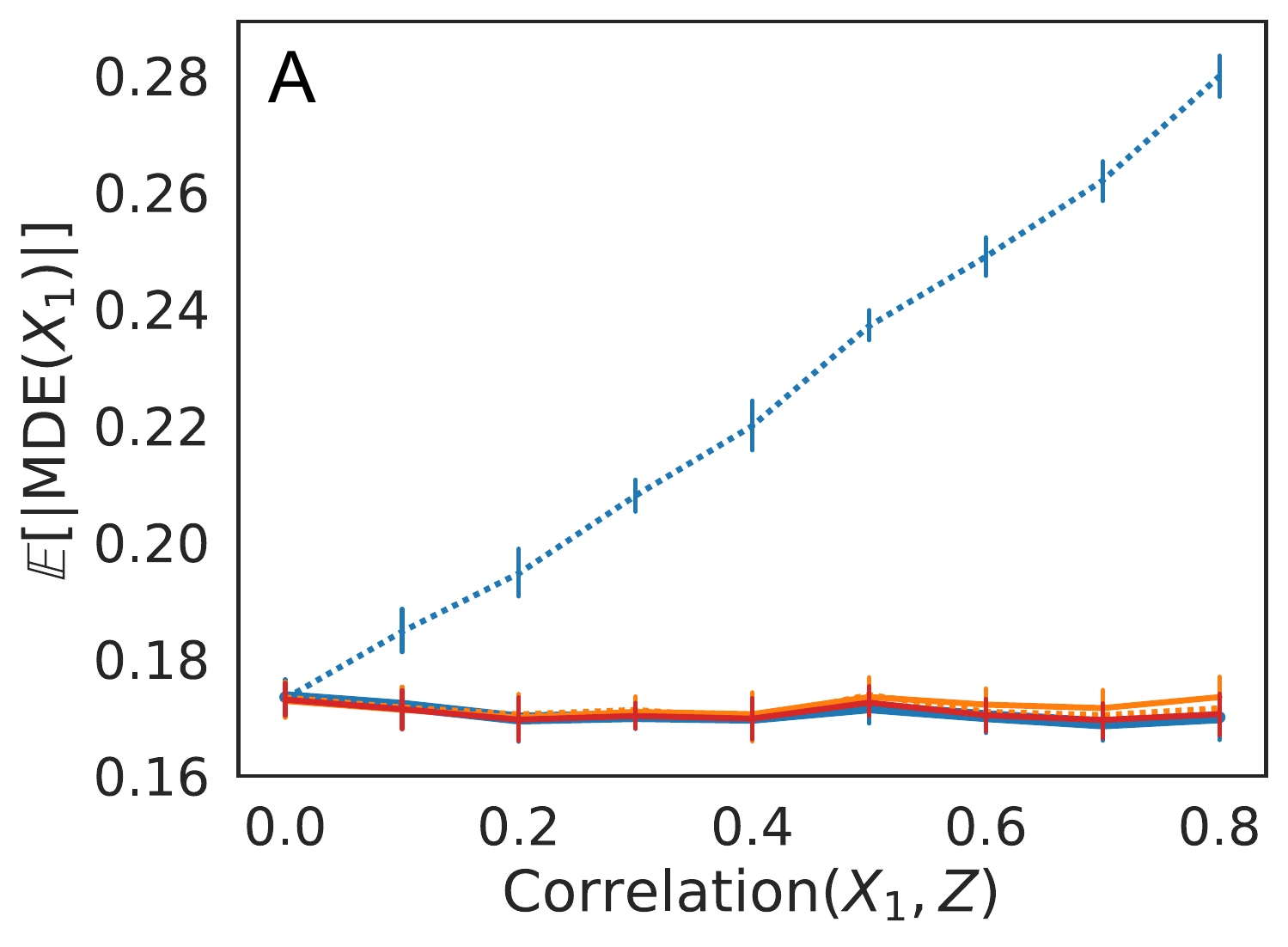}
  \includegraphics[width=.195\linewidth]{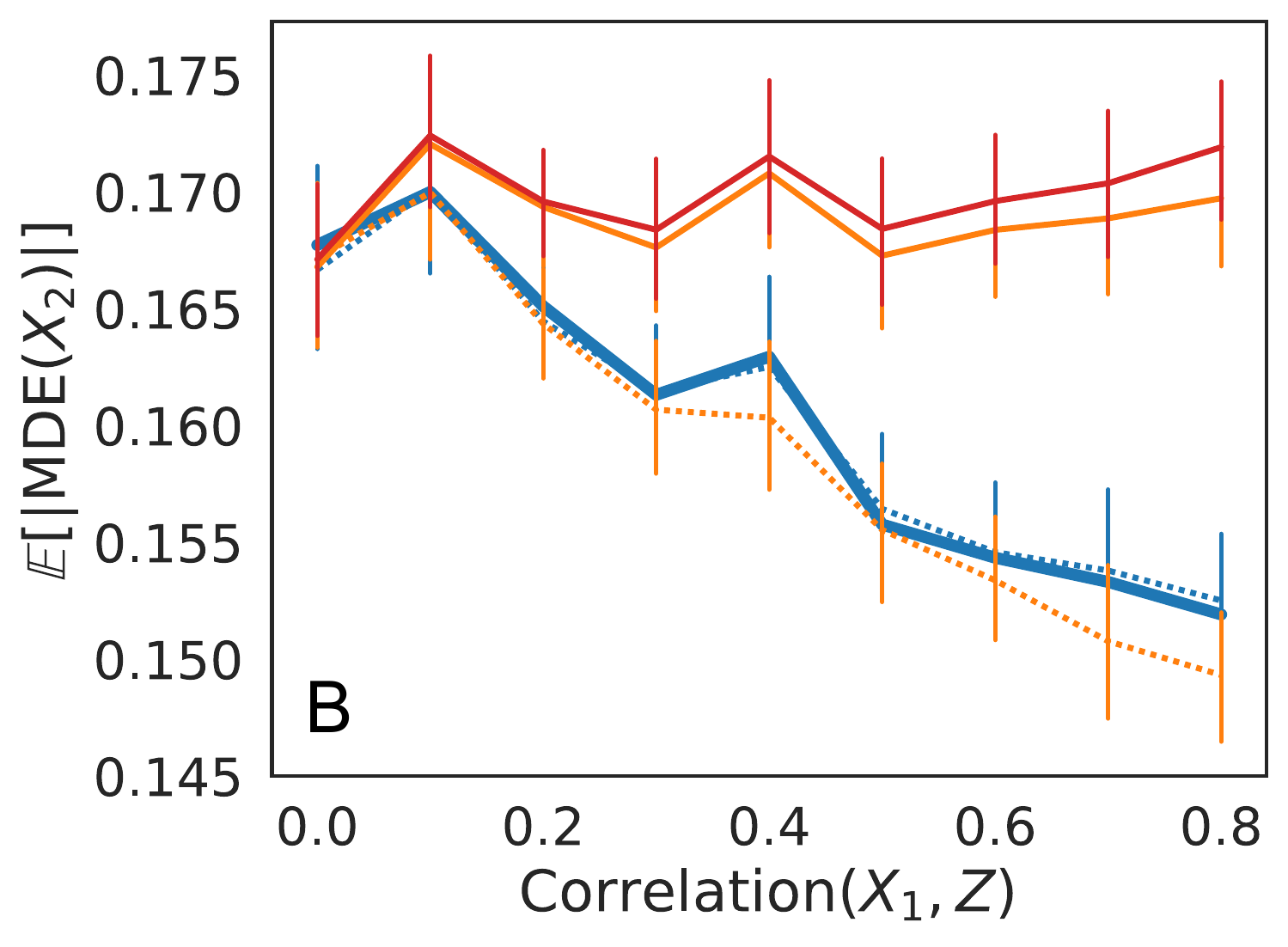}
  \includegraphics[width=.195\linewidth]{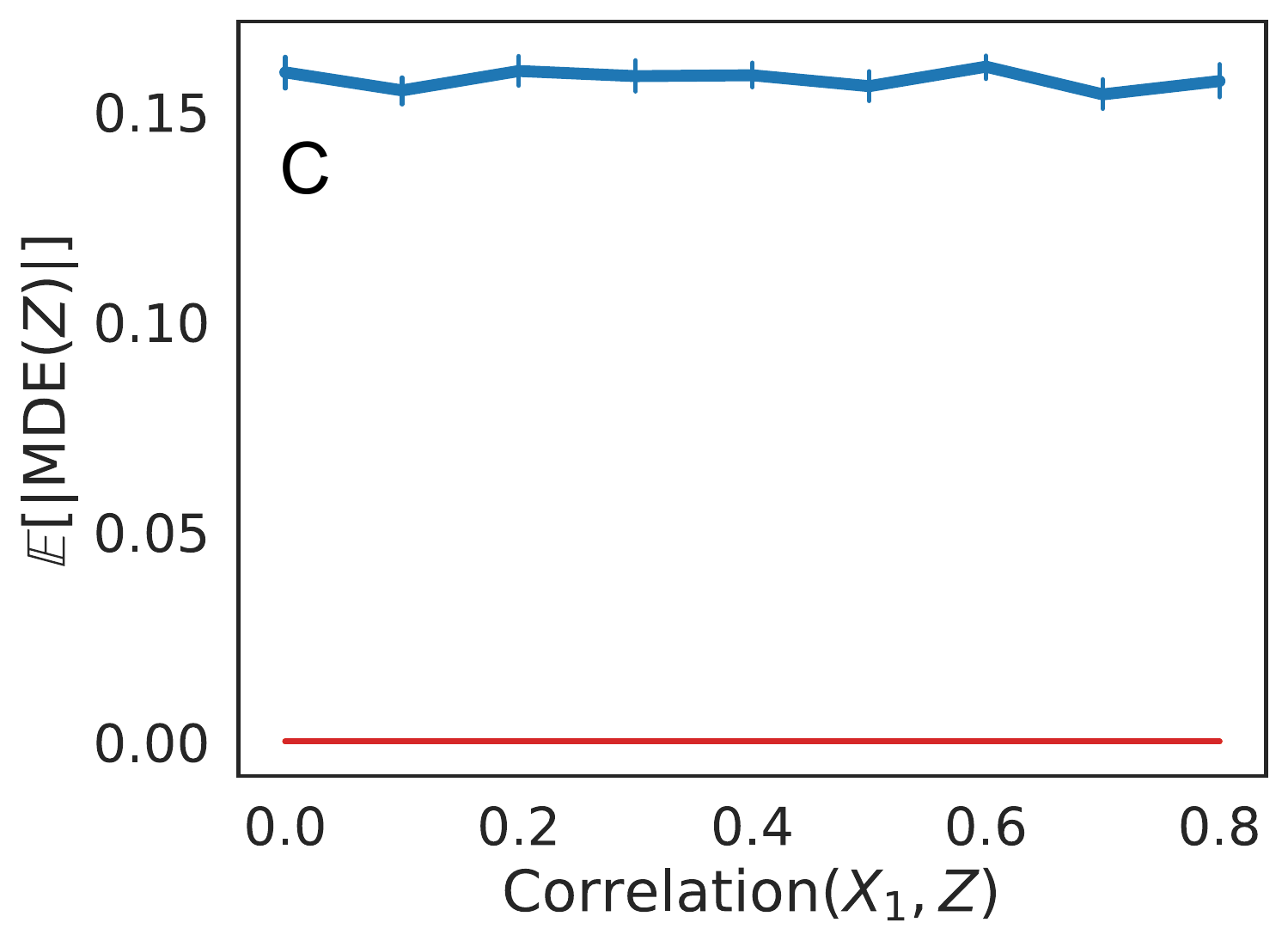}
  \includegraphics[width=.195\linewidth]{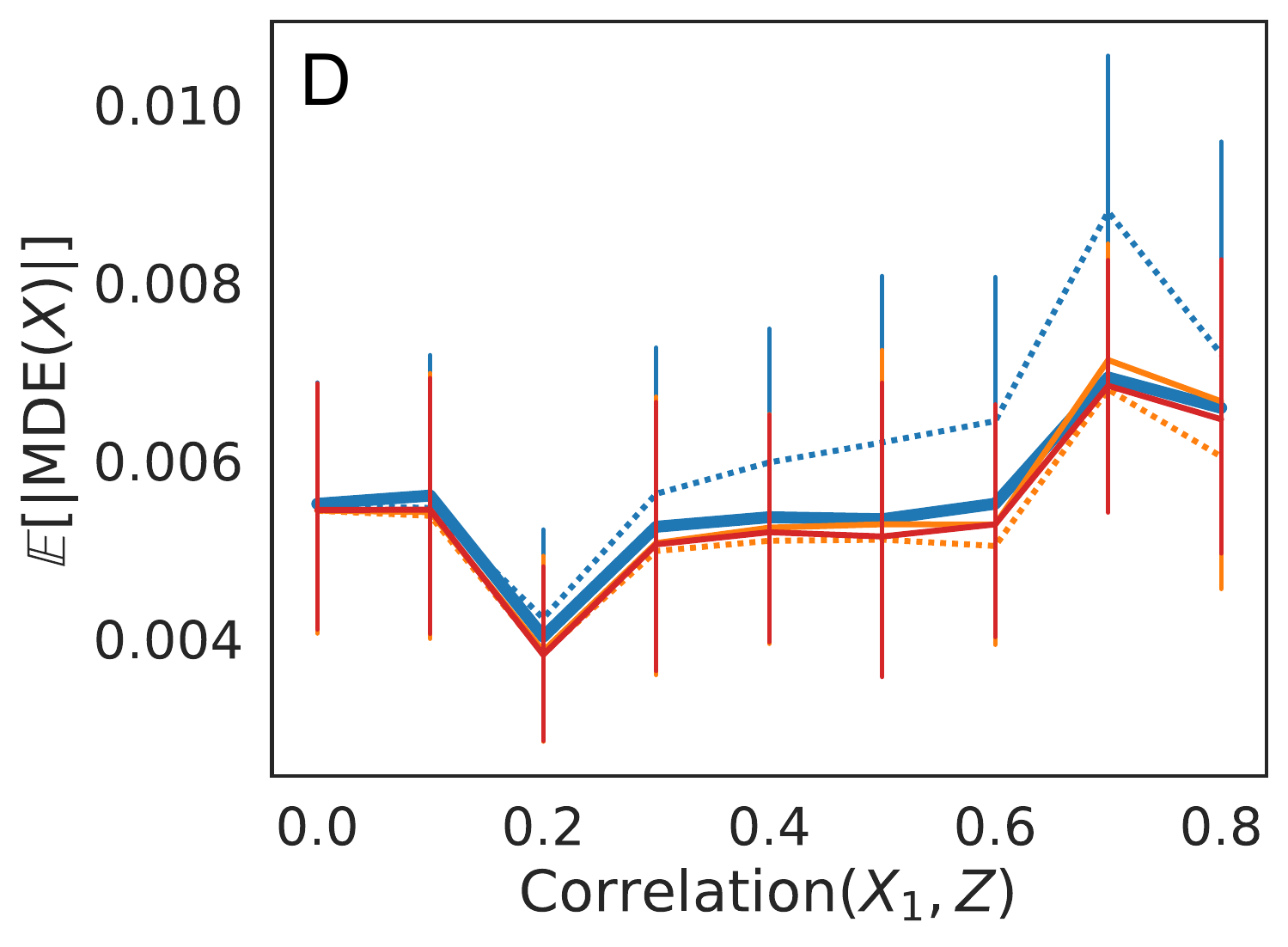}
  \includegraphics[width=.195\linewidth]{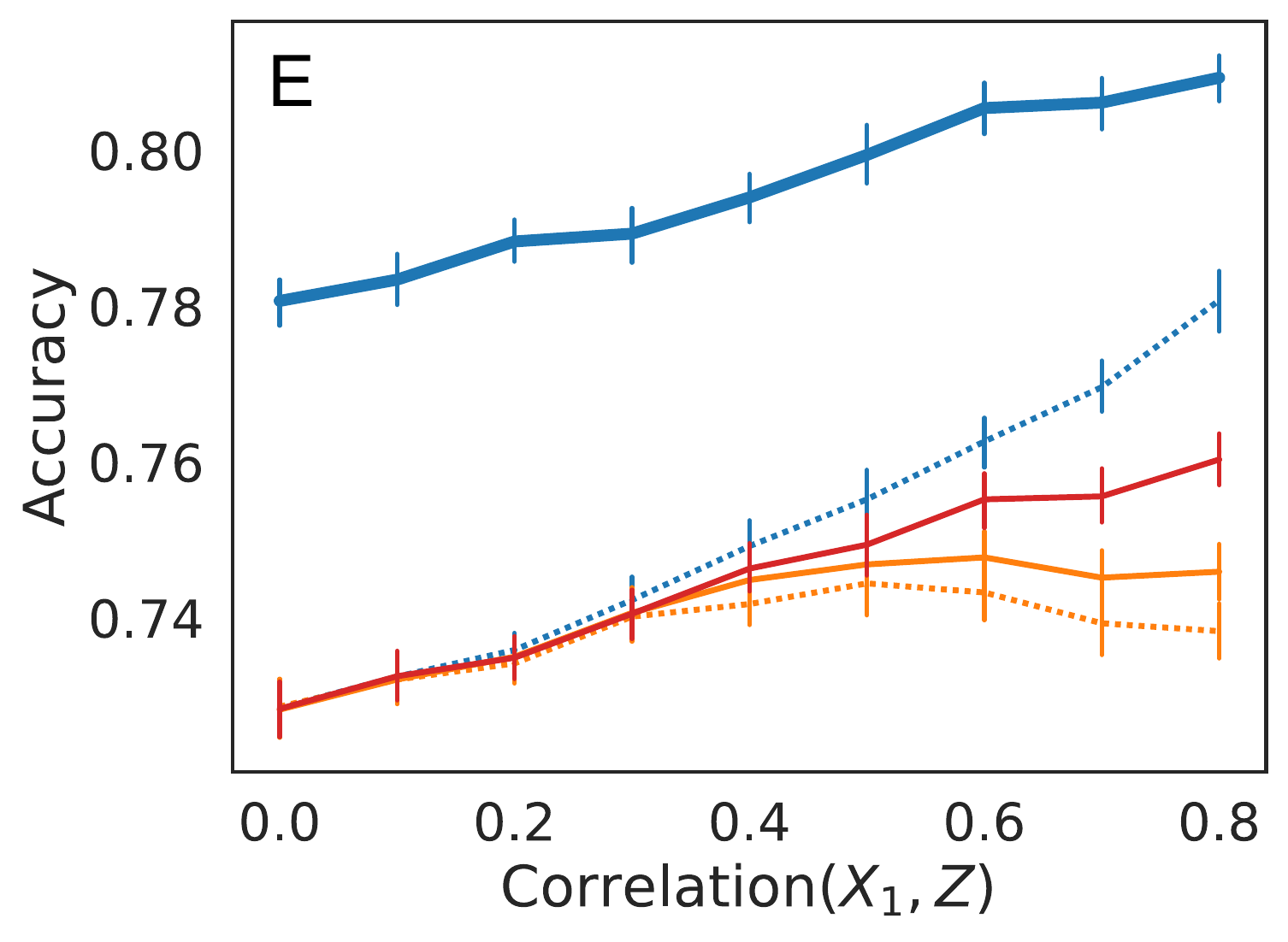}
  \includegraphics[width=.195\linewidth]{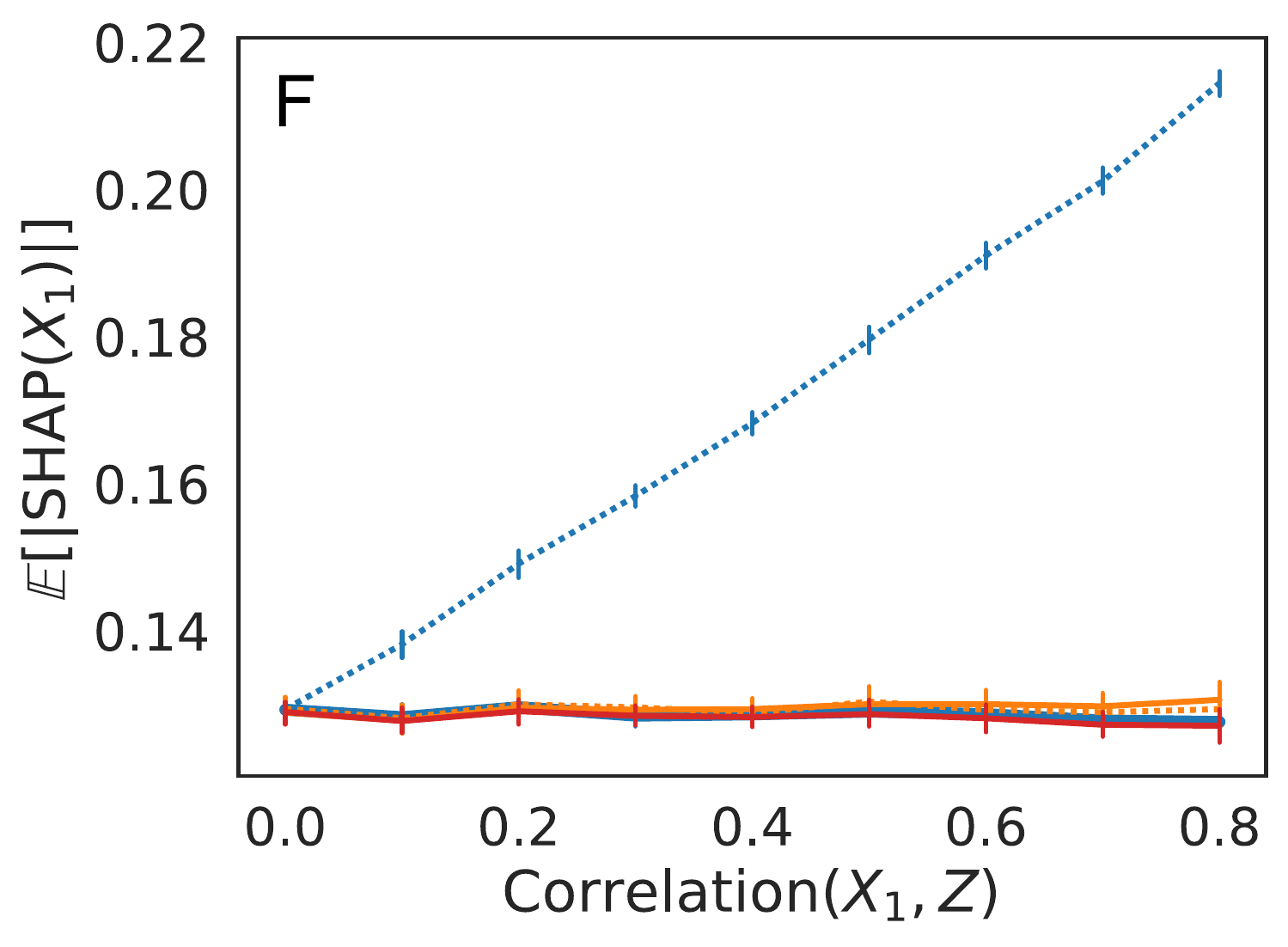}
  \includegraphics[width=.195\linewidth]{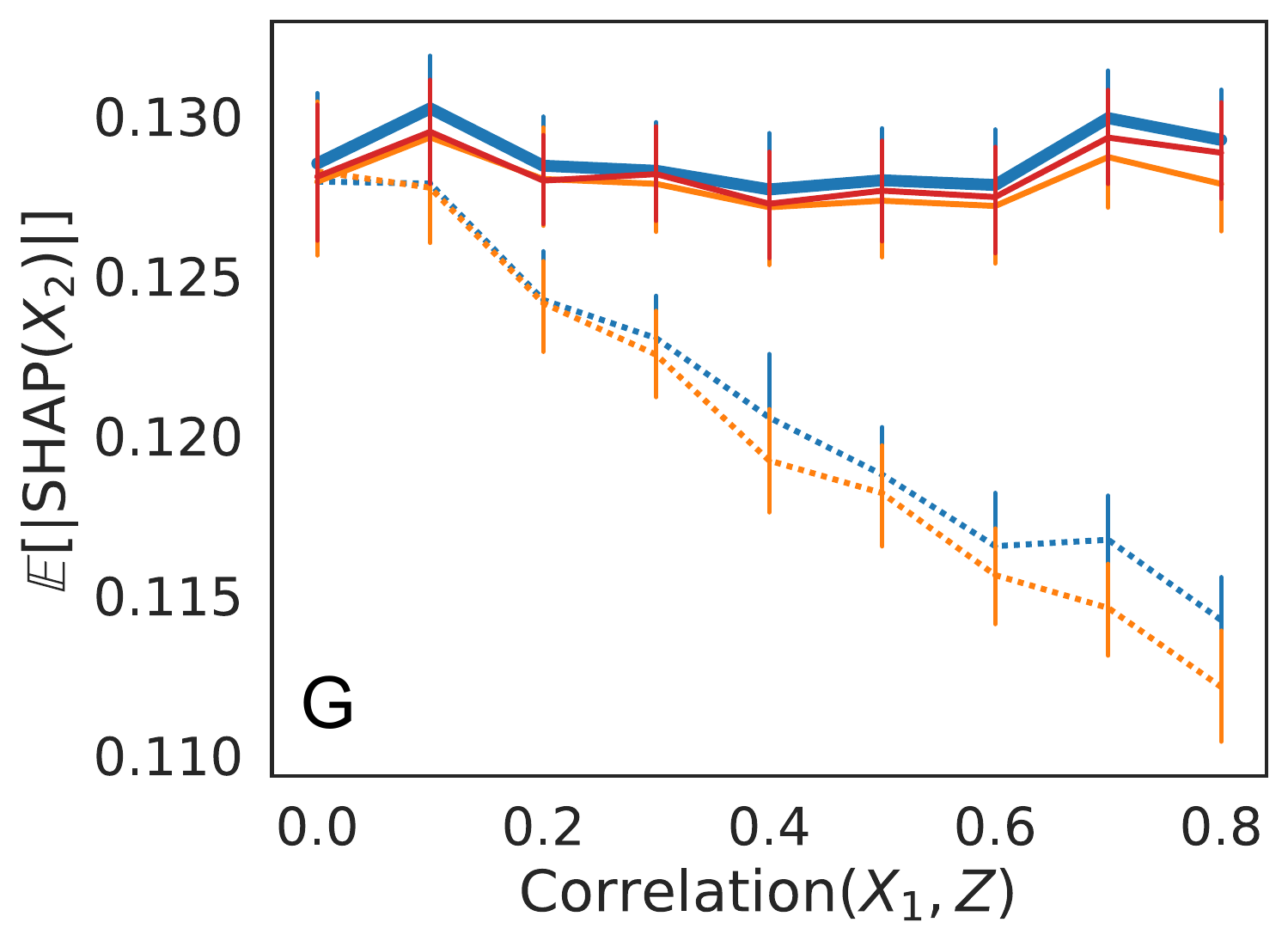}
  \includegraphics[width=.195\linewidth]{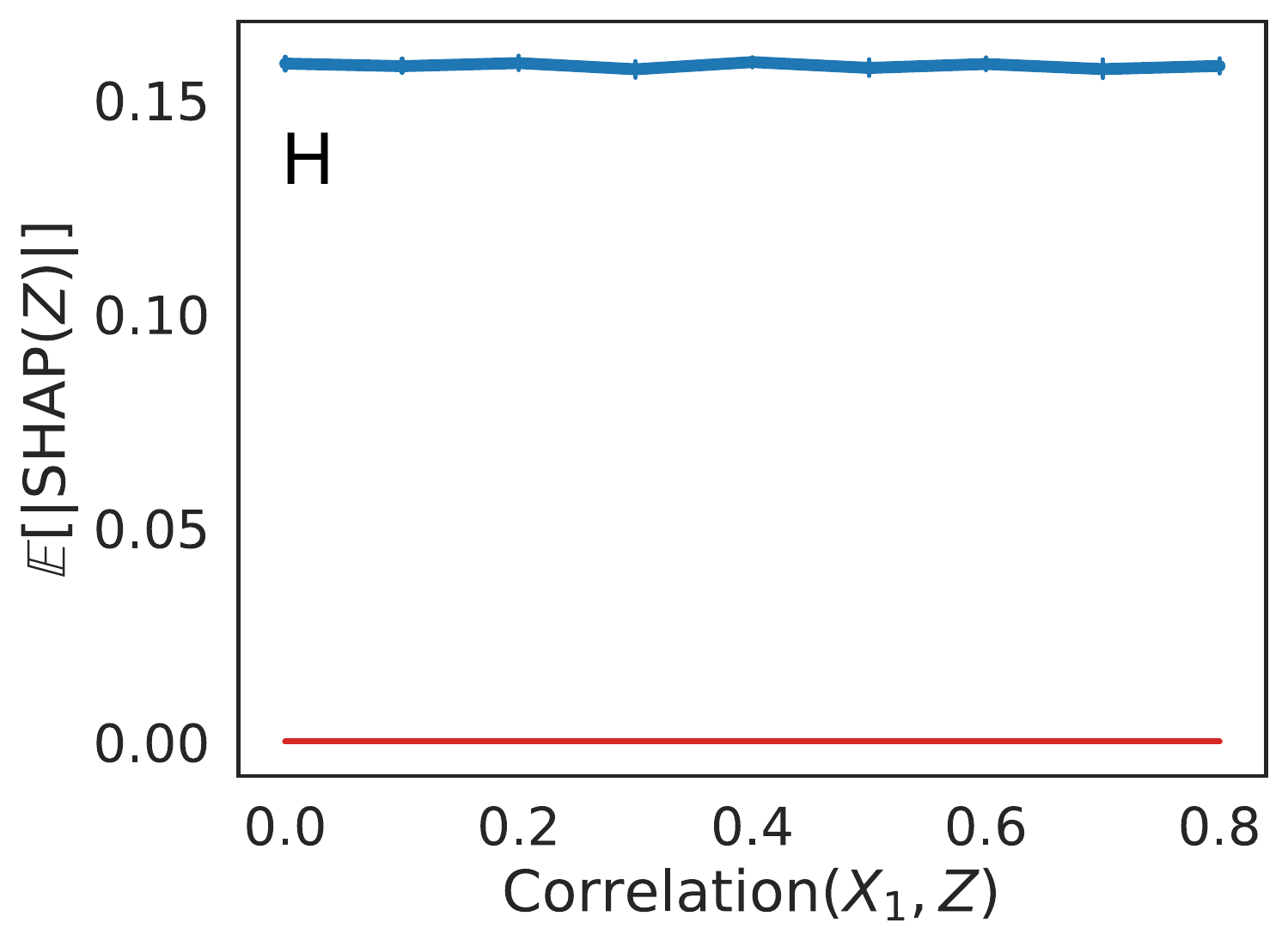}
  \includegraphics[width=.195\linewidth]{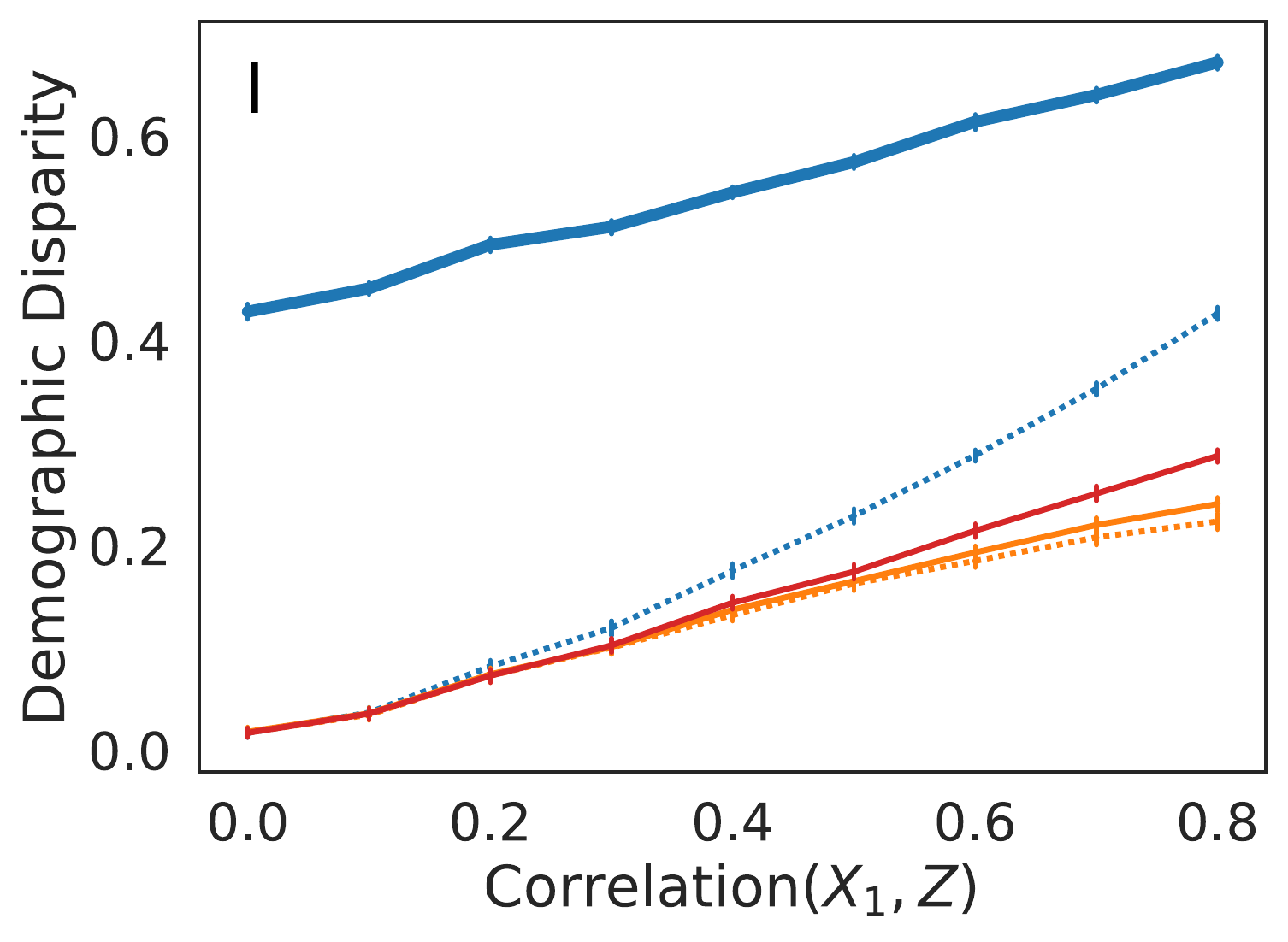}
  \includegraphics[width=.195\linewidth]{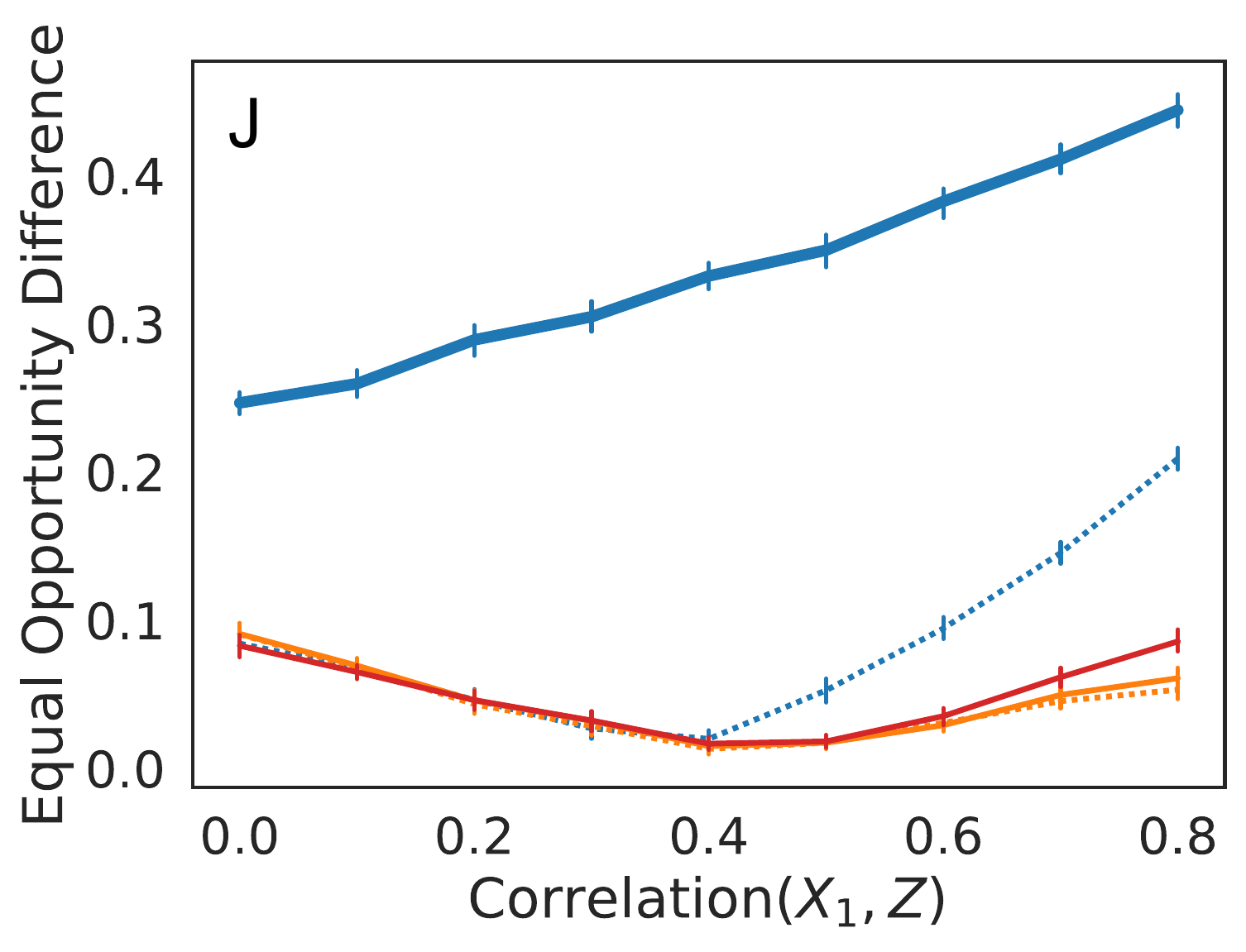}
\end{subfigure}
\caption{
Average absolute input influence, measured via MDE (panels A-D) or SHAP (F-H), model accuracy (E), demographic disparity (I), and accuracy disparity (J) versus the Pearson correlation between $X_1$ and $Z$, 
under Scenario~A, where $Y=\sigma(X_1 + X_2 + Z + 1)$.
}
\label{fig:aifsynset1}
\end{figure*}

\subsubsection{Comparison of counterfactual mixtures.}
\label{sec:vs-chiappa}

In contrast to our proposed methods, path-specific counterfactual fairness (PSCF) requires the knowledge of a full causal graph.
Hence, we study an exemplary linear model introduced in the PSCF paper~\cite{Chiappa2019Path}. We maintain the original notation:
%
%
%
\begin{align}
M &= \theta^m + \theta^m_z Z + \theta^m_c C + \epsilon_m,\\
L &= \theta^l + \theta^l_z Z + \theta^l_c C + \theta^l_l M + \epsilon_l, \\
Y &= \theta^y + \theta^y_z Z + \theta^y_c C + \theta^y_y M + \theta^y_l L + \epsilon_y,
\end{align}
where 
$C$, $M$, $L$ are components of~$\X$, while $\epsilon_c$, $\epsilon_m$, $\epsilon_l$ are exogenous noise variables. The causal influence of the protected attribute $Z$ on decisions $\Y$ and the mediator $M$ is assumed unfair and all other influences are fair. In other words, $Y$ is affected by direct discrimination via $Z$ and indirect discrimination via $M$. This means that the MIM needs to be applied first to $M$ and then to $Y$.

Without loss of generality, let us consider a scenario where we have enough samples to have perfect estimates of a well-specified model's parameters, so that the estimated model is $\hat{m}=\theta^m + \theta^m_z z + \theta^m_c c$.
In this scenario, the abduction step corresponds to computing $\epsilon_m = m-\hat{m}$, the intervention step to applying MIM to $\hat{m}$, yielding $\hat{m}^* = \theta^m + \theta^m_z \overline{z} + \theta^m_c c$, and the counterfactual prediction step to injecting the abducted noise into the estimated model, $\hat{m}^c = \theta^m + \theta^m_z \overline{z} + \theta^m_c c + \epsilon_m$.
 Similar to PCSF, the multi-stage MCM corrects the decision through a correction on all the variables that are influenced by the protected attribute along unfair pathways.
In the multi-stage MCM, we first apply the MCM to get a non-discriminatory counterfactual $\hat{m}^\text{c}$, then we propagate $\hat{m}^\text{c}$ to its descendants and apply the MCM to yield a fair counterfactual $\hat{l}^\text{c}$, and finally we propagate the two counterfactuals to $\hat{y}$ and apply the MIM to get~$\hat{y}^\text{c}$:
\begin{align}
\hat{m}^\text{c} &= \theta^m + \theta^m_z \overline{z} + \theta^m_c c + \epsilon_m= m - \theta^m_z(z-\overline{z}),\\
\hat{l}^\text{c} &= \theta^l + \theta^l_z z + \theta^l_c c + \theta^l_m \hat{m}^\text{c} + \epsilon_l = l - \theta_m^l(m-\hat{m}^\text{c}),\\
\hat{y}^\text{c} &= \theta^y + \theta^y_z \overline{z} + \theta^y_c c + \theta^y_m \hat{m}^\text{c} + \theta^y_l \hat{l}^\text{c},
\end{align}
where $\overline{z}$ stands for the mean of $Z$. If we applied only the MIM, we would not take advantage of estimating the noise terms and yield the following less-accurate estimators,
\begin{align}
\hat{m}_{\text{MIM}} &= \theta^m + \theta^m_z \overline{z} + \theta^m_c c,\\
\hat{l}_{\text{MIM}} &= \theta^l + \theta^l_z z + \theta^l_c c + \theta^l_m \hat{m}_{\text{MIM}},\\
\hat{y}_{\text{MIM}} &= \theta^y + \theta^y_z \overline{z} + \theta^y_c c + \theta^y_m \hat{m}_{\text{MIM}} + \theta^y_l \hat{l}_{\text{MIM}},
\end{align}
The difference in estimating $\epsilon_m$ yields $\hat{y}^\text{c} = \hat{y}_{\text{MIM}}+ \theta^y_m\epsilon_m + \theta^y_l \theta^l_m\epsilon_m$. Thus resulting in a larger error for the MIM than the MCM, i.e., $\E(Y-\hat{Y}_{\text{MIM}})^2
 = \E(Y-\hat{Y}^\text{c})^2 + (\theta^y_m\epsilon_m + \theta^y_l \theta^l_m\epsilon_m)^2$  

A comparison with PSCF reveals that $\hat{y}^\text{c} = \hat{y}_{\text{PSCF}} + \Delta$, where $\Delta = \overline{z}( \theta^y_z + \theta^y_m \theta^m_z + \theta^y_l \theta^l_m \theta^m_z )$. In fact, the mean squared error w.r.t. $Y$ is larger for PSCF than for MCM by the square of the difference, i.e., $\E(Y-\hat{Y}_{\text{PSCF}})^2 = \E(Y-\hat{Y}^\text{c})^2+\Delta^2$.
%
%
%
PSCF is based on NDE (Equation~\ref{eq:nde}), it was introduced for binary $Z$, and relies on a choice of reference value, $z'$, also known as baseline, which is assumed $z'=0$ in the above example. However, this choice is arbitrary and it is not clear what baseline should be for non-binary $Z$. By contrast, the MCM introduces a distribution $\pi(z')$ over the reference intervention, which mimics "probabilistic interventions" from explainability literature~\cite{Datta2016Algorithmic, Janzing2019Feature}. 
This difference between PSCF and MCM mirrors the difference between NDE and MDE, respectively, and it leads to $\Delta$.
Thanks to this, the MCM can be applied to continuous~$Z$ and it results in more accurate models.
The above result that MCM is at least as accurate as PSCF is true for any linear model and any choice of the reference~$z'$.

\section{Results of experiments}
We examine the performance of our method and other supervised learning methods addressing discrimination in binary classification on synthetic and real-world datasets. We measure $\E_{\X,\Z} \allowbreak|\text{SHAP}_{\Y}(X_i\allowbreak|\X, \Z)|$, following the measure of global feature influence proposed by \citet{Lundberg2017unified}, and $\E_{X_i,X'_i} |\text{MDE}_{\Y}(X_i,X'_i)|$, both of which are evaluated using outcome probabilities. 
Note that these measures are different than our loss functions, which make the results non-obvious, yet still intuitive.
To reduce computational costs, we use sub-sampling to compute the measures. 
In addition, we measure accuracy and demographic disparity ($|\p(\hat{y}=1|z=0)-\p(\hat{y}=1|z=1)|$). Results for other measures, such as equalized odds and equal opportunity difference, can be found in Appendix B. 
The datasets are partitioned into 20:80 test and train sets and all results, including model accuracy, are computed on the test set.





\begin{figure*}[t]
\centering
\begin{subfigure}{.9\linewidth}
  \centering
  \includegraphics[width=\linewidth]{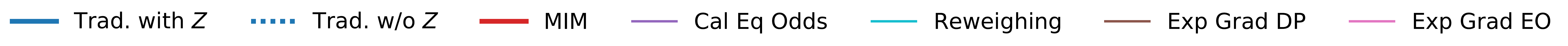}
\end{subfigure}
\begin{subfigure}{\linewidth}
  \centering
  \includegraphics[width=.195\linewidth]{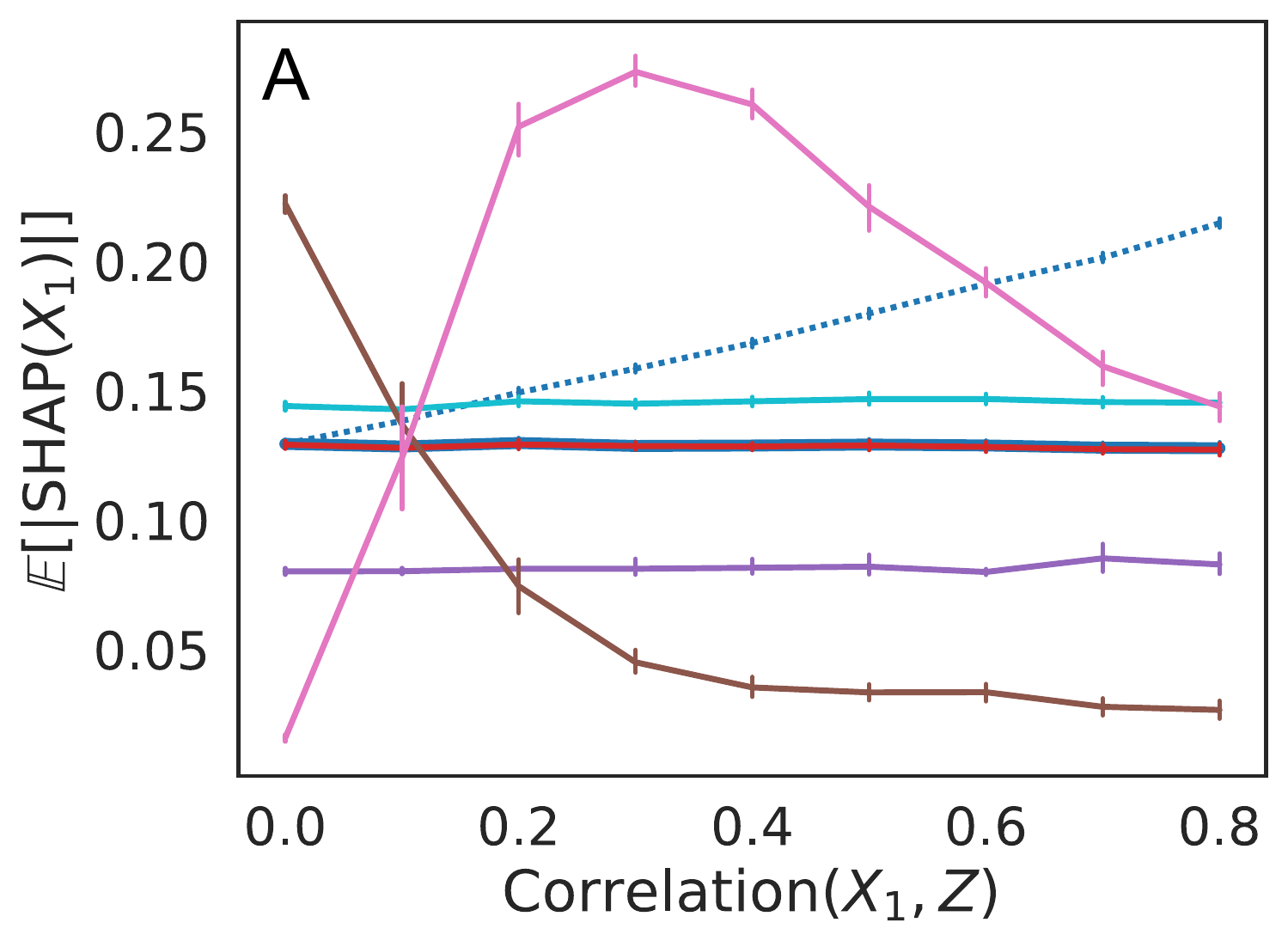}
  \includegraphics[width=.195\linewidth]{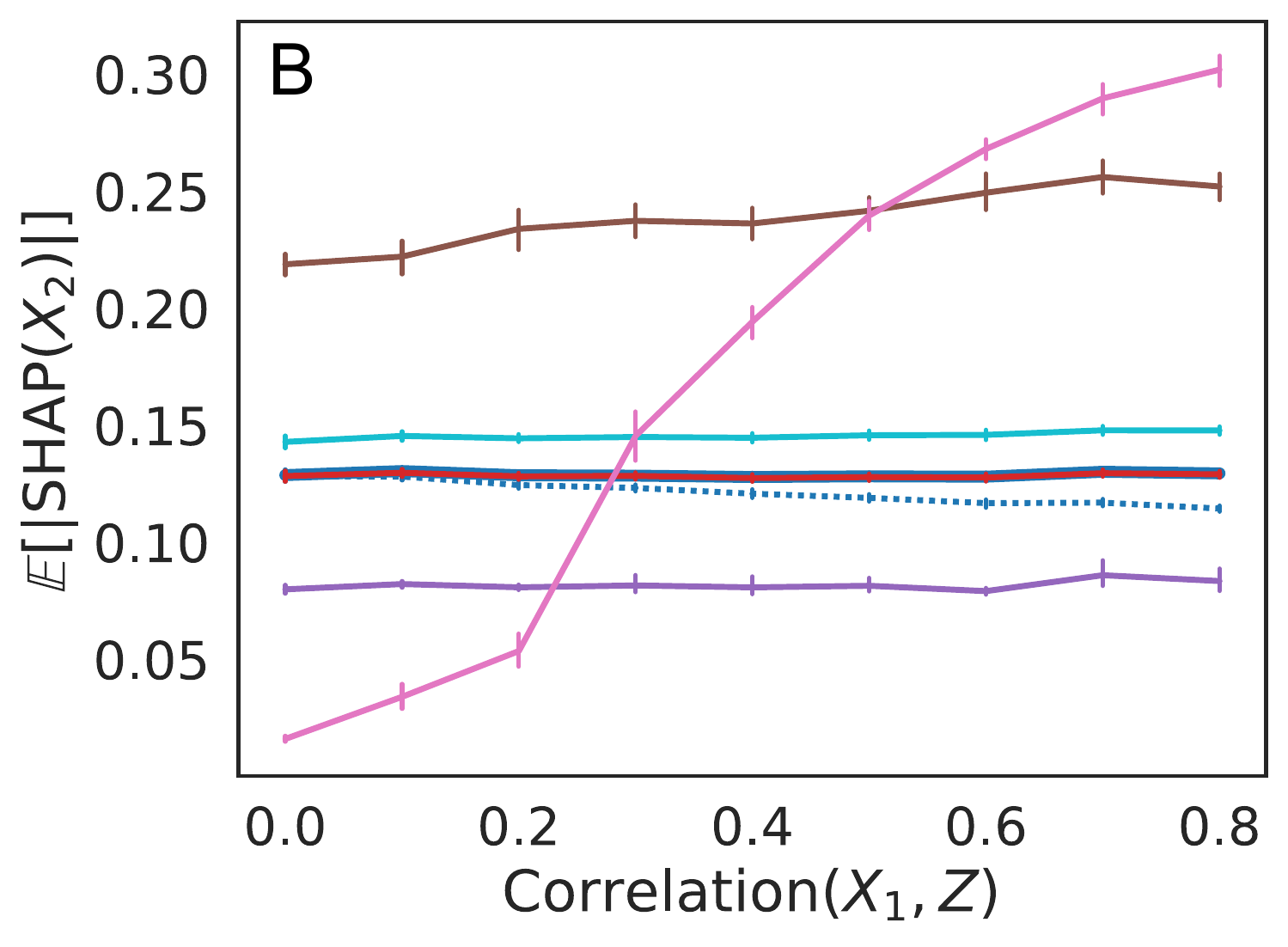}
  \includegraphics[width=.195\linewidth]{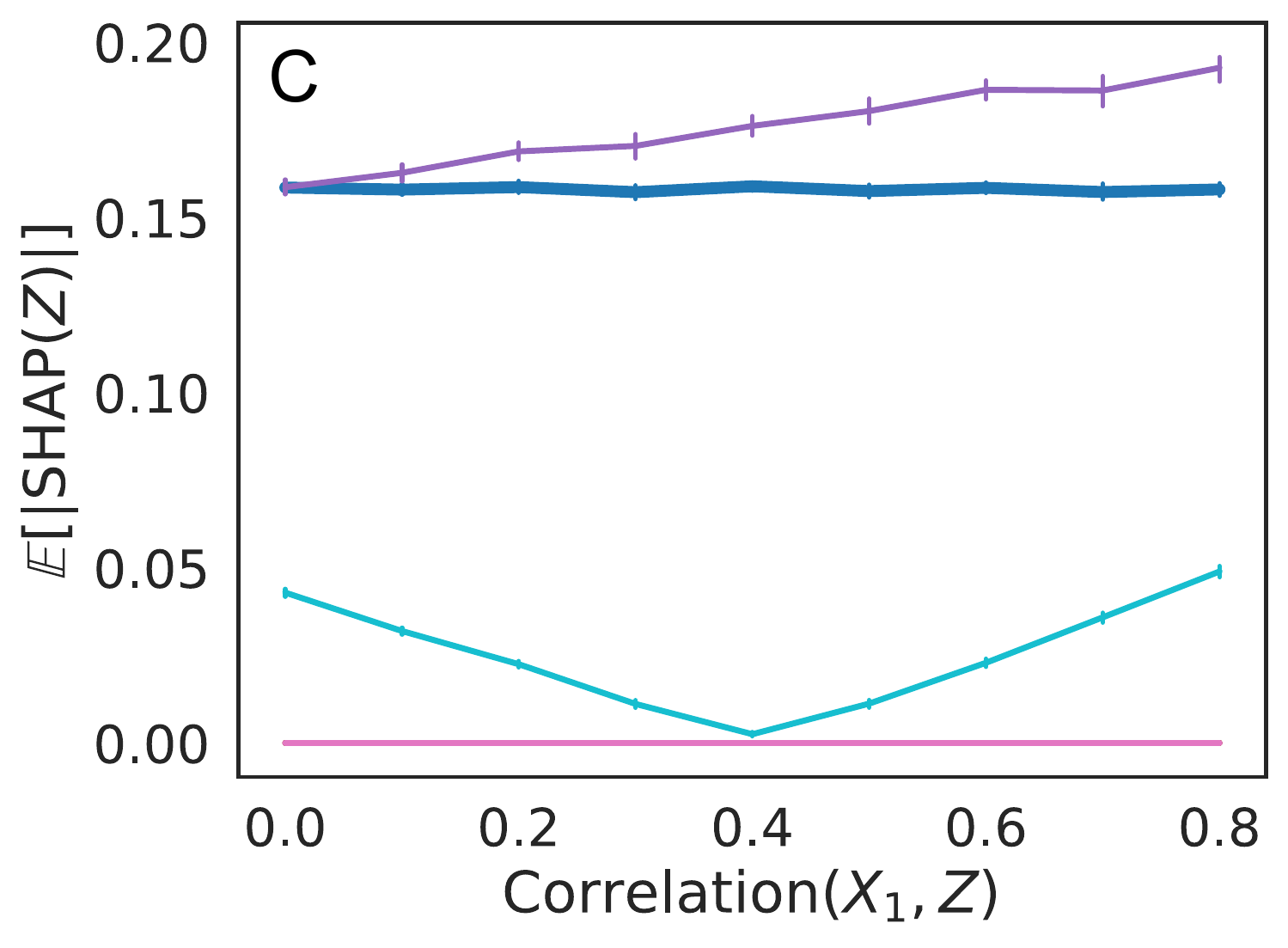}
  \includegraphics[width=.195\linewidth]{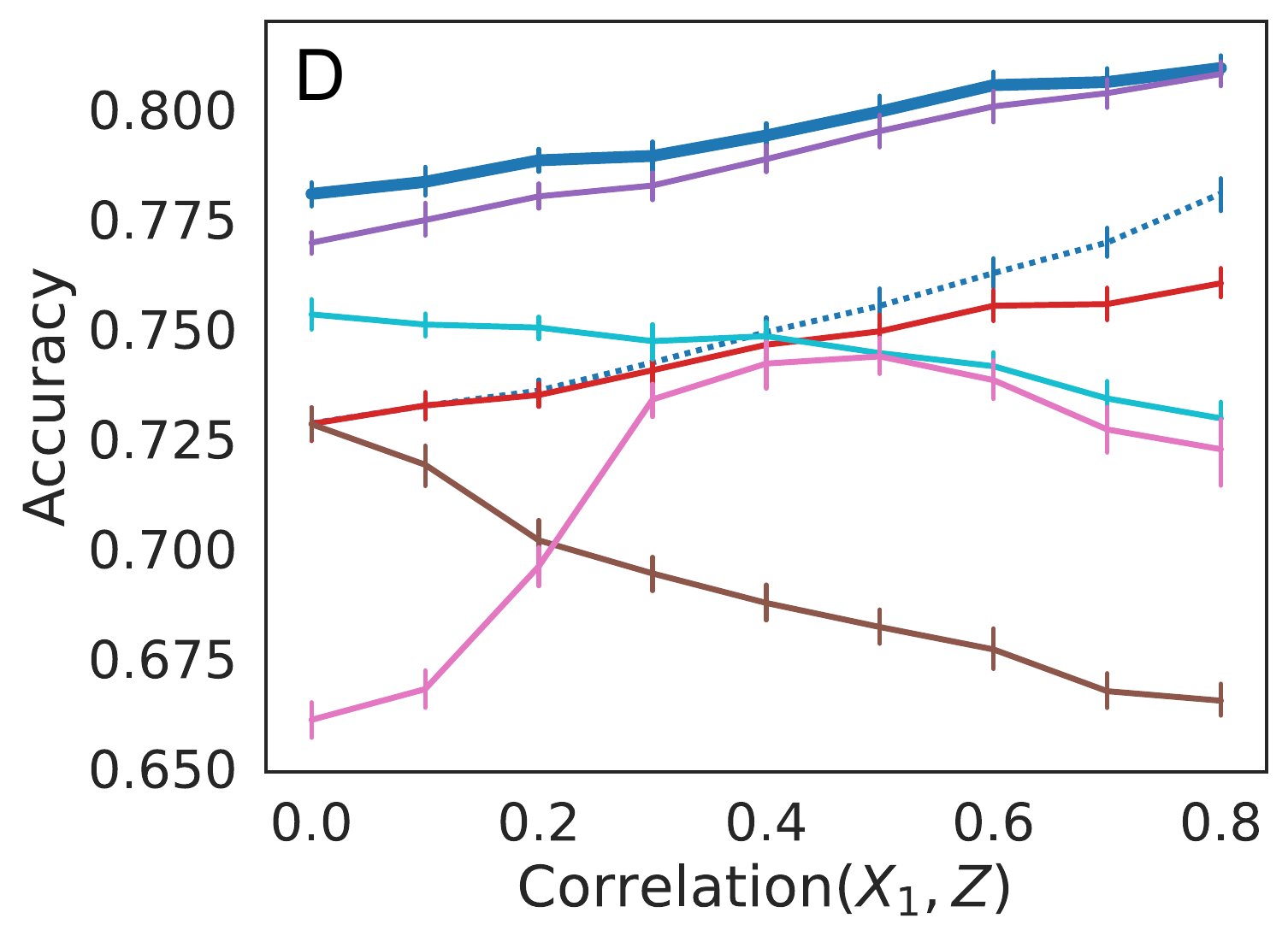}
  \includegraphics[width=.195\linewidth]{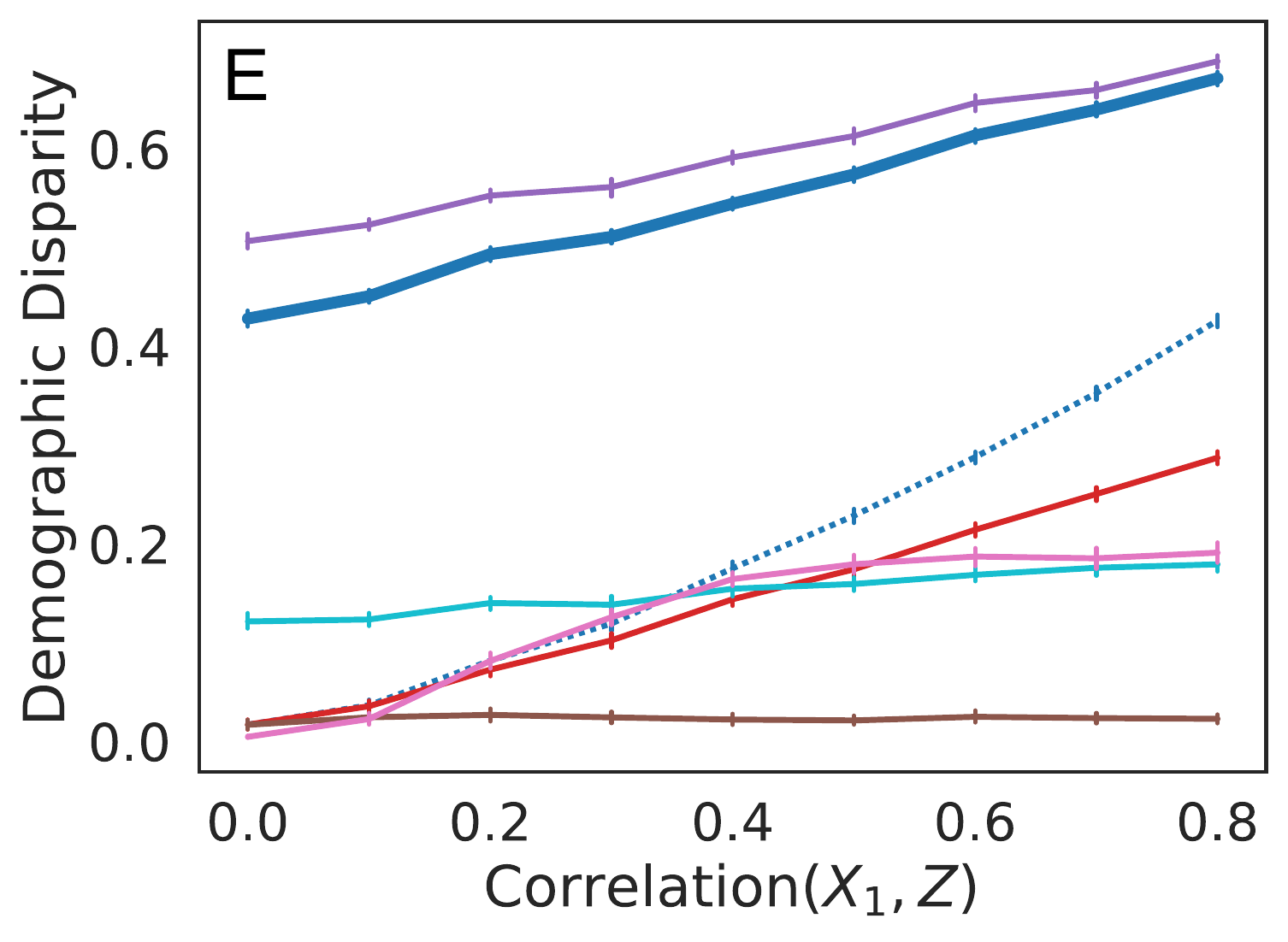}
\vspace{0.1cm}
\end{subfigure}
\begin{subfigure}{\linewidth}
  \centering
  \includegraphics[width=.195\linewidth]{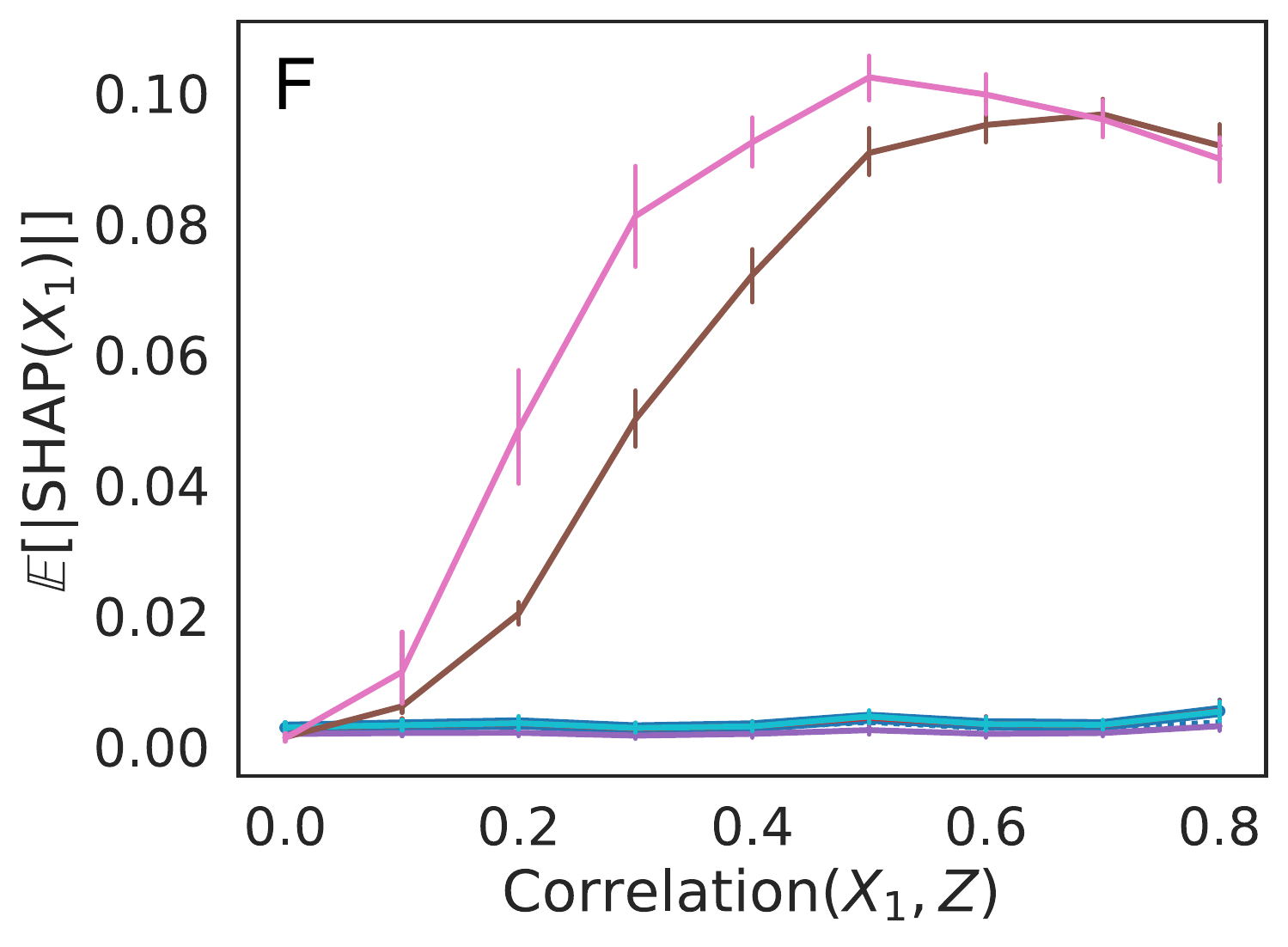}
  \includegraphics[width=.195\linewidth]{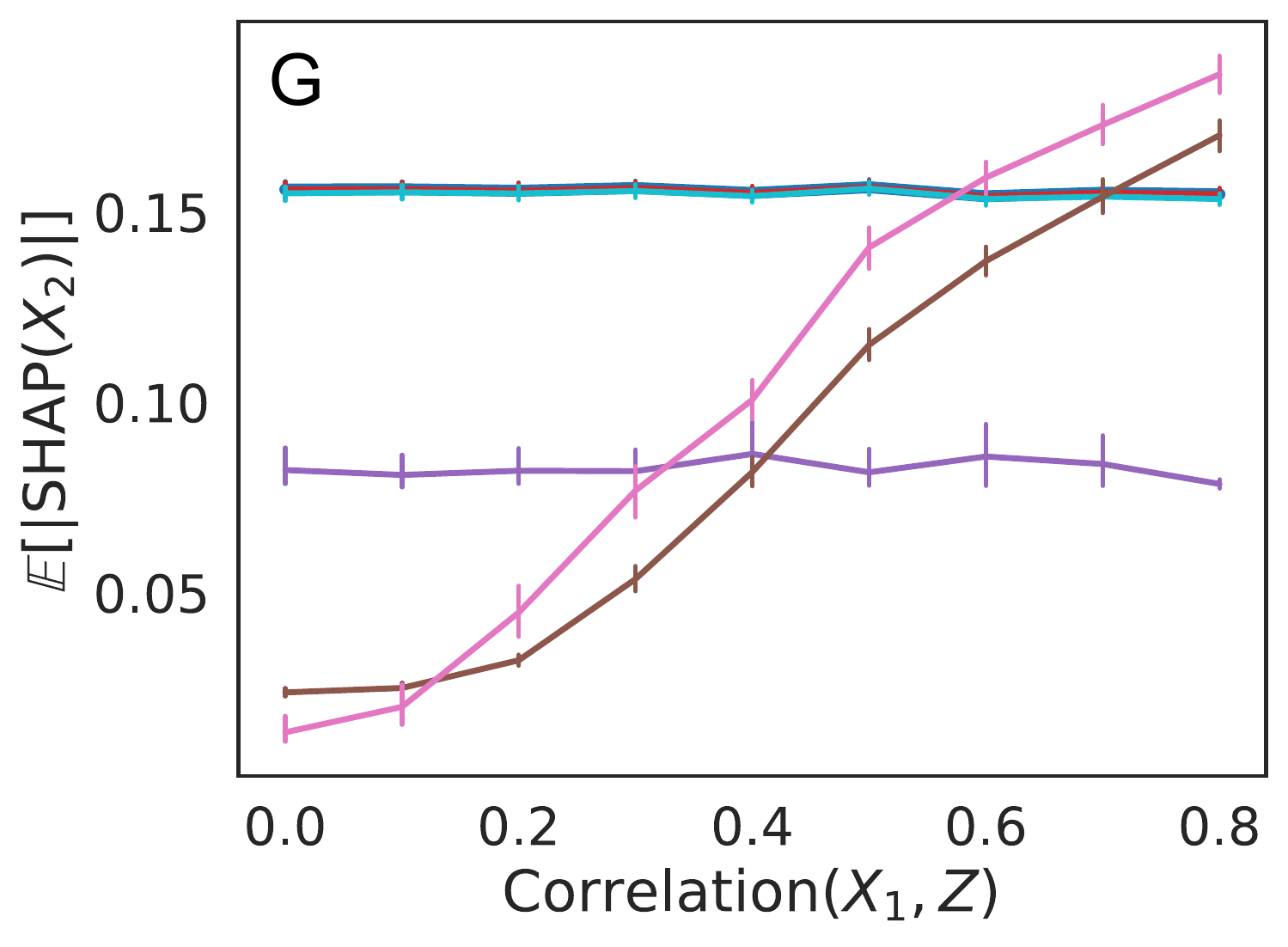}
  \includegraphics[width=.195\linewidth]{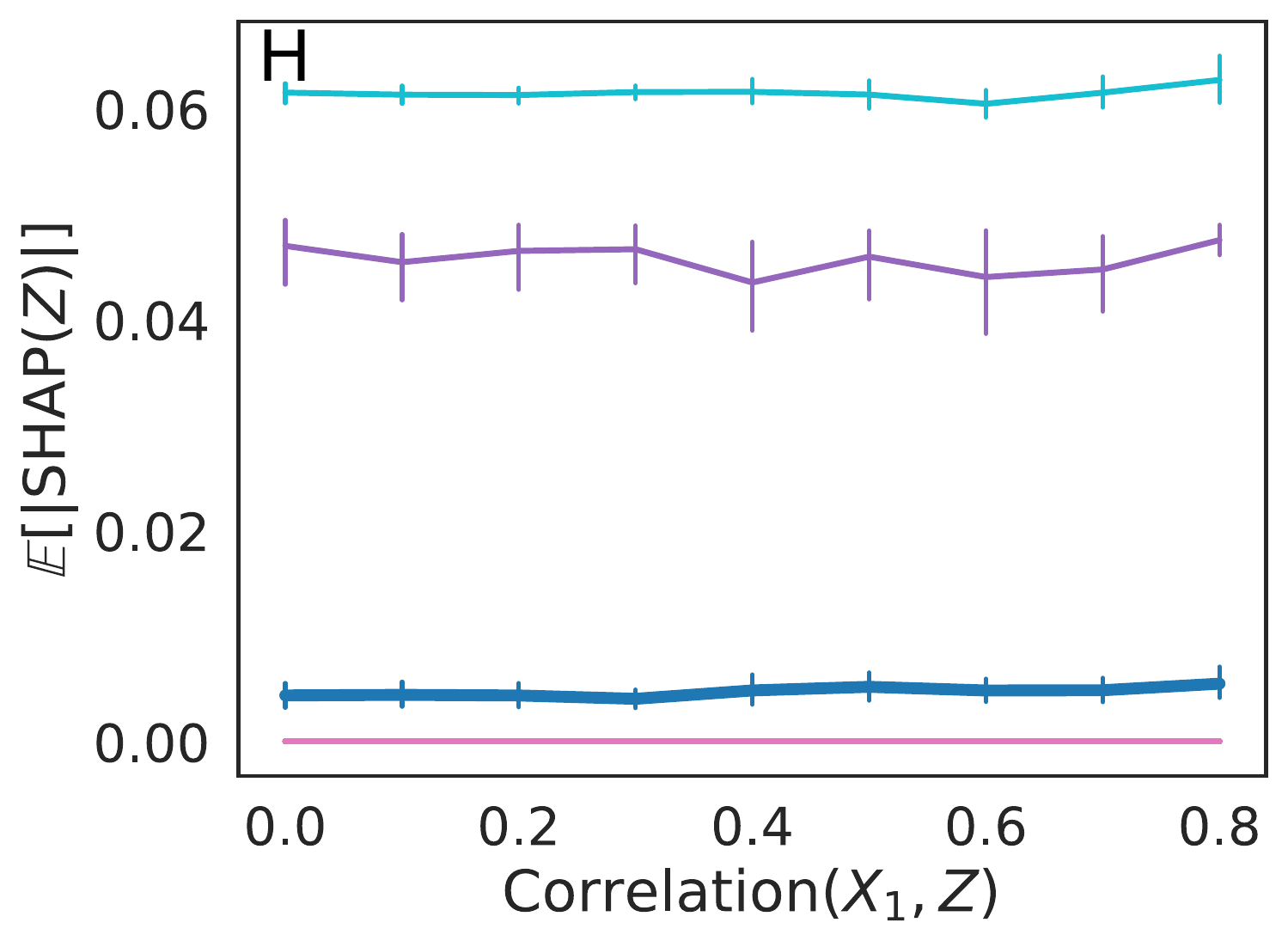}
  \includegraphics[width=.195\linewidth]{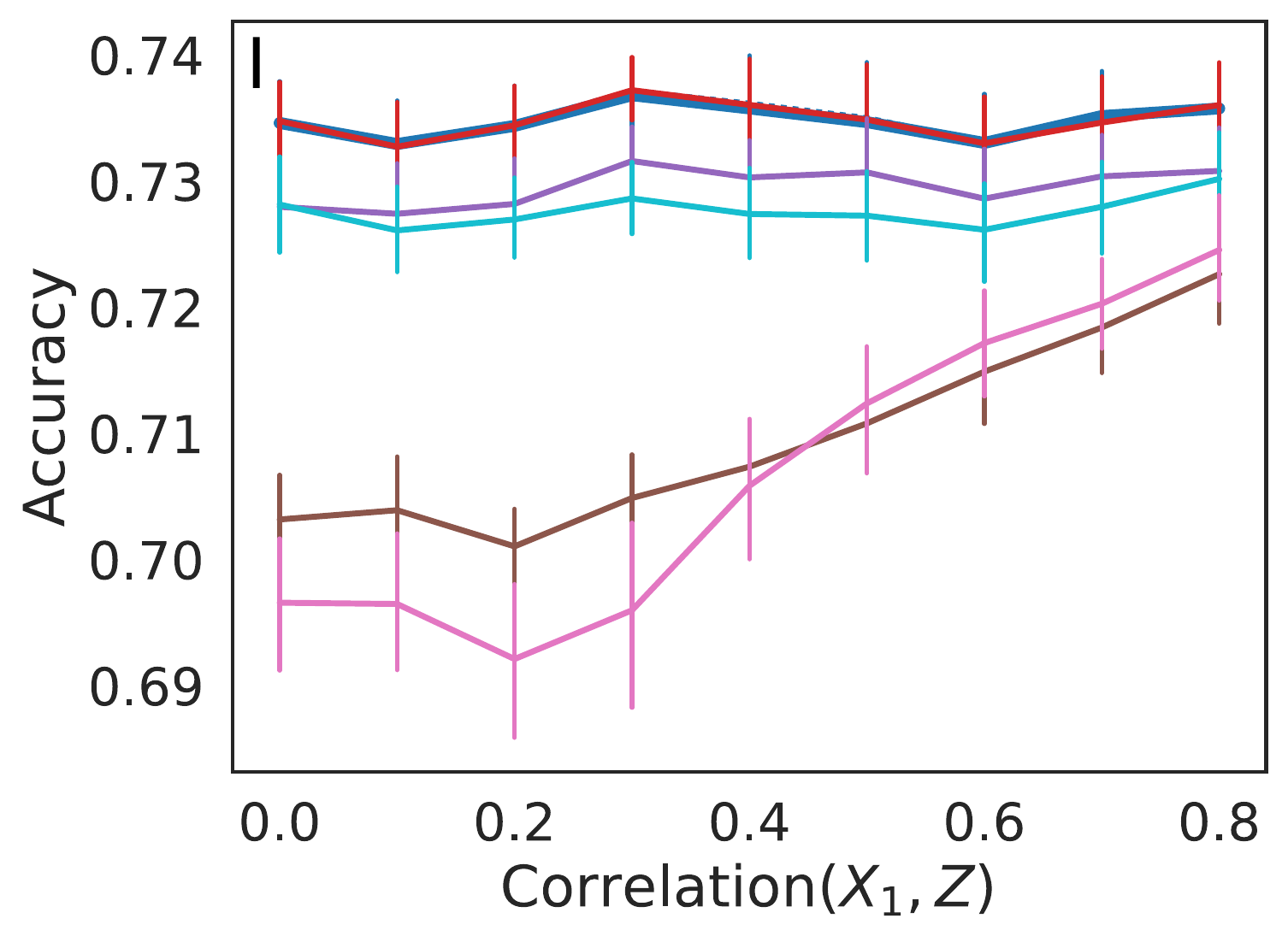}
  \includegraphics[width=.195\linewidth]{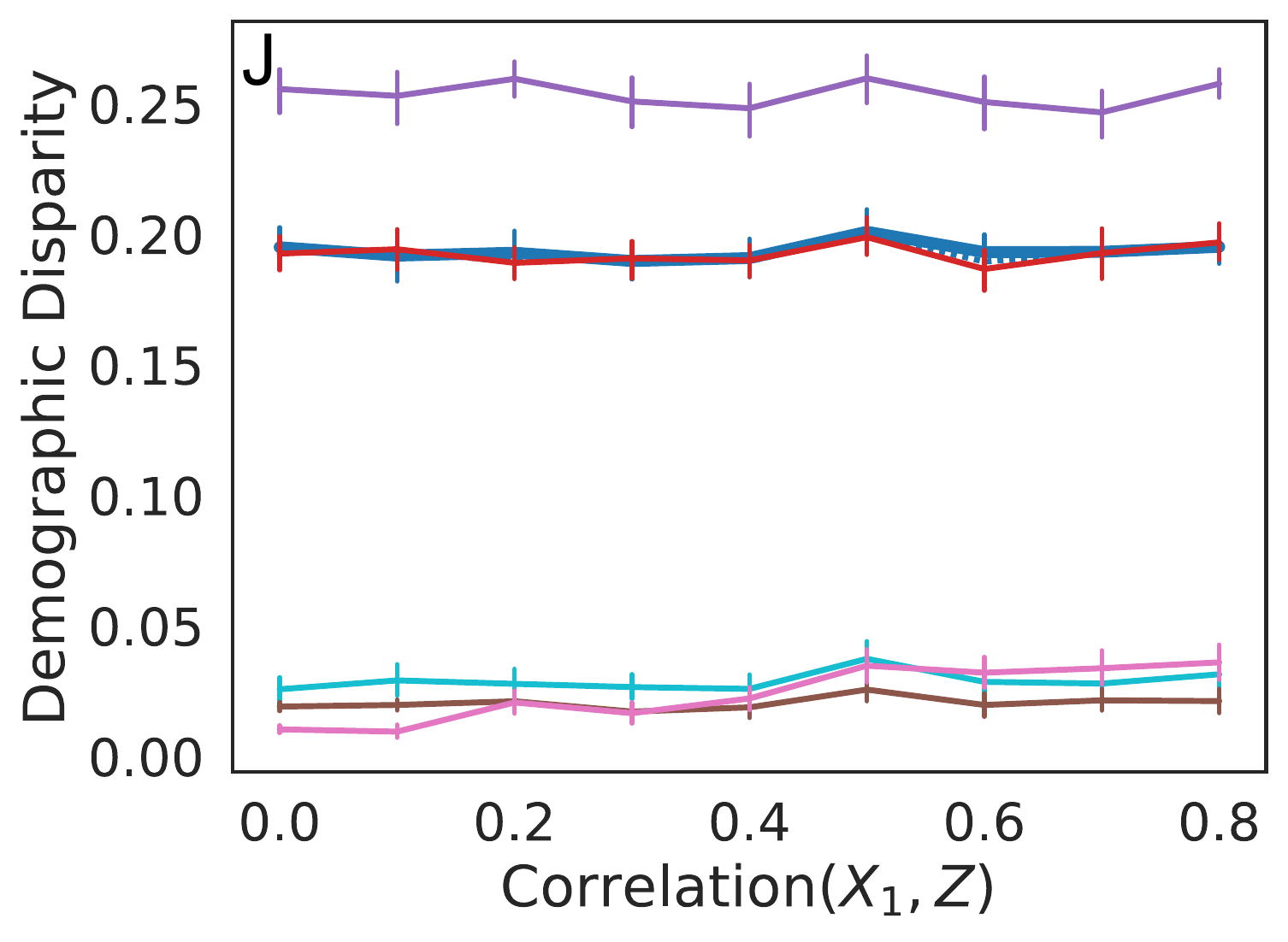}
\end{subfigure}

\caption{SHAP influence of $X_1$, $X_2$, and $Z$, model accuracy, and demographic disparity as we increase the correlation $r(X_1,Z)$ under the two synthetic scenarios: (top row) Scenario~A, $Y=\sigma(X_1 + X_2 + Z + 1)$, and (bottom row) Scenario~B, $Y=\sigma(0*X_1 + X_2+ 0*Z + 1)$. Error bars show 95\% confidence intervals based on 30 samples.}
\label{fig:synth}
\end{figure*}

\subsection{Evaluated learning methods}
\label{sec:rw-fair}

Several methods have been proposed to train machine learning models that prevent a combination of disparate treatment and impact~\cite{Pedreshi2008Discrimination, Feldman2014Certifying, Zafar2015Fairness}.
Such methods, however, can induce a discriminatory bias in model parameters~\cite{Lipton2018Does}.
Other studies propose novel mathematical notions of fairness, such as equalized opportunity, $\p(\hat{y}=1|y=1,z=0)=\p(\hat{y}=1|y=1,z=1)$, and equalized odds, $\p(\hat{y}=1|y=b,z=0)=\p(\hat{y}=1|y=b,z=1)$, $b\in\{0,1\}$  \cite{Donini2018Empirical, Woodworth2017Learning, Hardt2016Equality, Pleiss2017Fairness}, or parity mistreatment, i.e., $\p(\hat{y}\neq y|z=0)=\p(\hat{y}\neq y|z=1)$~\cite{Zafar2017Fairness}. Recent works expose the \textit{impossibility} of simultaneously satisfying multiple non-discriminatory objectives, such as equalized opportunity and parity mistreatment~\cite{Chouldechova2017Fair, Kleinberg2017Inherent,Friedler2016impossibility}. 
Thus, there exist multiple supervised learning methods for addressing discrimination, but they are often mutually exclusive. We therefore evaluate four of such methods addressing different non-discriminatory objectives at each of the stages of a machine learning pipeline where discrimination can be addressed: pre-processing, in-processing, and post-processing.

\textbf{Pre-processing:} Reweighing approach from \citet{KamiranReweigh2012}. Before training a given model, this approach modifies the weights of features with the goal of removing discrimination, defined as demographic disparity, by the protected feature. 

\textbf{In-processing:} Reductions model from \citet{Agarwal2018reductions} yields a randomized classifier with the lowest empirical error subject to a given fairness constraint. We evaluate four variations of reductions constraining on demographic parity, equalized odds, equal opportunity, and error ratio (represented as ''DP'', ''EO'', ''TPR'', and ''ER''). (2) 

\textbf{Post-processing:} Calibrated equalized odds approach from \citet{Pleiss2017Fairness} that extends \citet{Hardt2016Equality}. Building upon the prior work, calibrated equalized odds maintains calibrated probability estimates, i.e., estimates are independent of the protected attribute, while matching an equal cost constraint between the predictions of two groups. In our evaluation the constraint is a weighted combination between the false-negative and false-positive rates between the two groups in the protected attribute.     

In all cases, we use the implementations of these algorithms as provided in the AI Fairness 360 (AIF360) open-source library \cite{aif360-oct-2018}. Each of the models requires access to protected attribute during training time. The post-processing approach, calibrated equalized odds, also needs access to the protected attribute during test time. The baseline ``traditional'' model is a result of standard supervised learning. Underlying classifier for all the evaluated models is logistic regression. We also evaluate a logistic regression model that drops the protected attribute, $Z$, before training. In the figures these models are abbreviated as ``Trad'': standard supervised learning, ``Exp Grad'': reductions model, and ``Cal Eq Odds'': calibrated equalized odds.

\subsection{Synthetic datasets}
\label{sec:synthetic}
To generate the datasets we draw samples from a multivariate normal distribution with standard normal marginals and given correlations. We then convert a column of our matrix into binary values, set that as $Z$, and set the rest as $\X$. We compare the learning methods while increasing the correlation $r(X_1,Z)$ from $0$ to $1$. We first introduce and study Scenario~A, $Y=\sigma(X_1 + X_2 + Z + 1)$, where $\sigma$ is the logistic function and the correlations between both ($X_1$, $X_2$) and ($X_2$, $Z$) are zero. Then, we have Scenario~B of $Y=\sigma(0*X_1 + X_2+ 0*Z + 1)$ where the correlation between ($X_2$, $Z$) is $0.5$.

\subsubsection{Comparison of introduced methods}

As the MIM and the two OPT methods minimize loss functions based on the preservation of the influence of non-protected attributes, the resulting models perform comparably (red and two orange lines in Figure \ref{fig:aifsynset1}).
All introduced methods achieve their objectives (compare them against the blue lines in Figure \ref{fig:aifsynset1}), i.e., they all remove the influence of $\Z$ (Figures~\ref{fig:aifsynset1}C, \ref{fig:aifsynset1}H), the MIM preserves the influence of pooled $\X$ (Figure~\ref{fig:aifsynset1}D), the OPT-MDE preserves the MDE of individual $\X_i$ (Figures~\ref{fig:aifsynset1}A, \ref{fig:aifsynset1}B), and the OPT-SHAP preserves the SHAP of individual $\X_i$ (Figures~\ref{fig:aifsynset1}F, \ref{fig:aifsynset1}G).
Interestingly, the MIM performs nearly the same as the OPT-SHAP across all measures, despite not being designed to achieve the feature-specific loss of OPT-SHAP (Equation~\ref{eq:loss-shap-feature-specific}). Since the MIM is guaranteed to preserve the SHAP of the pooled $\X$, and SHAP meets the completeness axiom (a.k.a. additivity axiom)~\cite{Datta2016Algorithmic, Janzing2019Feature}, which says that the sum of influence of individual features equals to the influence of all features pooled together, hence the MIM can achieve both the pooled and individual objectives, as in this case. 
Note, however, that the MIM is slightly more accurate than the OPT-SHAP (Figure~\ref{fig:aifsynset1}E) at the cost of minimally higher demographic disparity (Figure~\ref{fig:aifsynset1}I) and equal opportunity difference, i.e., accuracy disparity (Figure~\ref{fig:aifsynset1}J). 
%

\subsubsection{Comparison vs. state-of-the-art methods}
Given the similarity of the MIM to the OPT methods, its computational efficiency, and for readability, here we compare only the MIM with the traditional and state-of-the-art methods (figures including OPT methods are in Appendix B).
The MIM preserve $X_1$'s influence with respect to the standard full model as $r(X_1,Z)$ increases (red and solid blue lines in Figures \ref{fig:synth}A, \ref{fig:synth}B, \ref{fig:synth}F, \ref{fig:synth}G). As expected in Scenario~A, the influence of $X_1$ increases with correlation for the traditional method that simply drops $Z$, i.e., it induces indirect discrimination via $X_1$ (dotted blue line in Figure \ref{fig:synth}A). In the remainder of the paper we report results for the SHAP influence, since the results for MDE are qualitatively the same (Appendix C). Importantly, even though the MIM does not optimize for any fairness measure, it performs better in demographic disparity (Figure \ref{fig:synth}E) and all other fairness measures (Appendix B) than the traditional method dropping~$Z$. 

Other methods addressing discrimination either change the influence of $X_1$ with the growing correlation $r(X_1,Z)$ (``Exp Grad'' methods in Figure \ref{fig:synth}) or
use the protected attribute $Z$ and thus discriminate directly ("Cal Eq Odds" and "Reweighing" methods in Figure \ref{fig:synth}). 
On the one hand, the method optimizing for parity of impact (``Exp Grad DP'') in Scenario~A unnecessarily decreases the influence of $\X_1$ (brown line in Figure \ref{fig:synth}A), which leads to an accuracy loss (Figure \ref{fig:synth}D), because its goal is to remove the correlation between $\hat{Y}$ and $Z$. 
In Scenario~B, the changes in the influence of $X_1$ with the growing correlation are especially noteworthy. The affected methods (``Exp Grad'') are increasingly influenced by $X_1$ as it gets more associated with the protected attribute (Figure \ref{fig:synth}F), despite $X_1$ not having impact on $Y$, because this enables them to increasingly utilize $X_2$ in their model of $Y$ (Figure \ref{fig:synth}G) and improve accuracy (Figure \ref{fig:synth}I) under a respective fairness constraint.
Other reductions approaches, constrained on equal opportunity and error ratio, yield similar outcomes (Appendix B).
On the other hand, the methods allowing the influence of $Z$ perform relative well in Scenario~A, because they counteract discrimination by using $Z$ directly (violet and teal lines in Figures~\ref{fig:synth}C, \ref{fig:synth}H) to maintain stable influence of $X_1$ and $X_2$ on $\hat{Y}$ (Figures~\ref{fig:synth}A, \ref{fig:synth}B, \ref{fig:synth}F, \ref{fig:synth}G) and a high model accuracy (Figures~\ref{fig:synth}D, \ref{fig:synth}I), independently of $r(X_1,Z)$.
However, in Scenario~B, where there is no discrimination, these methods introduce reverse discrimination to counteract the correlation between $X_2$ and $Z$, without considering the possibility that this correlation is a \textit{fair relationship}, and achieve lower accuracy than the MIM (Figure~\ref{fig:synth}I).

\begin{figure*}[h!]
\begin{subfigure}[b]{\textwidth}
  \centering
  \includegraphics[width=0.24\linewidth]{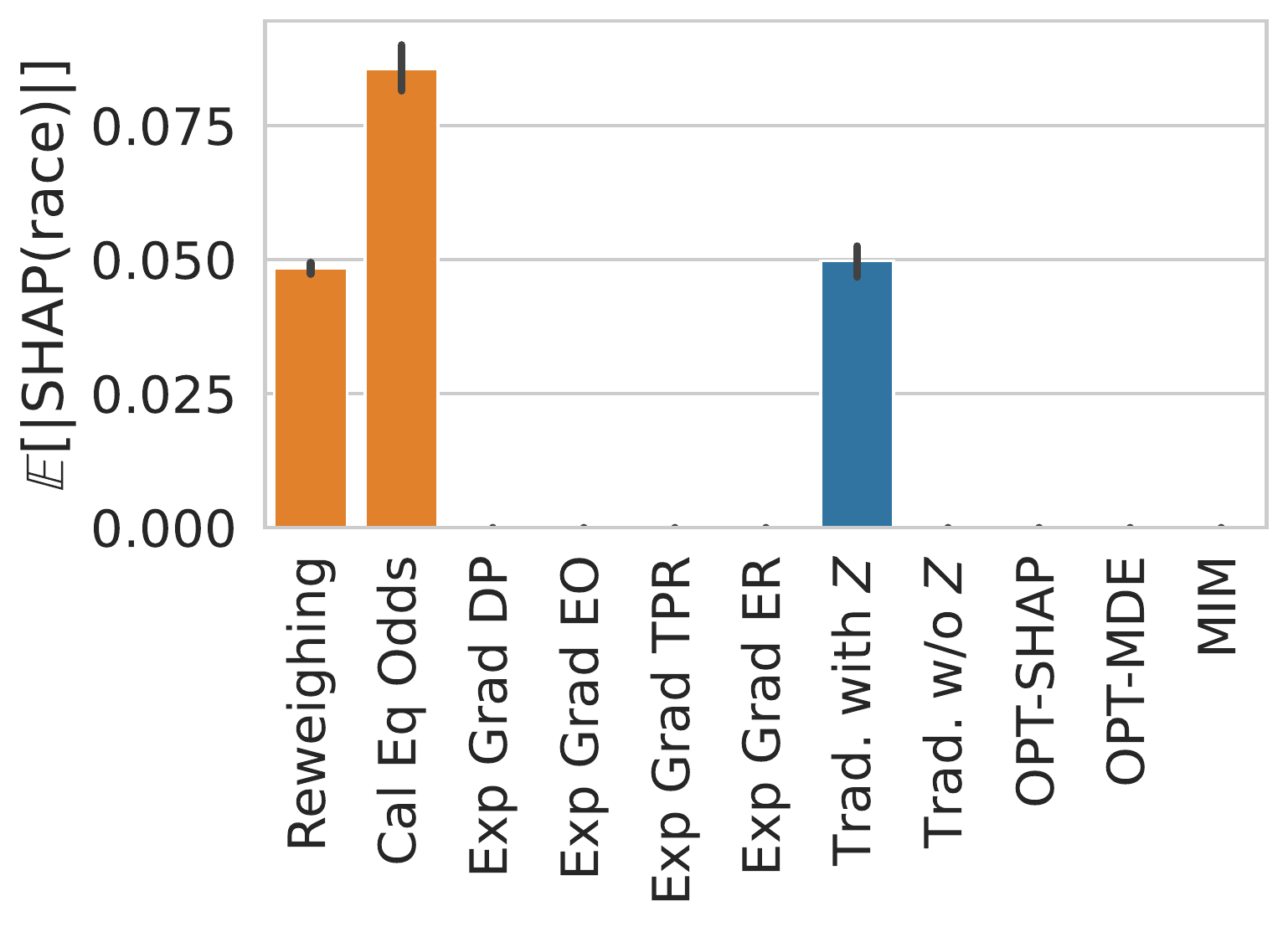}
  \includegraphics[width=0.24\linewidth]{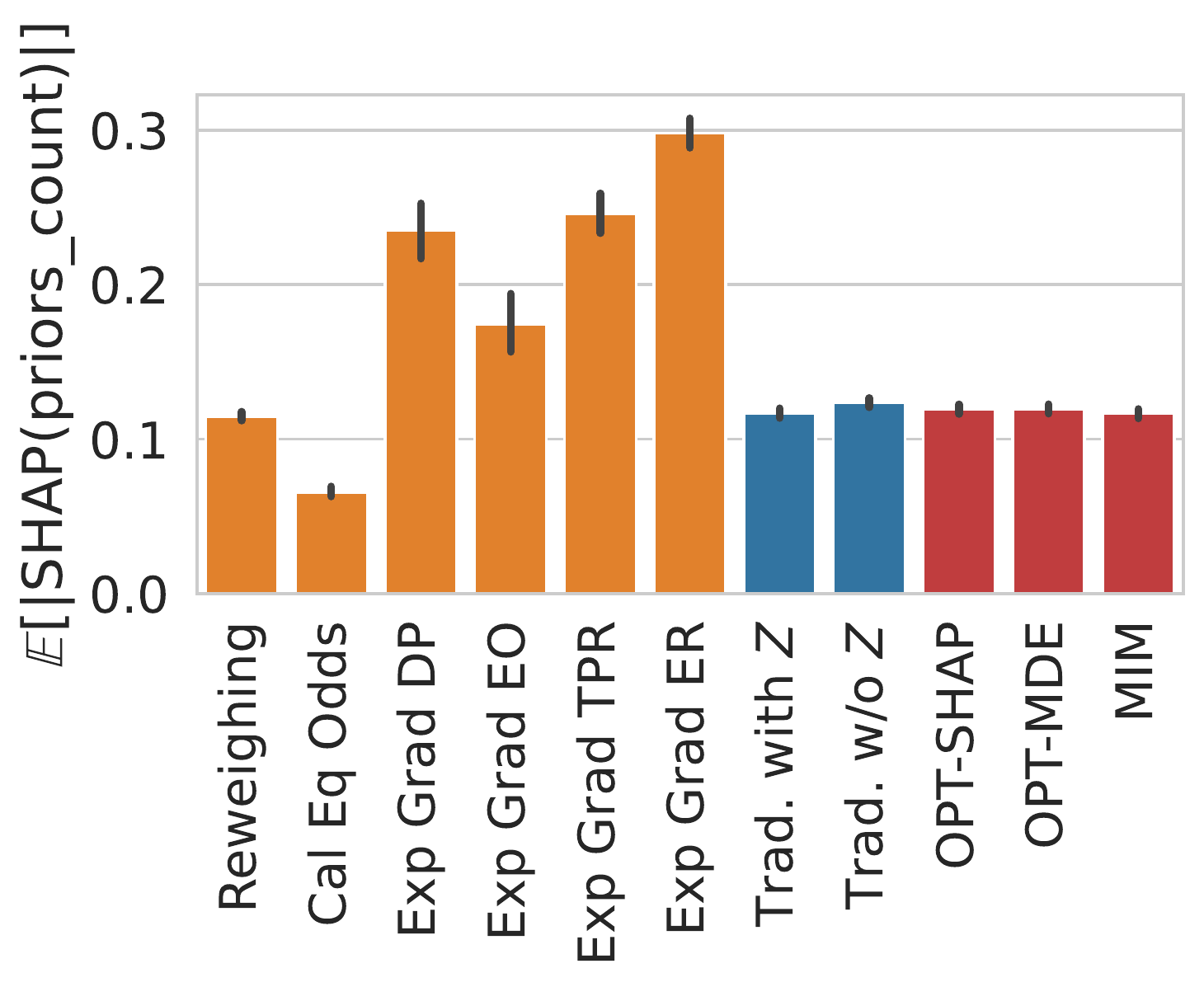}
  \includegraphics[width=0.24\linewidth]{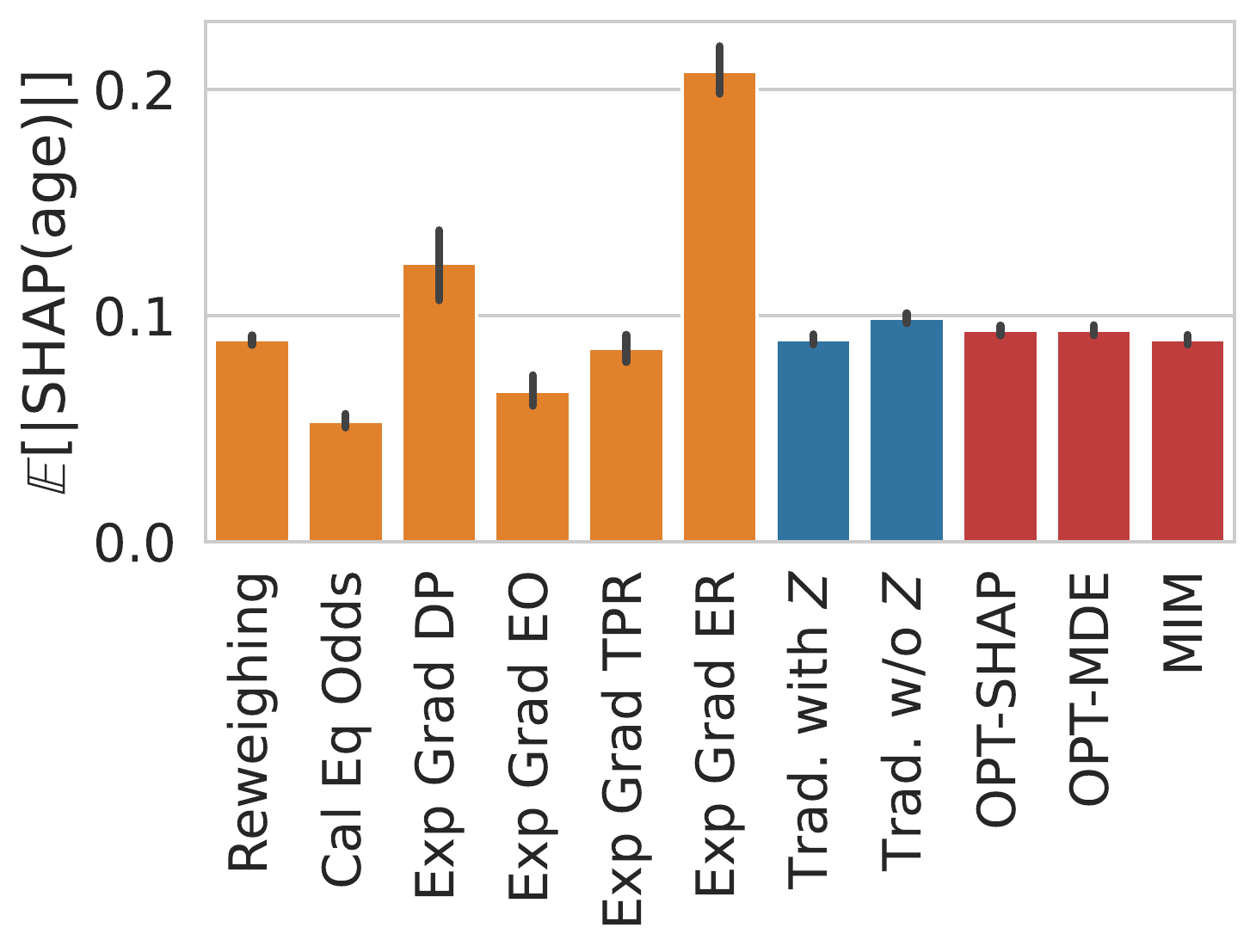}
  \includegraphics[width=0.24\linewidth]{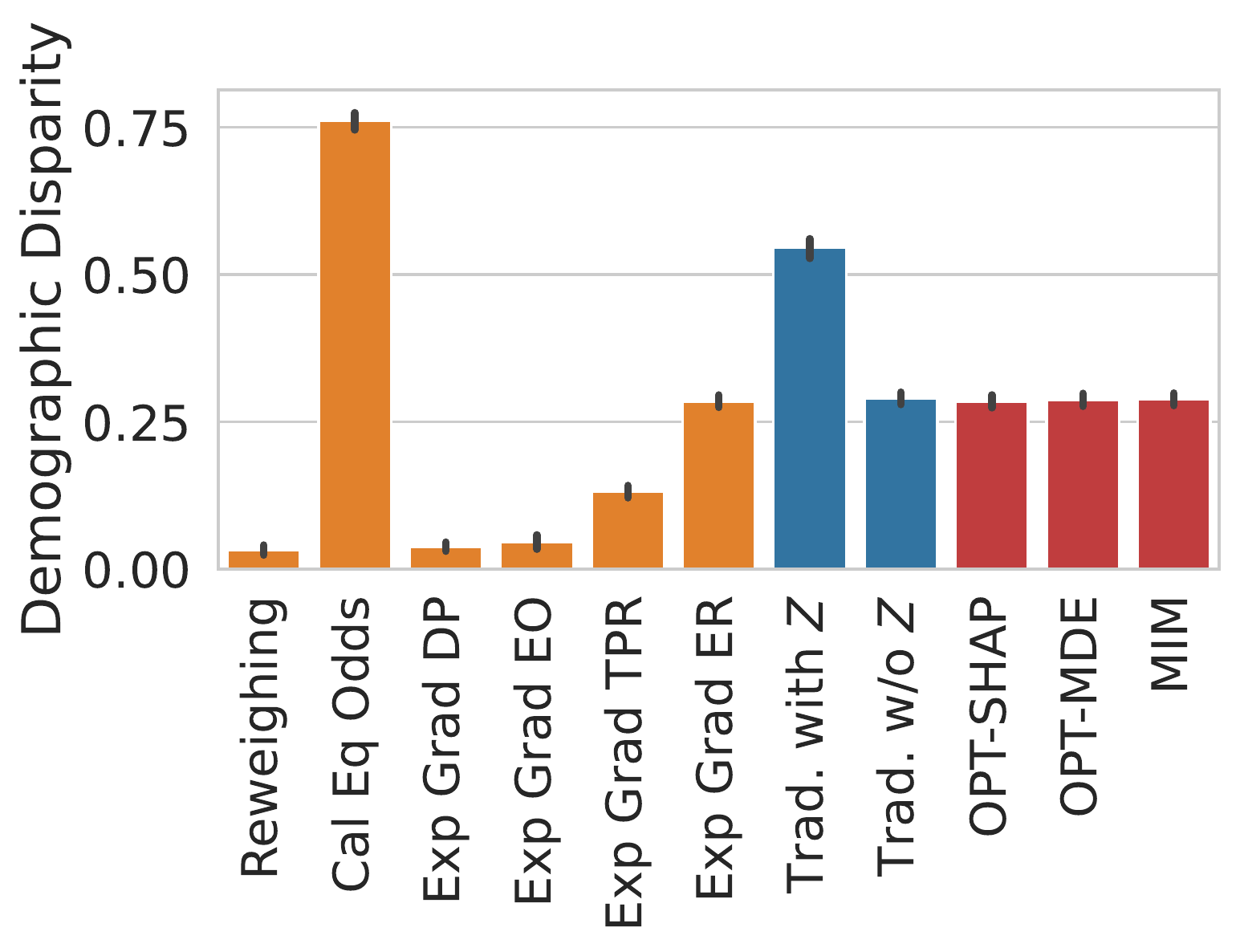}
  \caption{COMPAS}
\end{subfigure}
\begin{subfigure}[b]{\textwidth}
  \centering
  \includegraphics[width=0.24\linewidth]{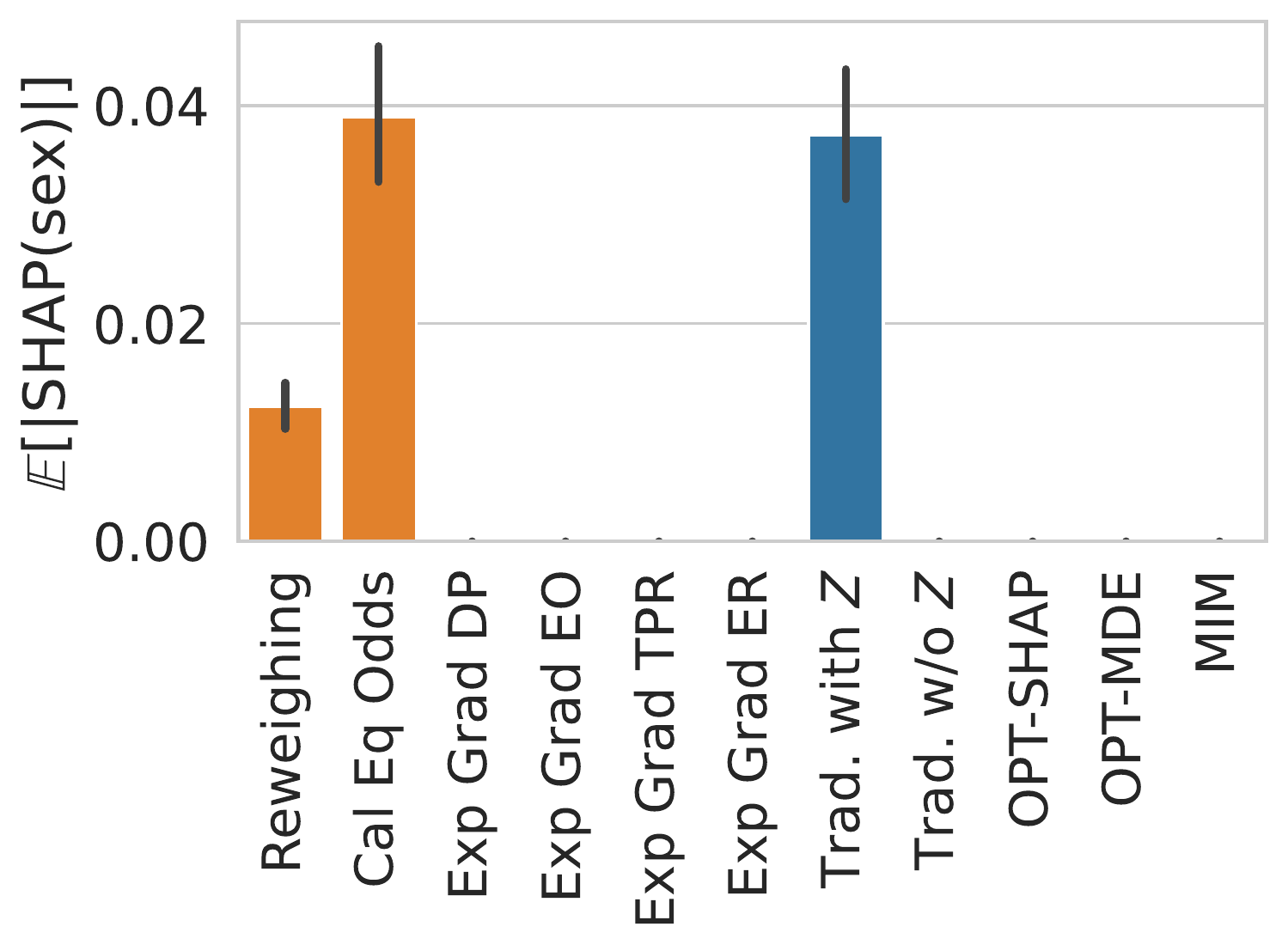}
  \includegraphics[width=0.24\linewidth]{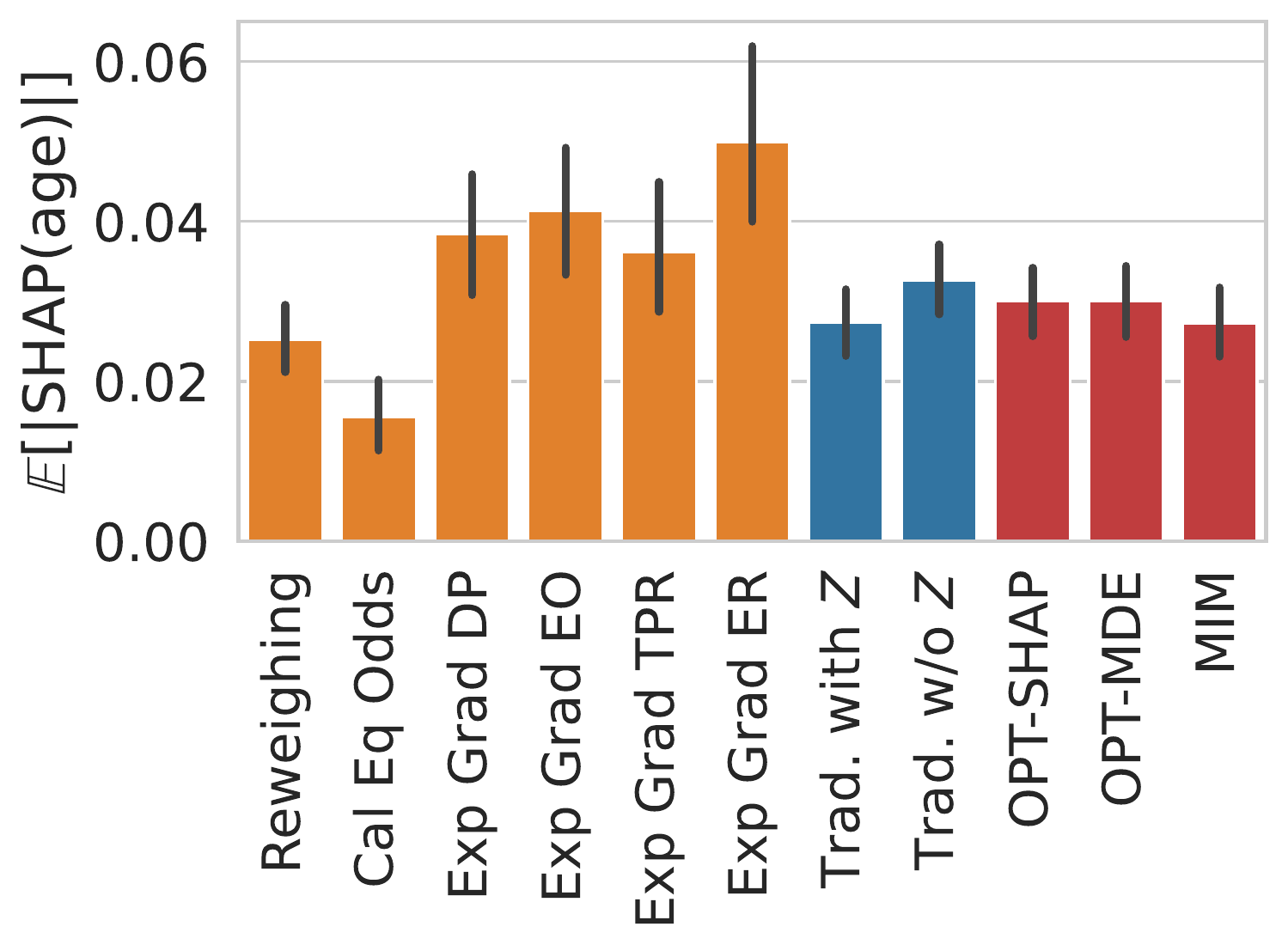}
  \includegraphics[width=0.24\linewidth]{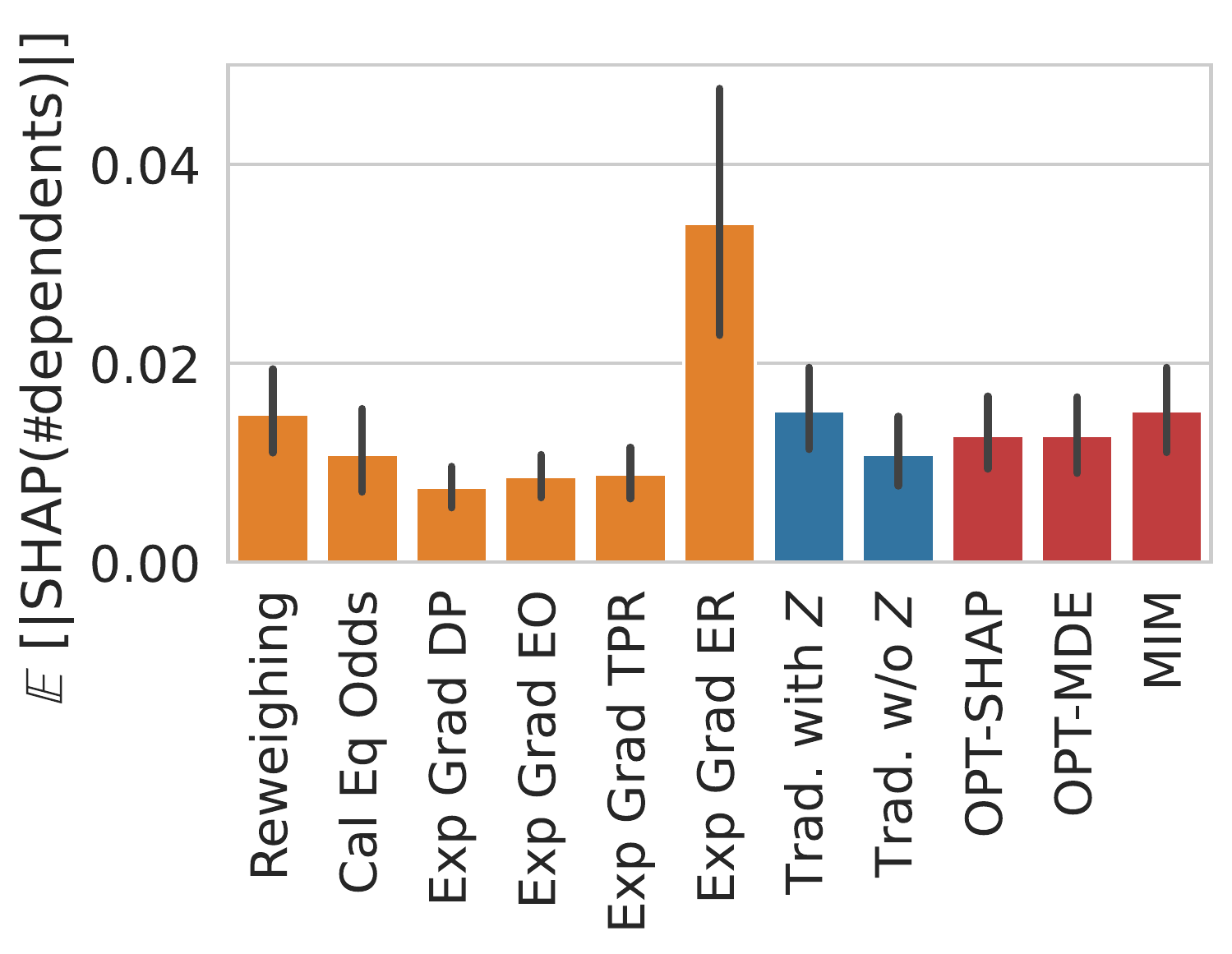}
  \includegraphics[width=0.24\linewidth]{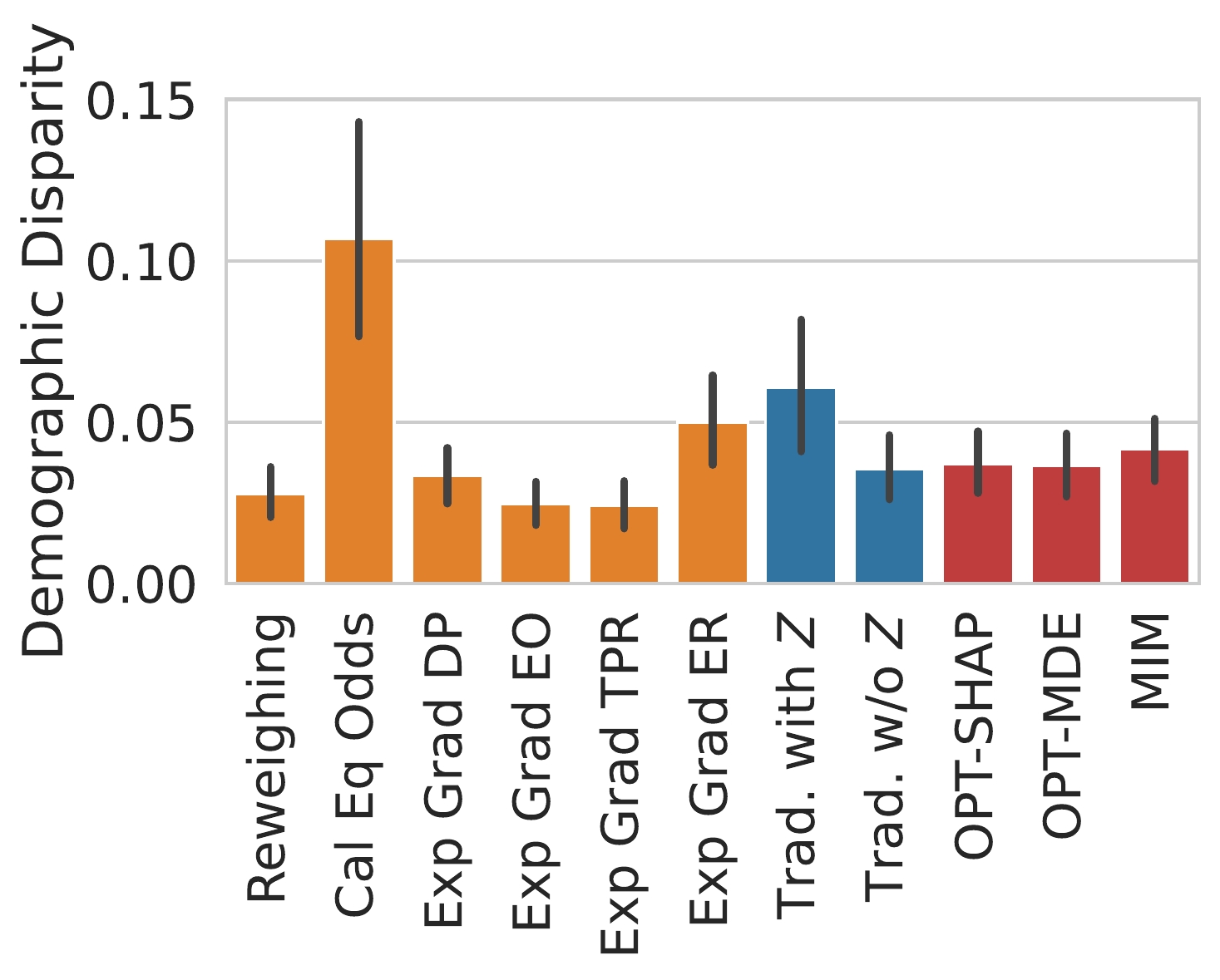}
  \caption{German Credit}
\end{subfigure}
\begin{subfigure}[b]{\textwidth}
  \centering
  \includegraphics[width=0.24\linewidth]{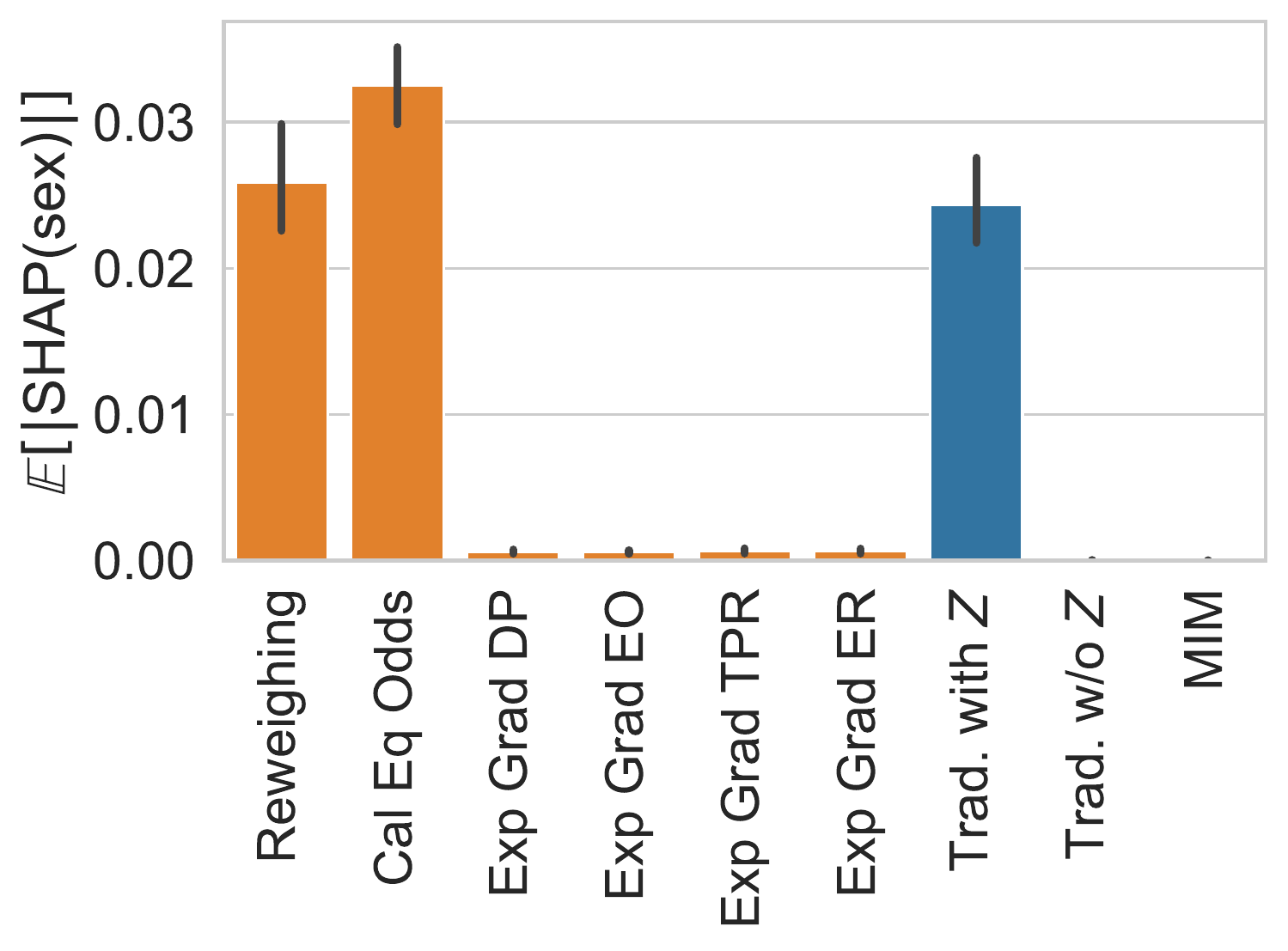}
  \includegraphics[width=0.24\linewidth]{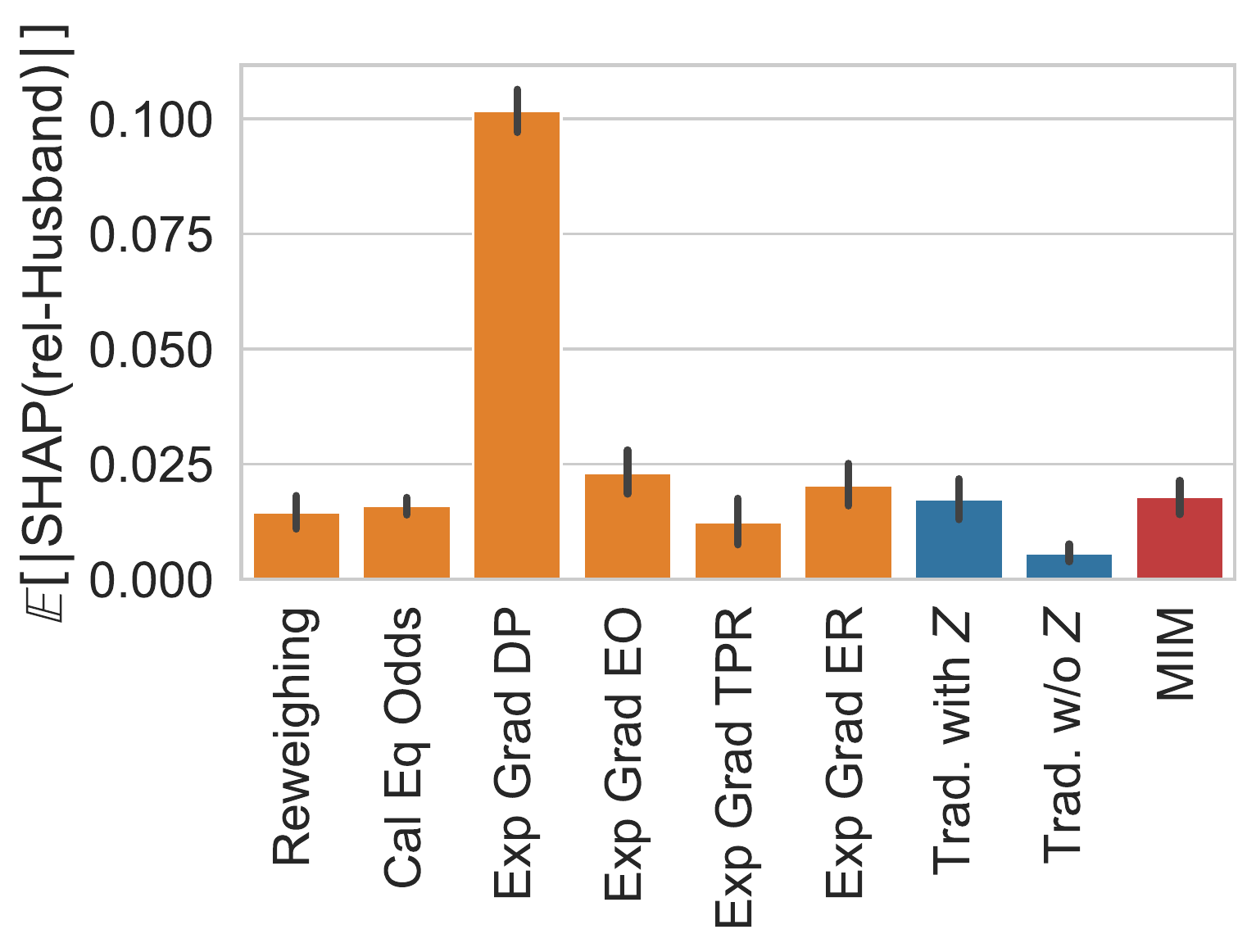}
  \includegraphics[width=0.24\linewidth]{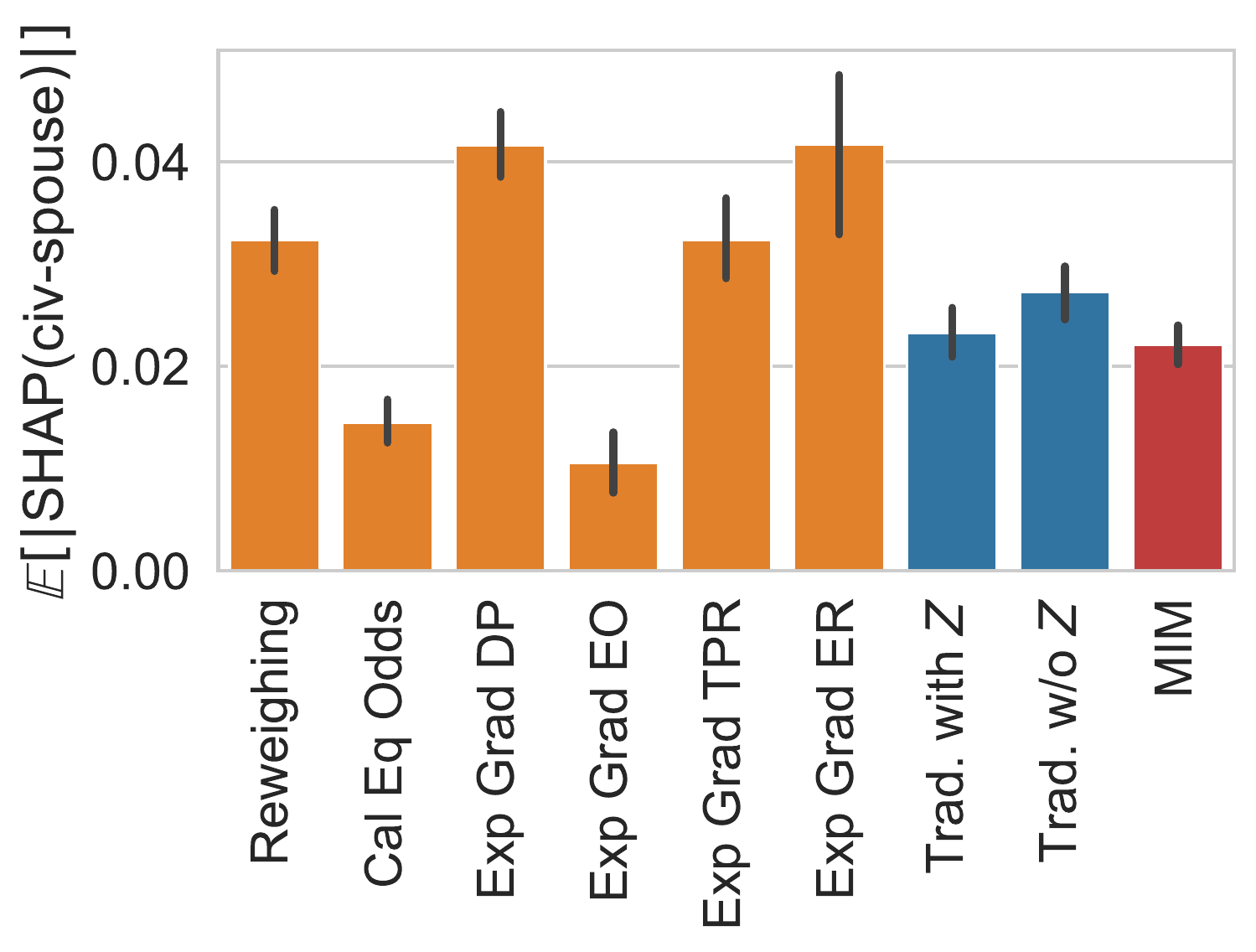}
  \includegraphics[width=0.24\linewidth]{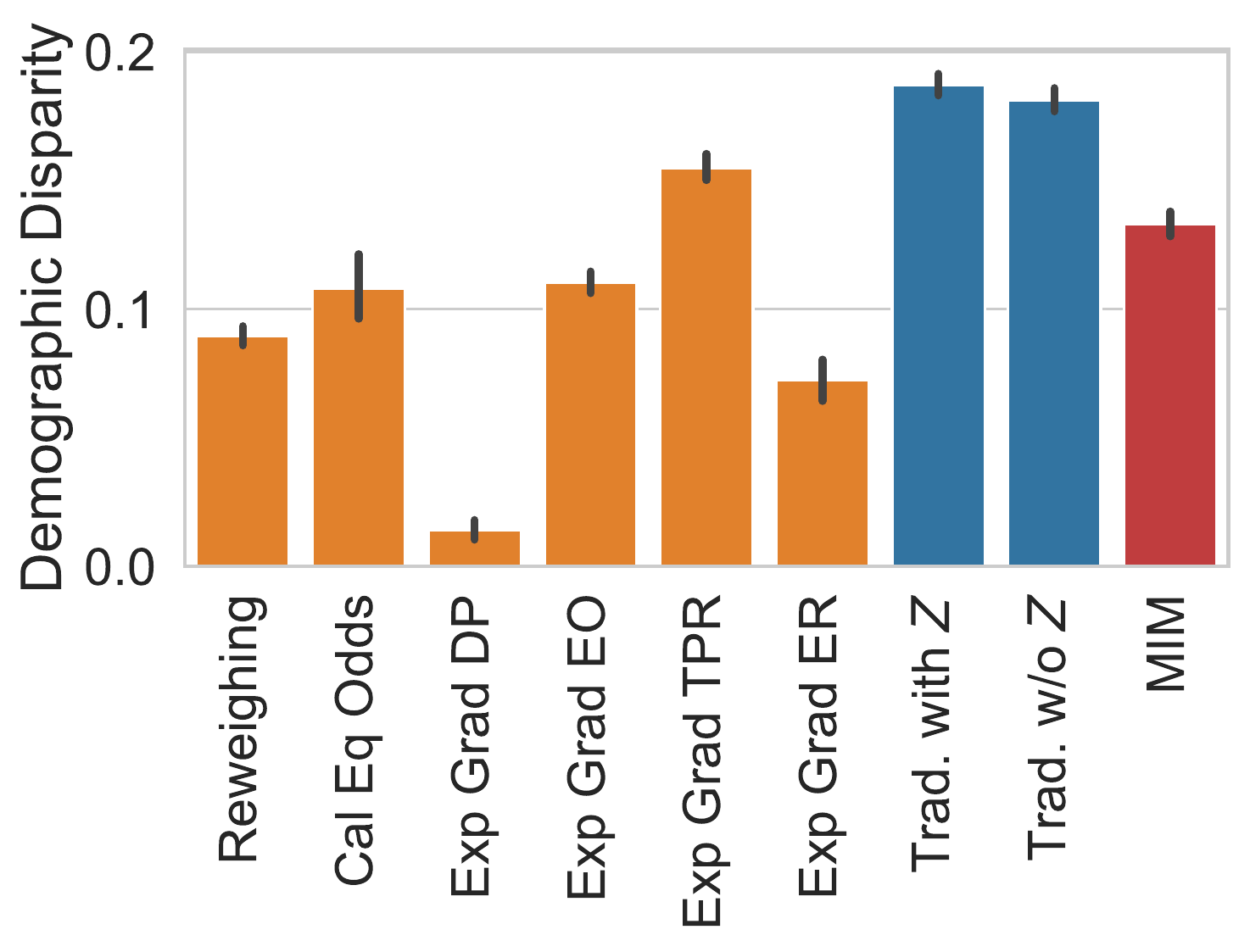}
  \caption{Adult Census Income}
\end{subfigure}

\caption{Averaged absolute SHAP for the protected attribute and two features most correlated with it and demographic disparity on the COMPAS, German Credit, and Adult Census Income datasets. Error bars show 95\% confidence intervals.}
\label{fig:aifcompas}
\end{figure*}
\subsection{Real-world datasets}
We train and test (80:20 random split) the evaluated methods on the COMPAS criminal recidivism dataset \cite{Larson2016How}, German Credit, and Adult Census Income \cite{Dua:2019} datasets popular in machine learning fairness research.

\begin{itemize}
    \item \textbf{COMPAS.} Here the model predicts the recidivism of an individual based on their demographics and criminal history with race being the protected attribute. We use the binary outcomes as provided by \citet{aif360-oct-2018}. To make the presentation more clear, we exacerbate the racial bias by removing 500 samples of positive outcomes (no recidivism) for African-Americans. The two attributes most correlated with race are age and number of prior counts.  
    \item \textbf{German Credit.} A financial dataset with the task being to determine if a loan applicant's credit risk is ``good'' or ``bad'' using sex as the protected attribute. We drop non-numeric attributes leaving information about the loan applicant's job, household, and the sought loan. The two attributes most correlated with a applicant's sex are their age and number of dependents.
    \item \textbf{Adult Census Income.} The task for this dataset is to determine if someone's annual income is more than \$50k with sex being the protected attribute. Other attributes give information about a person's education, job, relationships, and demographics. The two attributes most correlated with a person's sex are if they are a husband and if they have a spouse in the armed forces. Note that due to the number of features of this dataset and its effect on computation time for input influence, we omit the results of the OPT methods. 
\end{itemize}

Data loading and pre-processing functions from the AIF360 library are used for these real-world datasets \cite{aif360-oct-2018}. We train and test all the evaluated models over 30 trials for the COMPAS and German Credit datasets and 10 trials for the Adult Census Income dataset.


In line with the synthetic results, the MIM (and OPT methods) is not influenced by the protected attribute (leftmost column in Figure~\ref{fig:aifcompas}) and, with respect to the traditional model, preserves the influence for the two attributes most correlated with the protected attribute in these real-world scenarios (blue and red bars in the two middle columns of Figure \ref{fig:aifcompas}). 
While most of the evaluated models outperform the MIM in terms of demographic disparity (the rightmost column in Figure~\ref{fig:aifcompas}), they are either influenced by the protected attribute (the leftmost column in Figure~\ref{fig:aifcompas}) or do not preserve the influence of at least one of the most correlated attributes (the two middle columns in Figure~\ref{fig:aifcompas}) and have significantly lower accuracy (Figure~\ref{fig:aiferror}), e.g., ``Exp Grad'' for COMPAS (Figures \ref{fig:aifcompas}a \& \ref{fig:aiferror}a). 
As with the synthetic results, the changes in influence for the features most correlated with the protected attribute indicate that these methods induce indirect discrimination during training, despite having better performance for certain fairness measures.

\begin{figure*}
\begin{subfigure}[b]{0.25\linewidth}
  \centering
  \includegraphics[width=\linewidth]{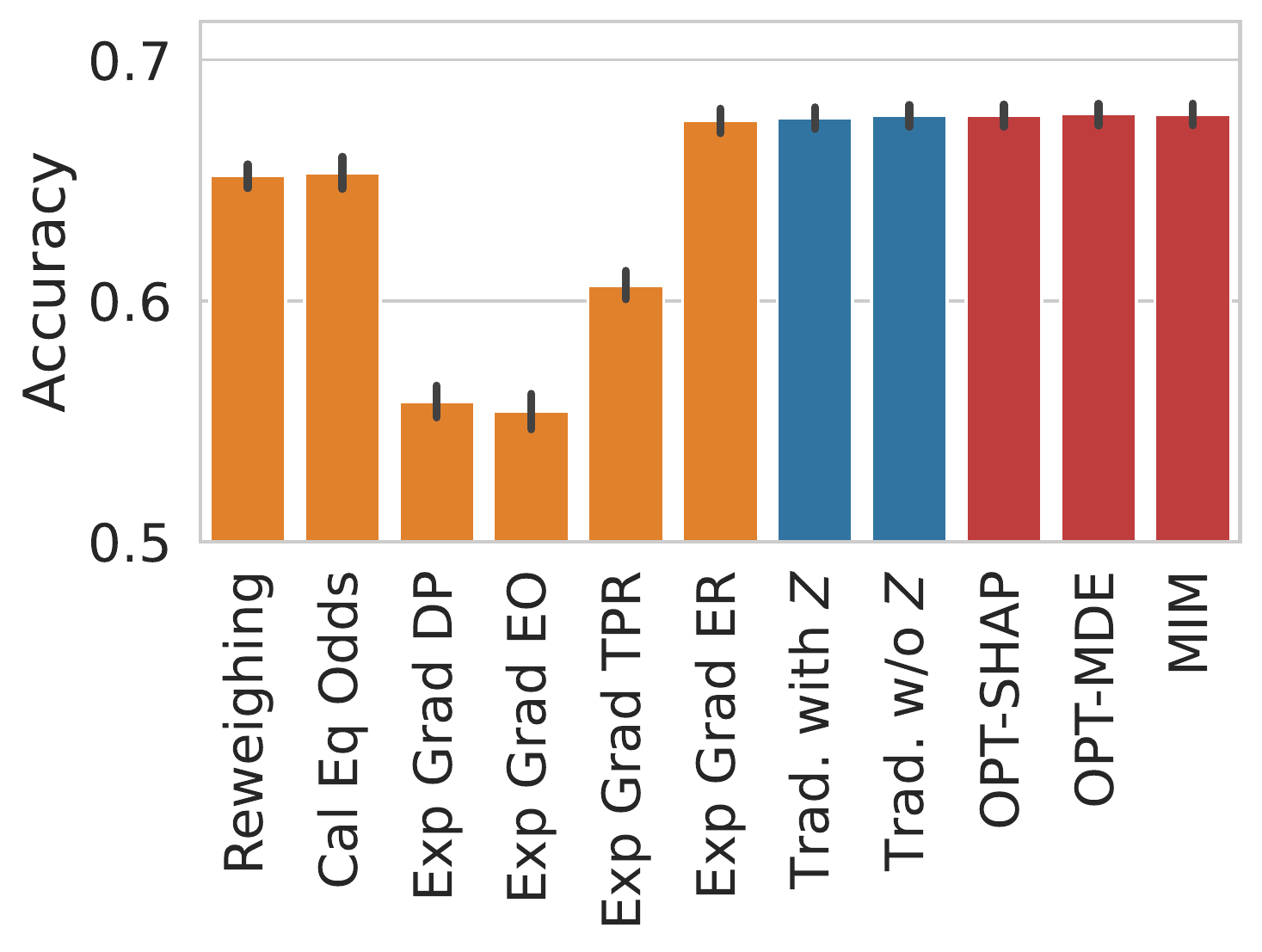}
  \caption{COMPAS}
\end{subfigure}
\begin{subfigure}[b]{0.25\linewidth}
  \centering
  \includegraphics[width=\linewidth]{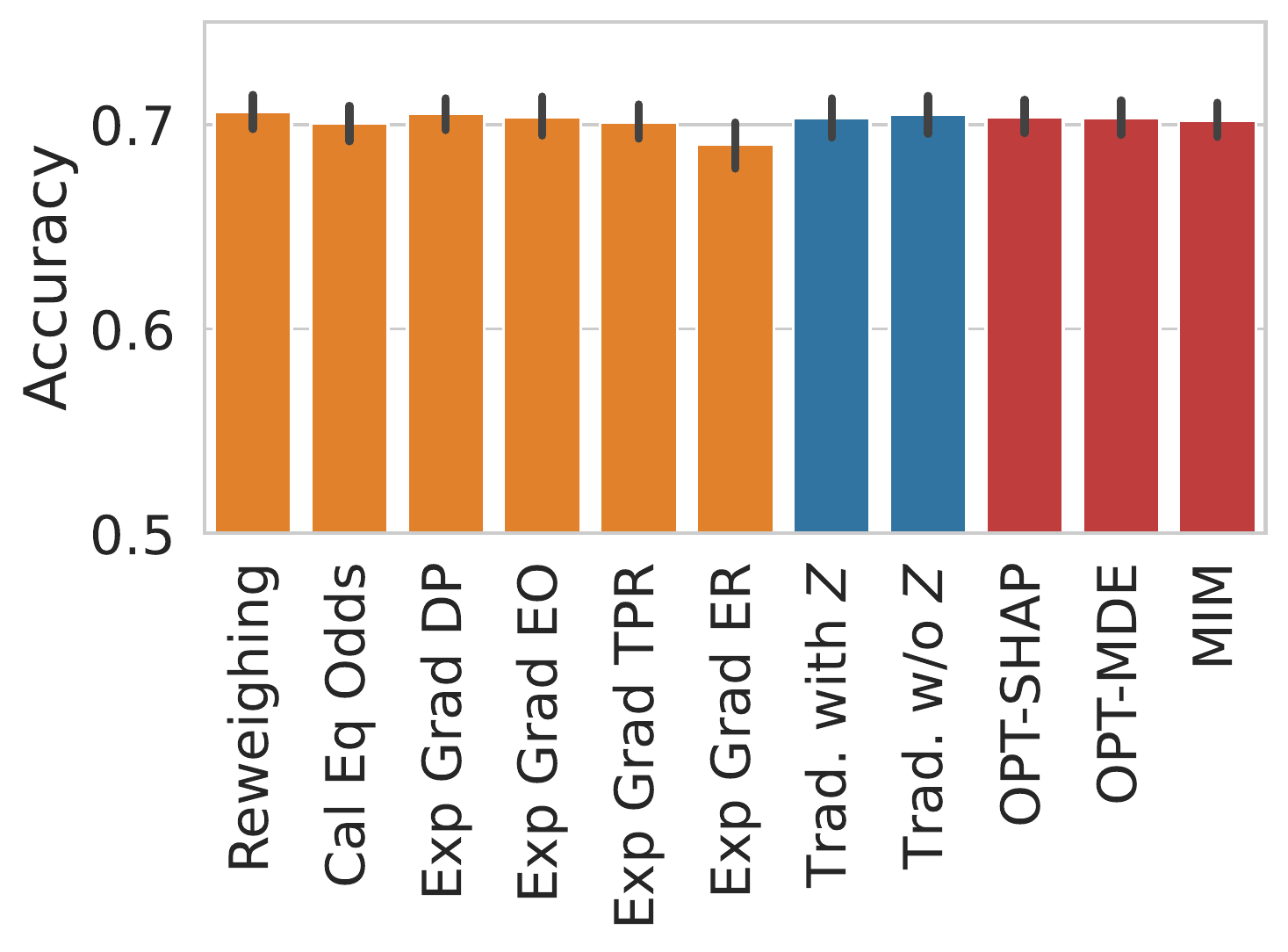}
  \caption{German Credit}
\end{subfigure}
\begin{subfigure}[b]{0.25\linewidth}
  \centering
  \includegraphics[width=\linewidth]{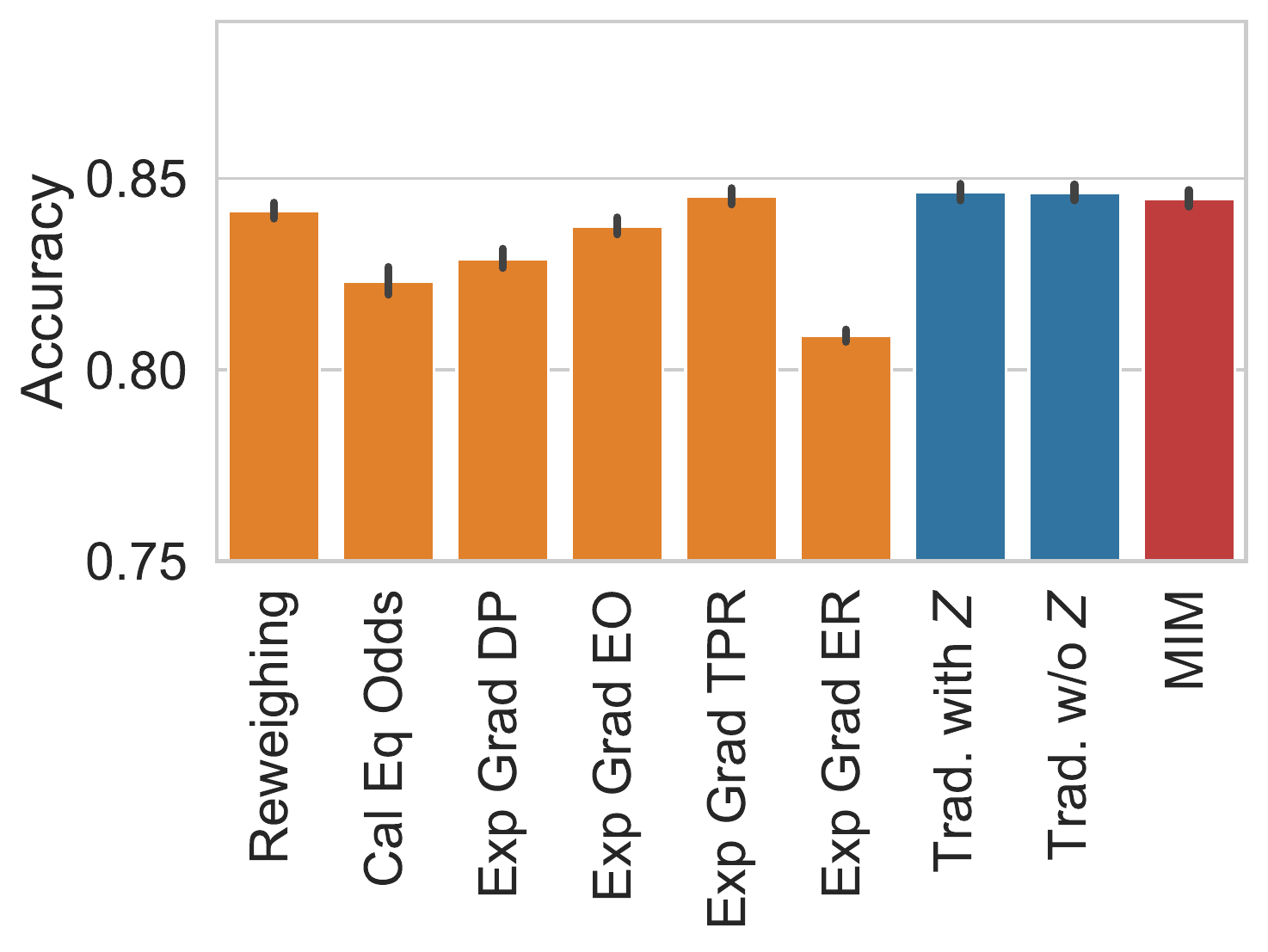}
  \caption{Adult Census Income}
\end{subfigure}

\caption{Accuracy for the evaluated models on the COMPAS, German Credit, and Adult Census Income datasets. Error bars show 95\% confidence intervals.}
\label{fig:aiferror}
\end{figure*}

\section{Limitations and future work}
This manuscript focuses on two influence measures, MDE and SHAP, and corresponding loss functions for influence preservation. Prior studies show that input influence measures like SHAP can be fooled into stating that a protected attribute has no influence on a model \cite{foolingshap}. With this, someone may be able to trick our approach into believing a model was fair by our definition, even though in reality it was not. In such adversarial scenarios, our approach may experience the limitations of other discrimination preventing methods where satisfying a specified fairness objective still leads to discrimination.
There exist many other influence measures than the two studied here, and other loss functions could be constructed based on these and other influence measures. We hope to explore these research directions in future works.

While our theoretical guarantees for the preservation of MDE or SHAP hold for wide classes of models, our experiments compare simple logistic models. It would be interesting to test the proposed methods on more complex non-linear models in various important real-world application scenarios across domains.
Given that the number of fairness objectives is already high and that we propose new fairness objectives, there is a need for evaluating learning algorithms addressing fairness. A potential approach could rely on realistic simulations of discrimination and test whether a given learning method is able to retrieve the non-discriminatory data-generating process.


Most importantly, any fairness objective can be misused by people to justify their systems as fair, especially if our limited understanding of causal processes happening in real-world decision-making adds up to the confusion. 
For instance, if a company develops a model of $Y$ using $\X$ and some $X_i$ is unfairly influenced, then first they shall apply our method to a model of $X_i$ and second to a model of $Y$. An omission of the first step, whether intentional or unintentional, would result in indirect discrimination.
In such contexts, we emphasize that understanding the causal processes relevant to the decision-making at hand in collaboration with domain experts and goodwill are of the highest priority, since it can lead to more accurate and more fair models.


\section{Conclusions}
The presented results shed a new light on the problem of discrimination prevention in supervised learning. 
First, we propose a formal definition of induced discrimination, inspired by discrimination via association~\cite{Wachter2019Affinity}. We measure influence of features to capture induced discrimination.
Second, we show that state-of-the-art methods addressing discrimination often return biased models when they are trained on datasets that are or are \textit{not} affected by discrimination. 
Third, for discrimination prevention we propose to use a marginal interventional mixture of full models, which prevents the induction of discrimination via association.
In the scenarios where discrimination does not affect the training data, the proposed learning algorithm falls back to a traditional learning, which ensures that the method does not bias the model needlessly. 
These results provide support for the use of the marginal interventional mixture in the circumstances where discrimination could have affected the training dataset. 


\begin{acks}
We thank Luis F. Lafuerza for his feedback and multiple rounds of comments and
Isabel Valera, Muhammad Bilal Zafar, and Krishna Gummadi for discussions on early versions of this work. P.A.G. acknowledges support from Volkswagen Foundation (Ref. 92136) and DARPA and ARO (Cooperative Agreement No. W911NF-20-2-0005). The views and conclusions contained in this document are those of the authors and should not be interpreted as representing the official policies, either expressed or implied, of the DARPA or ARO, or the U.S. Government.  
\end{acks}

\bibliographystyle{ACM-Reference-Format}
\bibliography{jabref,manual_additions}


\section*{Appendix A: Proof of Proposition 2}

Without loss of generality, we consider the case of two variables $x$ and $z$. 
From the definition of $L_\text{SHAP}(\X)$ and $\text{SHAP}_{\Y}(\x |\x\z)$, under $\ell_2$ loss:
\begin{align*}
&L_\text{SHAP}(\X) = 
\E_{\X} ( \E_{\Z''} \text{SHAP}_{\Y}(\X |\X\Z'') - \E_{\Z''} \text{SHAP}_{\hat{\Y}}(\X|\X\Z'') )^2 = \\
&=\E_{\X} \left( \E_{\X',\Z',\Z''} \left[
 \frac{Y_{\X,\Z''} - Y_{\X',\Z''} + Y_{\X,\Z'} - Y_{\X',\Z'}}{2}
 \right.\right. \\
   &\left. \left. -
 \frac{
 \hat{Y}_{\X,\Z''} - \hat{Y}_{\X',\Z''} + \hat{Y}_{\X,\Z'} - \hat{Y}_{\X',\Z'}
 }{2}
 \right] 
 \right)^2
\end{align*}

For an interventional mixture $\hat{y}_\pi(\x) = \E_{\tilde{\Z}} \hat{y}(\x,\tilde{\Z})$,
\begin{align*}
L_\text{SHAP}(\X) = 
 \E_{\X} \left( \E_{\X',\Z',\Z''} \left[
 \frac{Y_{\X,\Z''} - Y_{\X',\Z''} + Y_{\X,\Z'} - Y_{\X',\Z'}}{2}
 \right] \right. \\
   \left. -
 \E_{\X',\tilde{Z}} \left[ 
 \hat{Y}_{\X,\tilde{Z}} - \hat{Y}_{\X',\tilde{Z}}
 \right] \right)^2.
\end{align*}

Assuming that the function $y(x,z)$ is analytic, we can expand $y(x,z)$ into a Taylor series around the point $(x=0,z=0)$, which is a series of components $C x^k z^l$, where $C$ is a real-valued constant and $k$ and $l$ are integers from $0$ to $\infty$. Let us consider a related series, $y_n(x,z)=\sum_{m=1}^n \alpha_{m} x^{i(m)} z^{j(m)}$, that can represent a subset of $n$ components of a Taylor series, where $\alpha_{i}$ are some real-valued constants and $i(m)$ and $j(m)$ are functions returning unique pairs of non-negative integers ordered by $m$ such that $\forall_{m_1 \neq m_2} \left(i(m_1),j(m_1)\right) \neq \left(i(m_2),j(m_2)\right)$. 

Our proof strategy is to first show that the MIM is an interventional mixture that minimizes $L_\text{SHAP}(\X)$ for the case of $n=1$, i.e., $y_1(x,z)$. Then, we prove by induction that the MIM is an interventional mixture that minimizes $L_\text{SHAP}(\X)$ for any $n$. Since $y_\infty(x,z)$ includes the full Taylor series of $y(x,z)$, so this step ends the proof.


First, we show that for $y_1(x,z)=\alpha_1 x^k z^l$, where $k=i(1)$ and $l=j(1)$, the optimal mixing distribution $\pi^*(\tilde{Z})$ is the marginal distribution, i.e., $\pi^*(\tilde{Z}=z) = P(Z=z)$.
Note that the expectation in the interventional mixture can be written as $\E_{\tilde{Z}} x^k\tilde{Z}^l = \beta x^k$, where $\beta= \alpha_1 \E_{\tilde{Z}}\tilde{Z}^l$.
Then,
\begin{align}
L_\text{SHAP}(\X) = 
 \E_{\X} \left( \E_{\X',\Z',\Z''} \left[
 \frac{X^k Z''^l - X'^k Z''^l + X^k Z'^l - X'^k Z'^l}{2\alpha_1^{-1}}
 \right.\right. \nonumber\\
   \left. \left.-
 \beta X^k + \beta X'^k
 \right] \right)^2
 \label{eq:y_xk_zl}
\end{align}
and the minimization of this objective reduces to finding the optimal scalar $\beta$.
The necessary condition for the minimum of $L_\text{SHAP}(\X)$ is that its first derivative is zero. Since $X'$, $Z'$, and $Z''$ have the same means as $X$ and $Z$, respectively, so the first derivative can be simplified as follows,
\begin{align*}
 &\frac{\partial L_\text{SHAP}(\X)}{\partial \beta} =\\
 &= 
 -\E_{\X, \X',\Z',\Z''} \left[ \left(
 \alpha_1(X^k Z''^l - X'^k Z''^l + X^k Z'^l - X'^k Z'^l)
  \right.\right. \nonumber\\
   &\left. \left. - 2 \beta \left( X^k - X'^k
 \right)\right) (X^k - X'^k) \right]=\\
 &= -\E_{\X} \left[ \left(
 \alpha_1(X^k \overbar{Z^l} - \overbar{X^k} \overbar{Z^l} + X^k \overbar{Z^l} - \overbar{X^k Z^l})
 \right.\right. \nonumber\\
   &\left. \left. - 2 \beta \left( X^k -\overbar{X^k} 
 \right)\right) \left(X^k - \overbar{X^k} \right) \right],
\end{align*}
where $\overbar{X}$ is the mean of $X$. After performing a few basic algebraic operations (note that $\E_{\X}[\overbar{X^k} \overbar{Z^l}(X^k - \overbar{X^k} )]=\E_{\X}[\overbar{X^kZ^l}(X^k - \overbar{X^k} )]=0$), the necessary condition reads
\begin{align*}
\frac{\partial L_\text{SHAP}(\X)}{\partial \beta} &= - \E_{\X} \left[ \left(
 2 \alpha_1 X^k \overbar{Z^l} - 2 \beta \left( X^k - \overbar{X^k}
 \right)\right) \left(X^k - \overbar{X^k} \right) \right]=\\ 
 &= - 2 \left( \alpha_1 \V[X^k] \overbar{Z^l} - \beta \V[X^k] \right) = 0,
\end{align*}
where $\V[X]$ is the variance of $X$. This condition is fulfilled if $\beta =\alpha_1 \overbar{Z^l}$, which requires that $\E \tilde{Z}^l = \overbar{Z^l}$. The condition is met if the distributions of $Z$ and $\tilde{Z}$ are the same.
The extremum is actually a minimum, since the second derivative, $\V[ X^k$], is positive for any random variable $X$ with non-zero variance and positive $k$.
For the special case of $k=0$ or $l=0$ the MIM minimizes the objective globally, by achieving $L_\text{SHAP}(\X) = 0$.

Next, let us assume that for a certain $y_{n-1}(x,z) = \sum_{m=1}^{n-1} \allowbreak \alpha_{m} \allowbreak x^{i(m)}\allowbreak z^{j(m)}$, where $n>1$, the MIM minimizes $L_\text{SHAP}(X)$. Given this assumption, we aim to prove that the MIM is the optimal mixture also for $y_{n}(x,z) = y_{n-1}(x,z) + \alpha_n x^{i(n)} z^{j(n)}$. 
In this case, the objective can be written as
\begin{equation*}
    L_\text{SHAP}(X) = \E_X\left[( A_{n-1}(X) + a_n(X) )^2\right],
\end{equation*} where 
\begin{align*}
a_n(x) &= \alpha_n \left( x^{i(n)} \overbar{Z^{j(n)}} - \overbar{X^{i(n)}} \overbar{Z^{j(n)}} + x^{i(n)} \overbar{Z^{j(n)}} \right. \nonumber\\
   &\left. - \overbar{X^{i(n)} Z^{j(n)}} \right)/2
 - \beta_n \left( x^{i(n)} - \overbar{X^{j(n)}}
 \right)
\end{align*}
and $A_{n-1}(x) = \sum_{m=1}^{n-1} a_m(x)$. To simplify the notation, we write $\bm{\beta} = (\beta_1,...,\beta_n)$. Next, we expand the objective,
\begin{align*}
    L_\text{SHAP}(X) = \E_X\left[( A^2_{n-1}(X) + a^2_n(X) + 2 A_{n-1}(X) a_n(X) )\right].
\end{align*}
From the assumption we know that the MIM minimizes $\E_X\left[ A^2_{n-1}(X)\right]$. Also, $\E_X\left[ a^2_{n}(X)\right]$ is equivalent to the objective defined in Equation~\ref{eq:y_xk_zl}, which is minimized by the MIM as well, as we showed above. Hence, we shall focus now on the remaining term, i.e., $l(\bm{\beta}) = \E_X\left[( A_{n-1}(X) a_n(X) )\right]$. The necessary conditions for an extremum are
\begin{align*}
    \frac{\partial l}{\partial \beta_n} =  -2\E_X\left[ A_{n-1}(X) \left( X^{i(n)} - \overbar{X^{j(n)}}
 \right)\right] = 0,\\
 \forall_{m<n} \frac{\partial l}{\partial \beta_m} = -2  \E_X\left[ a_n(X) \left( X^{i(m)} - \overbar{X^{j(m)}}
 \right)\right] = 0.
\end{align*}
These conditions are satisfied when $\forall_{m=1}^n \beta_m =\alpha_m \overbar{Z^{i(m)}}$, which are met if $\forall_{m=1}^n \E \tilde{Z}^{i(m)} = \overbar{Z^{i(m)}}$. Hence, the necessary conditions for extremum are met if the distribution of $\tilde{Z}$ is the same as the marginal distribution of $Z$. The corresponding Hessian matrix is positive semi-definite, so the extremum is a minimum.

We have shown that the MIM is an interventional mixture that minimizes $L_\text{SHAP}(\X)$ for $y_1(x,z)$. In addition, it minimizes it also for $y_n(x,z)$, assuming that it does so for $y_{n-1}(x,z)$. Thus, by induction, the MIM is an interventional mixture that minimizes $L_\text{SHAP}(\X)$ for any $n$ and any function $y(x,z)$ that has a Taylor expansion. 



\section*{Appendix B: Additional results}
\label{sec:appendixB}


\textbf{Synthetic Scenarios.} Here, we present the SHAP influence of $X_1$, $X_2$ and $Z$, equal opportunity difference ($|\p(\hat{y}=1|y=1,z=0)-\p(\hat{y}=1|y=1,z=1)|$), disparate impact ($|\p(\hat{y}=1|z=0)/\p(\hat{y}=1|z=1)|$), and a relaxed version of equalized odds ($(|FPR_{z=0}-FPR_{z=1}|+|TPR_{z=0}-TPR_{z=1}|)/2$) of the scenarios in the main text: A) $Y=\sigma(X_1 + X_2 + Z + 1)$ and B) $Y=\sigma(0*X_1 + X_2+ 0*Z + 1)$. We also provide results for the OPT methods, the reductions method from \citet{Agarwal2018reductions} subject to equal opportunity and error ratios fairness constraints(``TPR'', and ``ER''). As expected, the evaluated models that use $Z$ by design are influenced by it ("Cal Eq Odds", Reweighing in Figure \ref{fig:aif111sm} \& \ref{fig:aif010sm}).

\textbf{Real-world datasets.} In Figures \ref{fig:aifcompassm}, \ref{fig:aifgermansm}, \& \ref{fig:aifcensussm} we provide results for the disparate impact, equalized odds, and equal opportunity fairness metrics and for the OPT methods on the COMPAS, German Credit, and Adult Census Income datasets. 

\begin{figure*}[ht]
\centering
\begin{subfigure}{.24\linewidth}
  \centering
  \includegraphics[width=\linewidth]{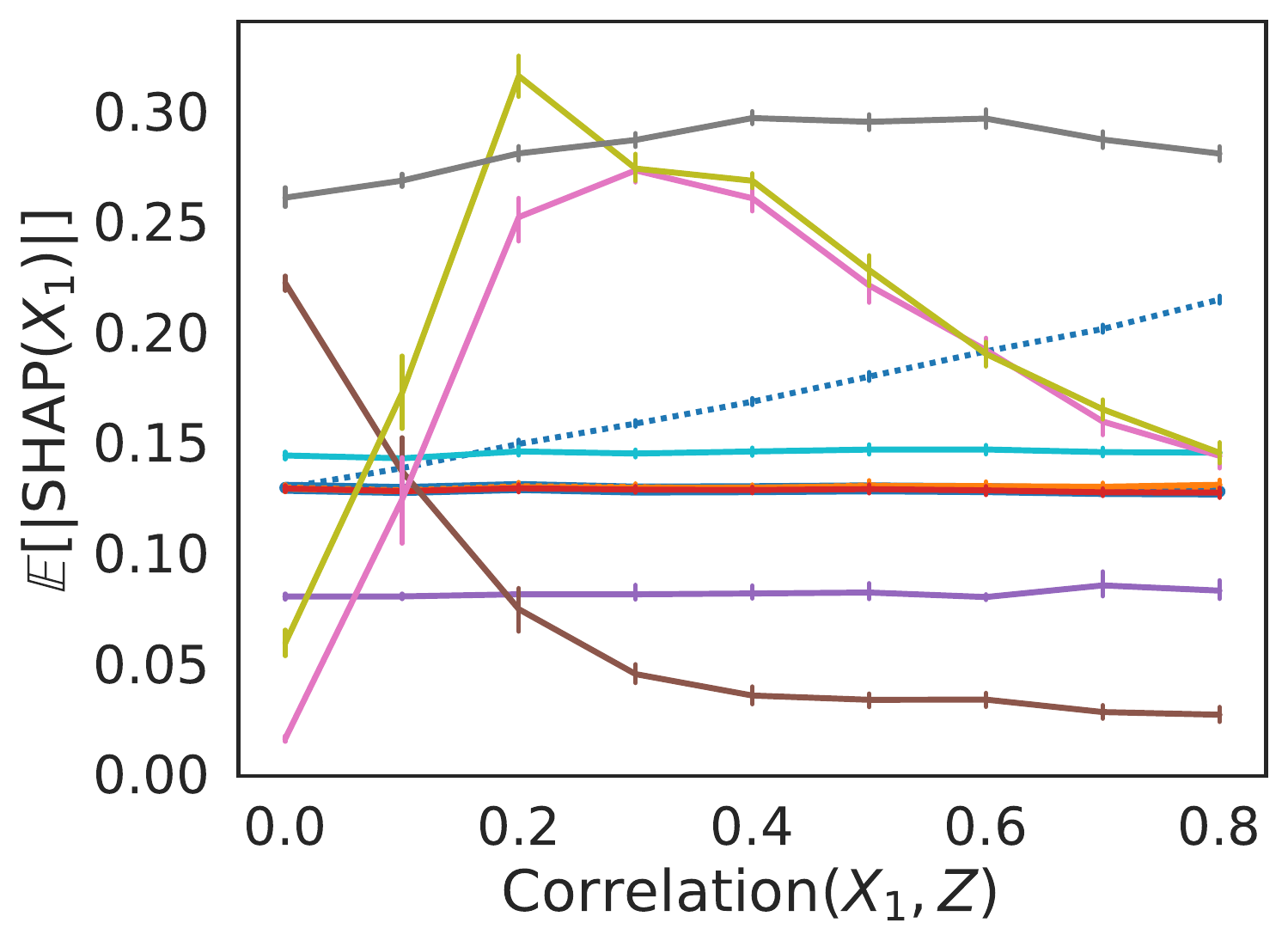}
\end{subfigure}
\begin{subfigure}{.24\linewidth}
  \centering
  \includegraphics[width=\linewidth]{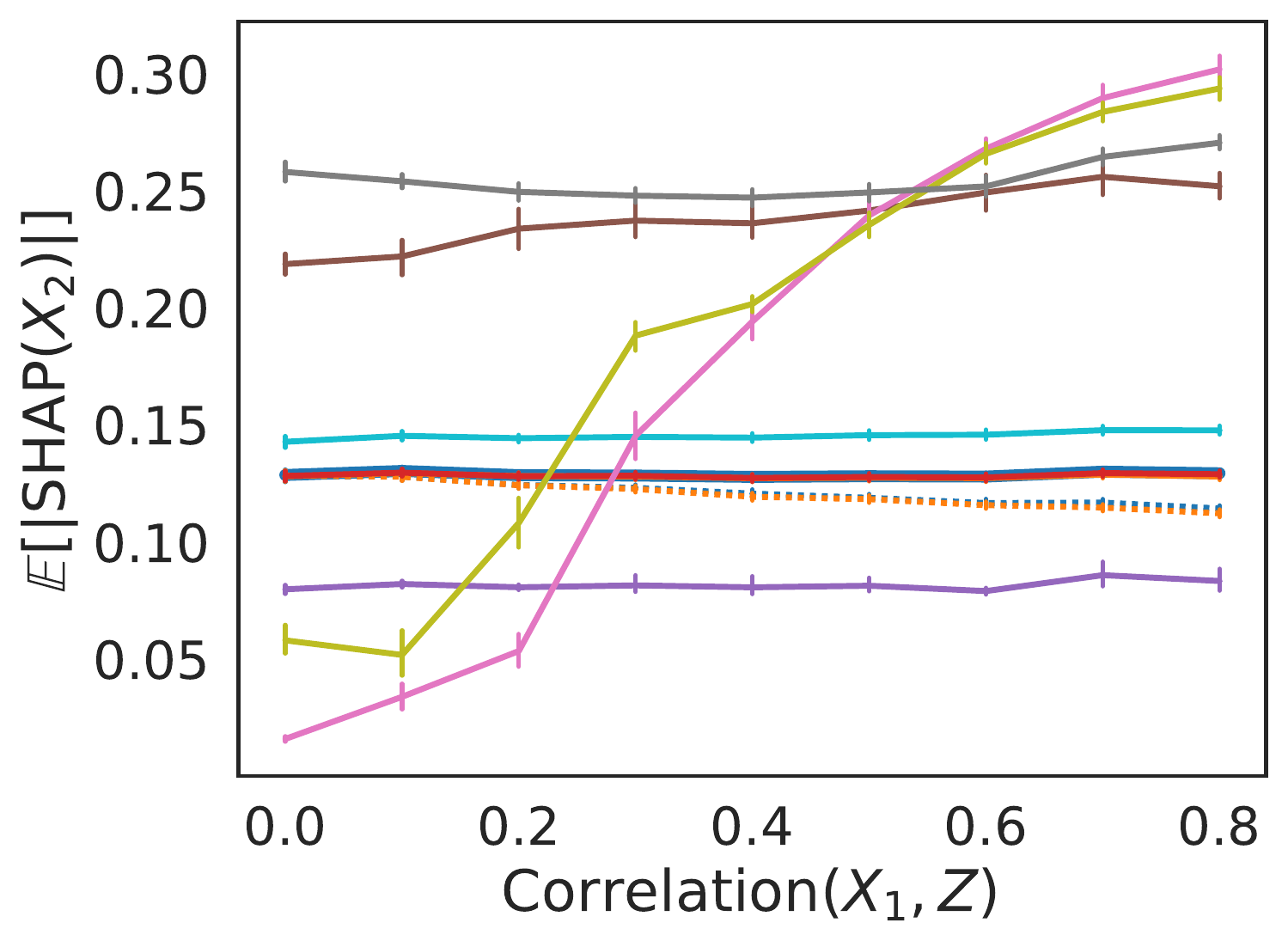}
\end{subfigure}
\begin{subfigure}{.24\linewidth}
  \centering
  \includegraphics[width=\linewidth]{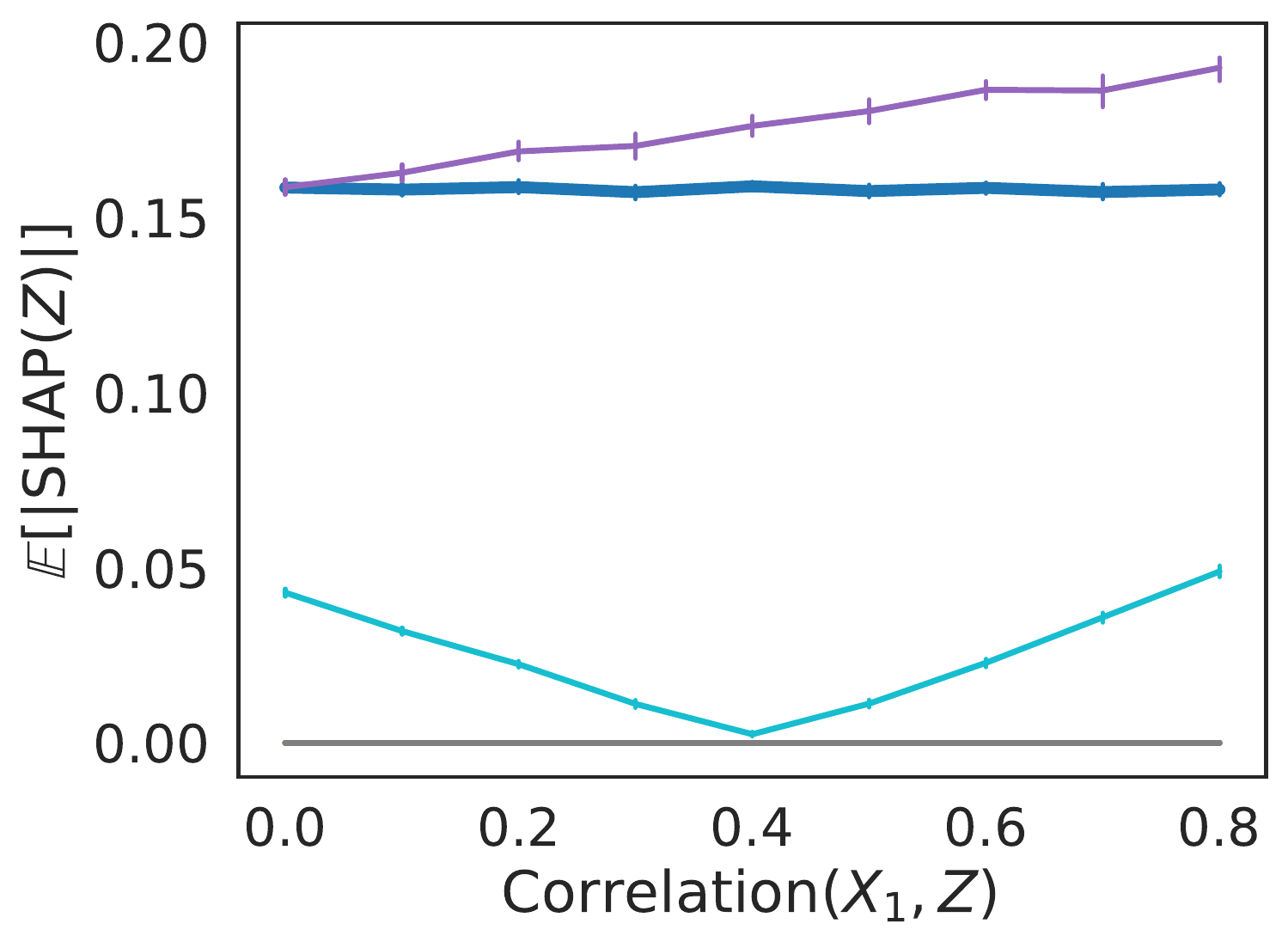}
\end{subfigure}
\begin{subfigure}{.24\linewidth}
  \centering
  \includegraphics[width=\linewidth]{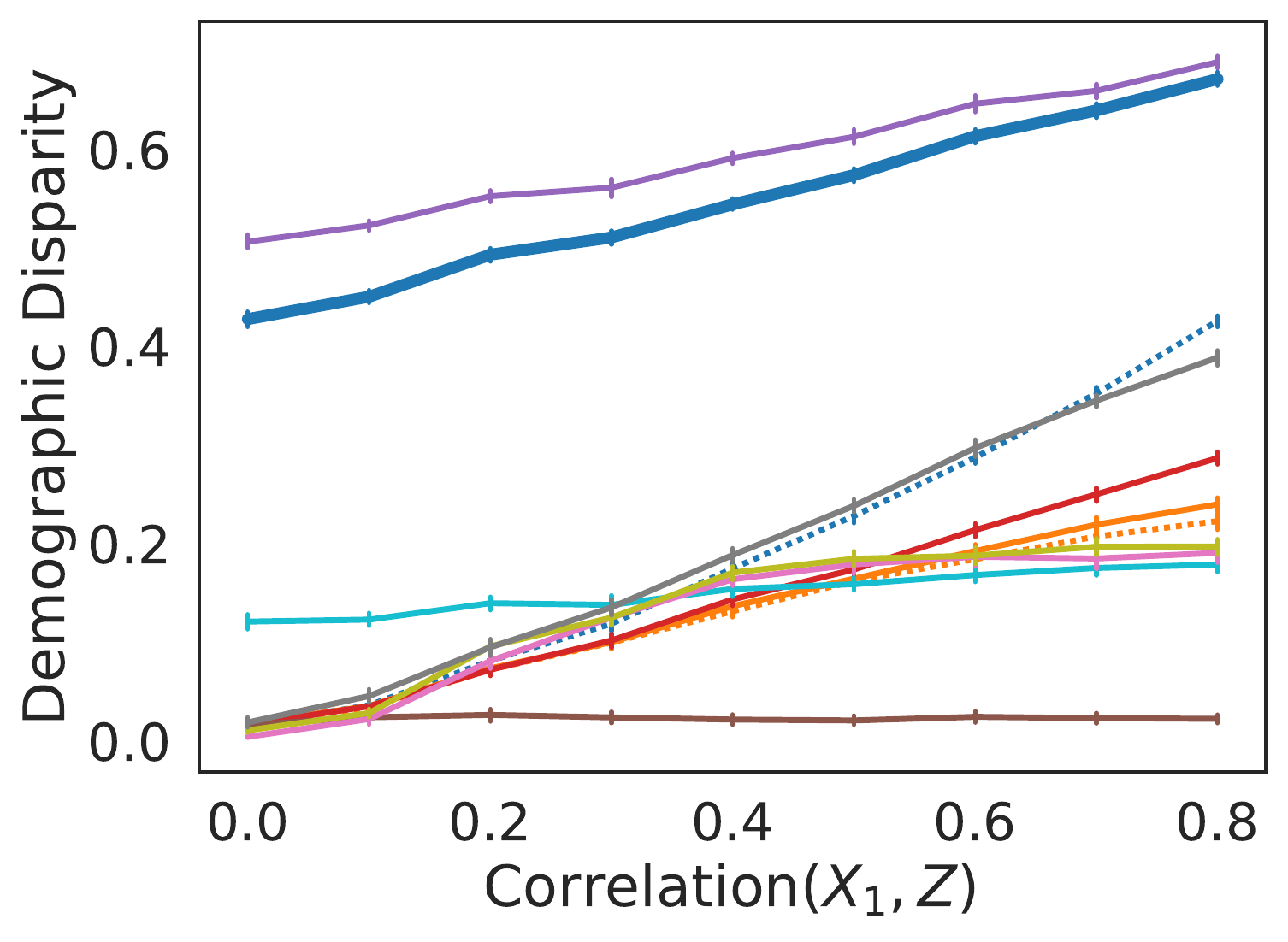}
\end{subfigure}
\begin{subfigure}{.24\linewidth}
  \centering
  \includegraphics[width=\linewidth]{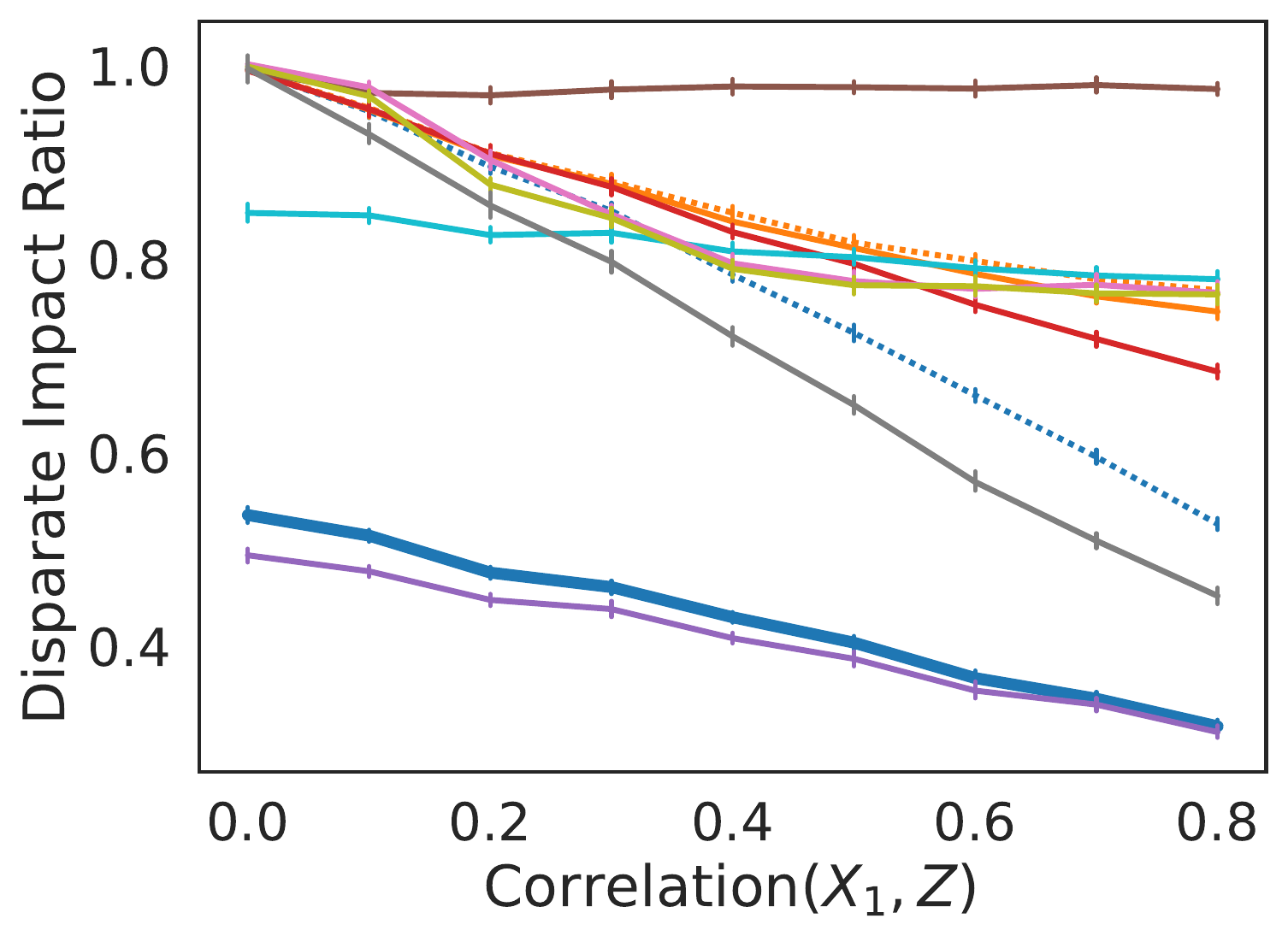}
\end{subfigure}
\begin{subfigure}{.24\linewidth}
  \centering
  \includegraphics[width=\linewidth]{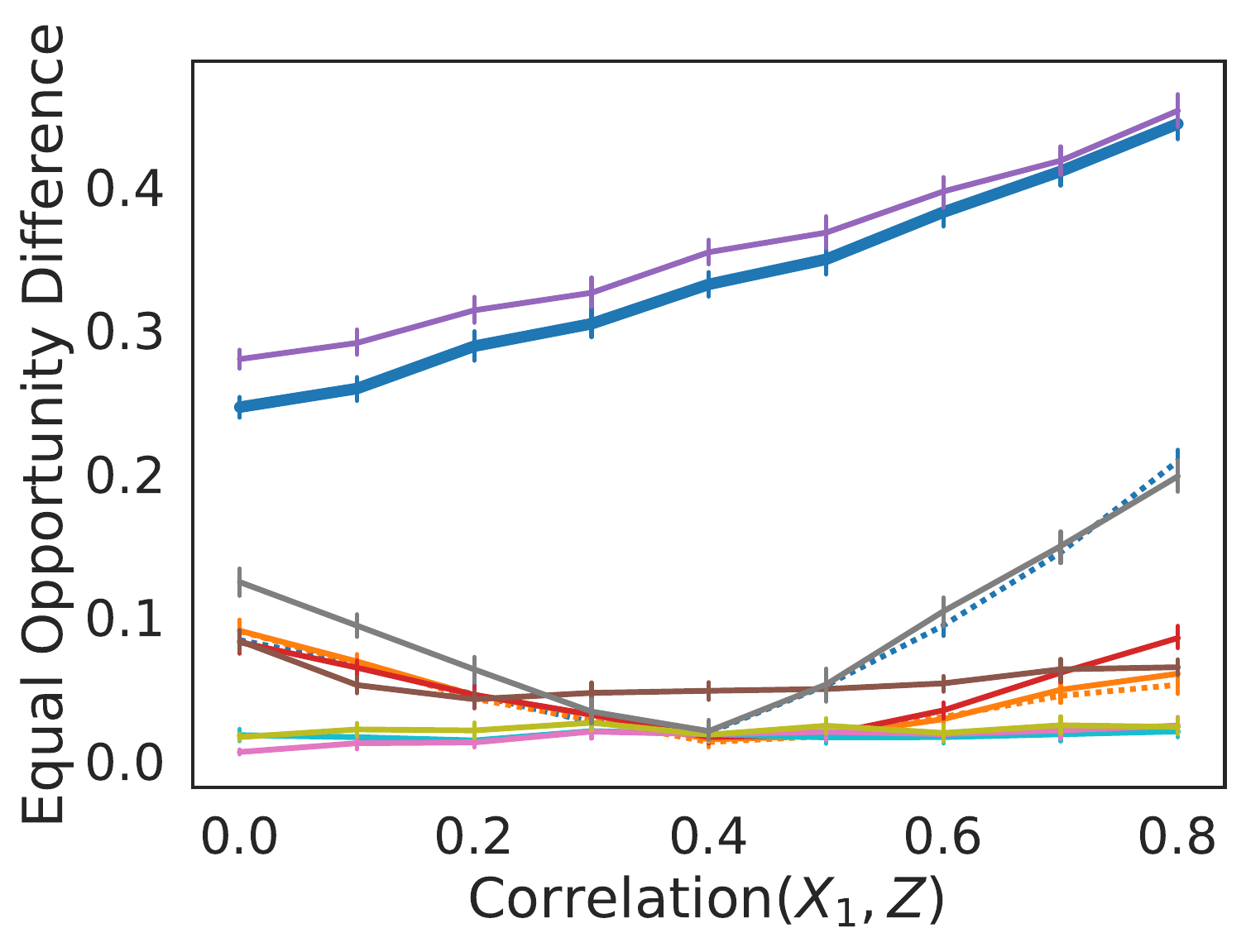}
\end{subfigure}
\begin{subfigure}{.24\linewidth}
  \centering
  \includegraphics[width=\linewidth]{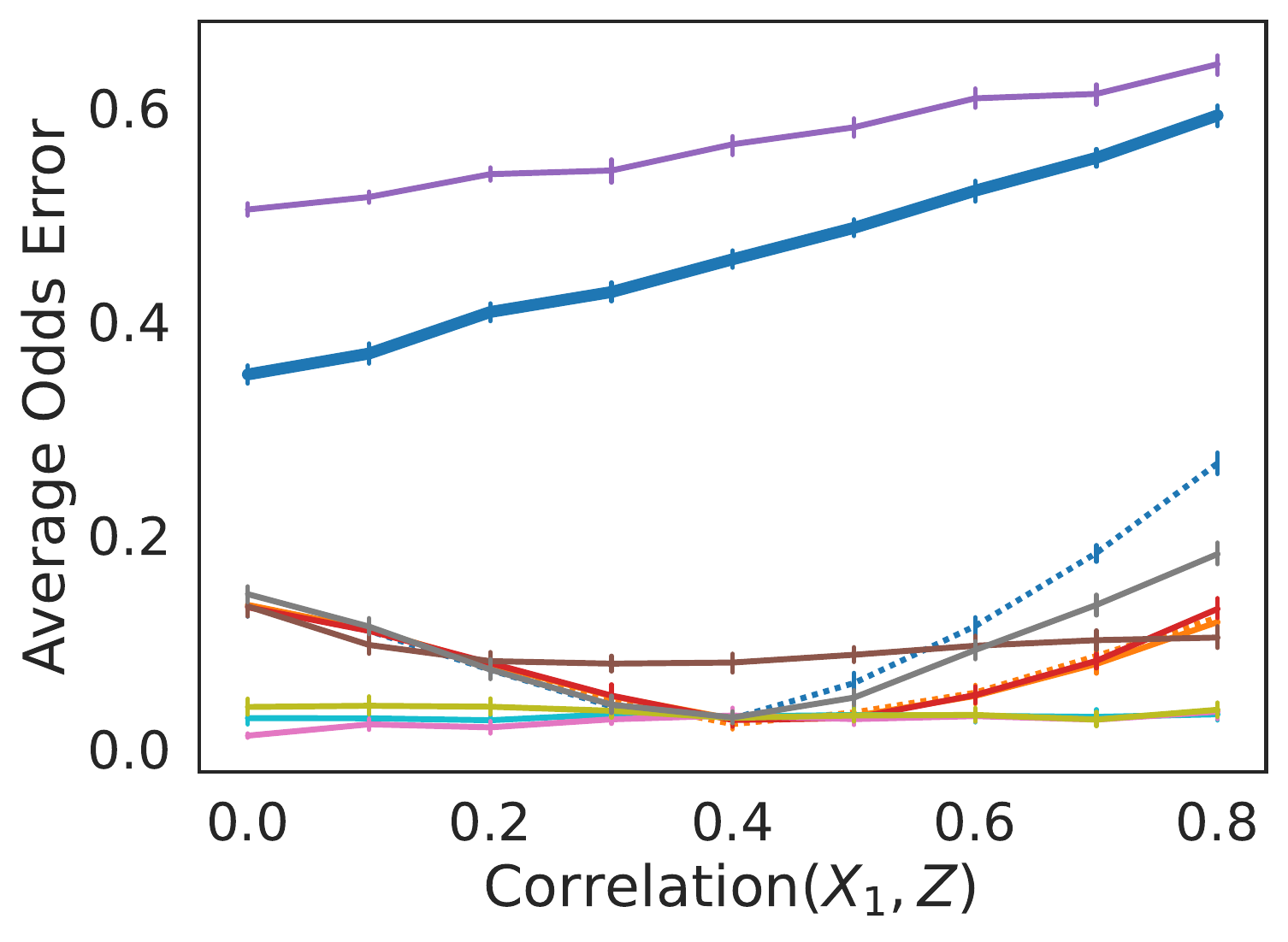}
\end{subfigure}
\begin{subfigure}{\linewidth}
  \centering
  \includegraphics[width=\linewidth]{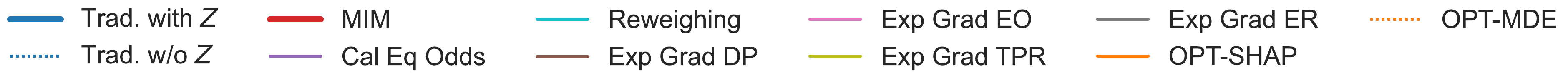}
\end{subfigure}

\caption{Scenario A: Averaged absolute SHAP of $X_1$, $X_2$, and $Z$ and four fairness measures as we increase the correlation $r(X_1,Z)$. Error bars show 95\% confidence intervals.}
\label{fig:aif111sm}
\end{figure*}

\begin{figure*}[h]
\centering
\begin{subfigure}{.24\linewidth}
  \centering
  \includegraphics[width=\linewidth]{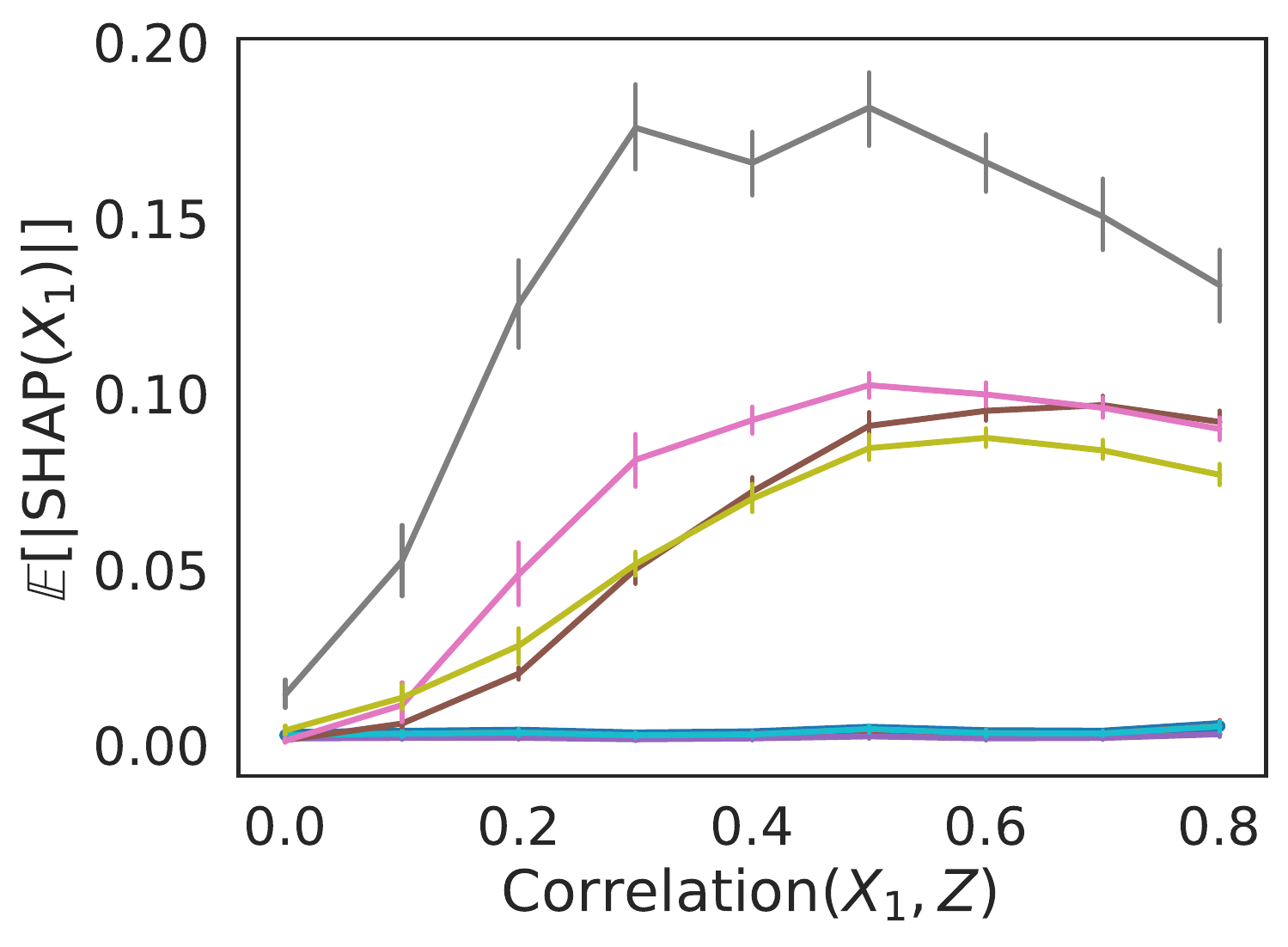}
\end{subfigure}
\begin{subfigure}{.24\linewidth}
  \centering
  \includegraphics[width=\linewidth]{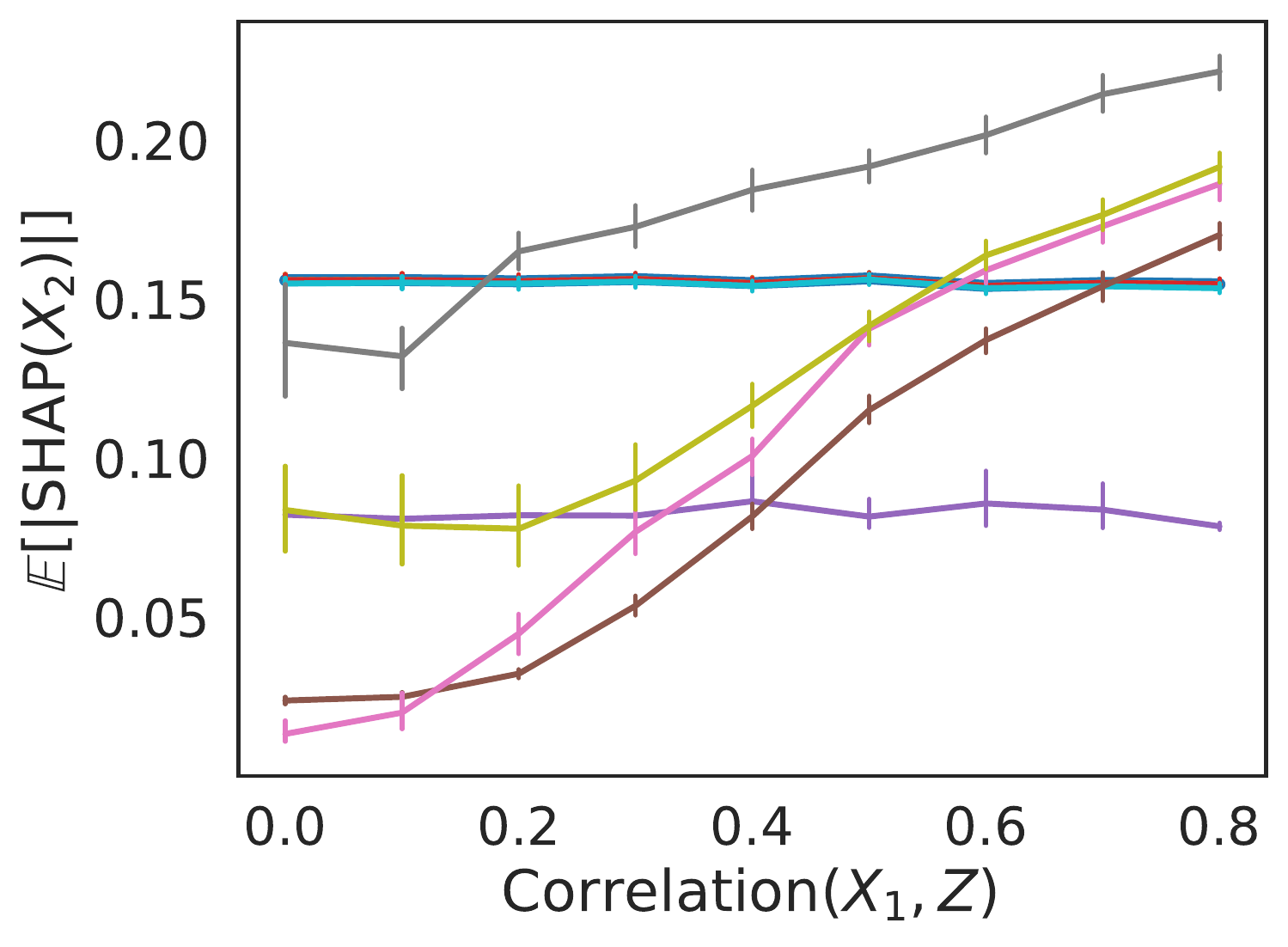}
\end{subfigure}
\begin{subfigure}{.24\linewidth}
  \centering
  \includegraphics[width=\linewidth]{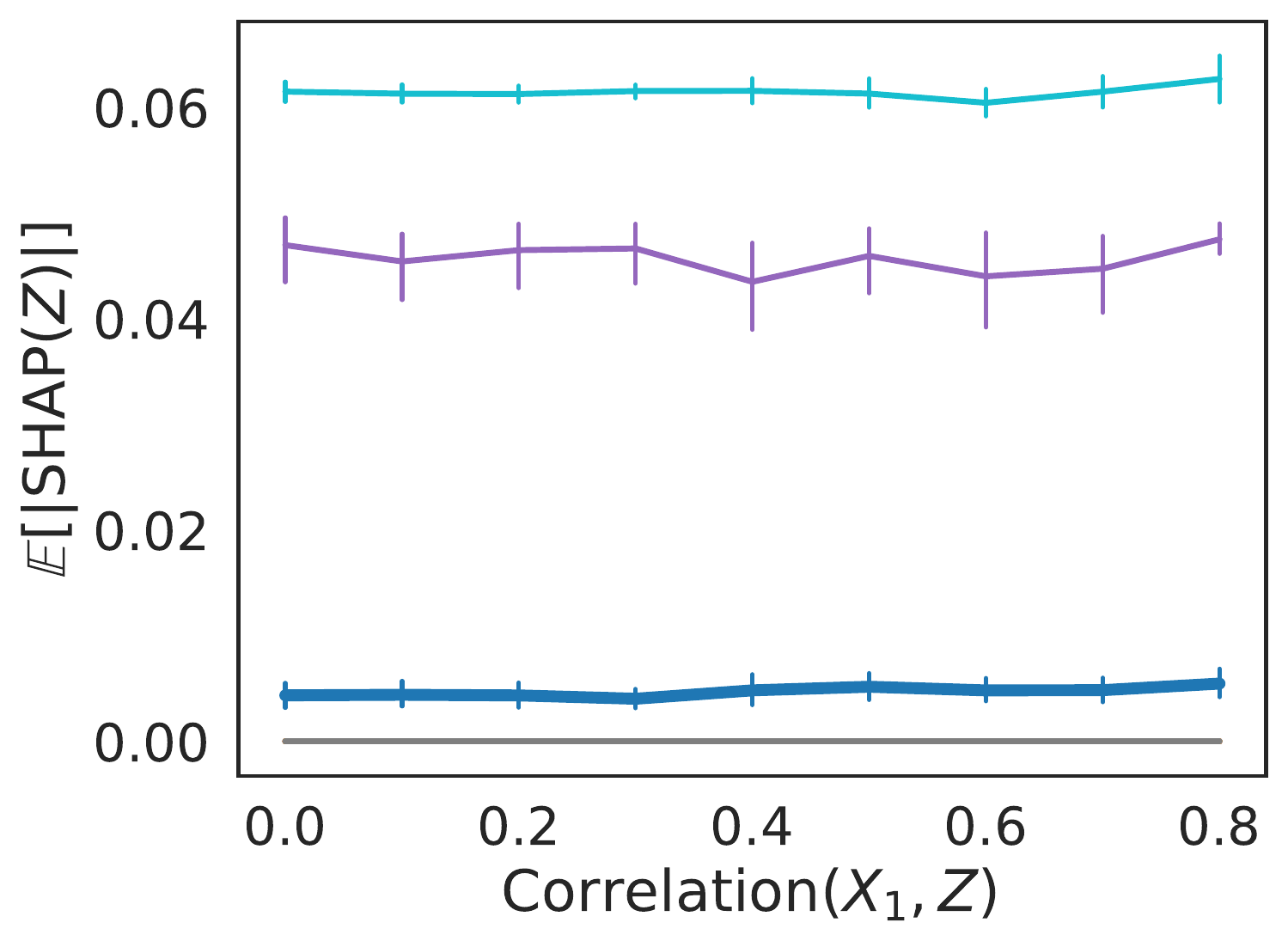}
\end{subfigure}
\begin{subfigure}{.24\linewidth}
  \centering
  \includegraphics[width=\linewidth]{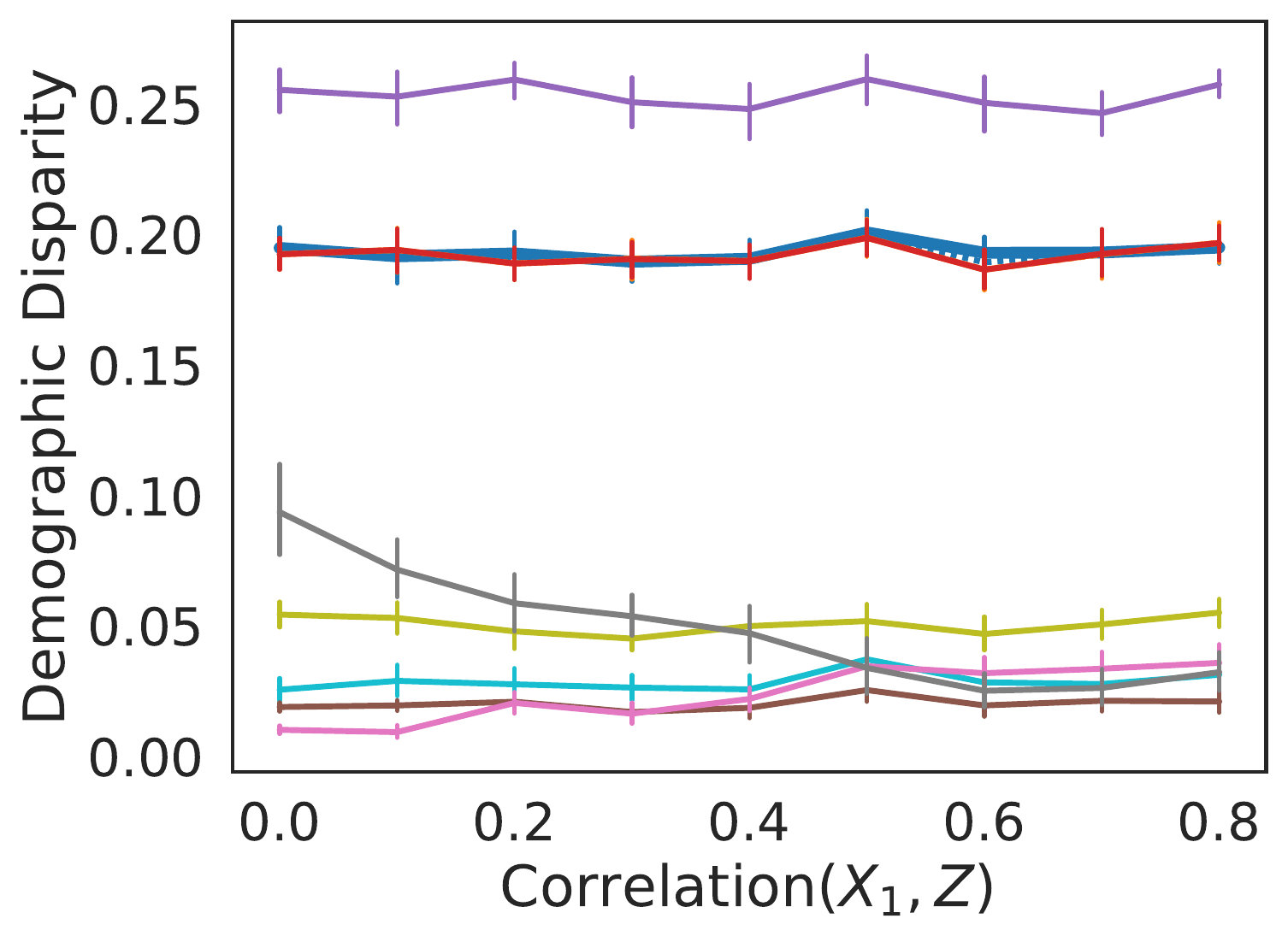}
\end{subfigure}
\begin{subfigure}{.24\linewidth}
  \centering
  \includegraphics[width=\linewidth]{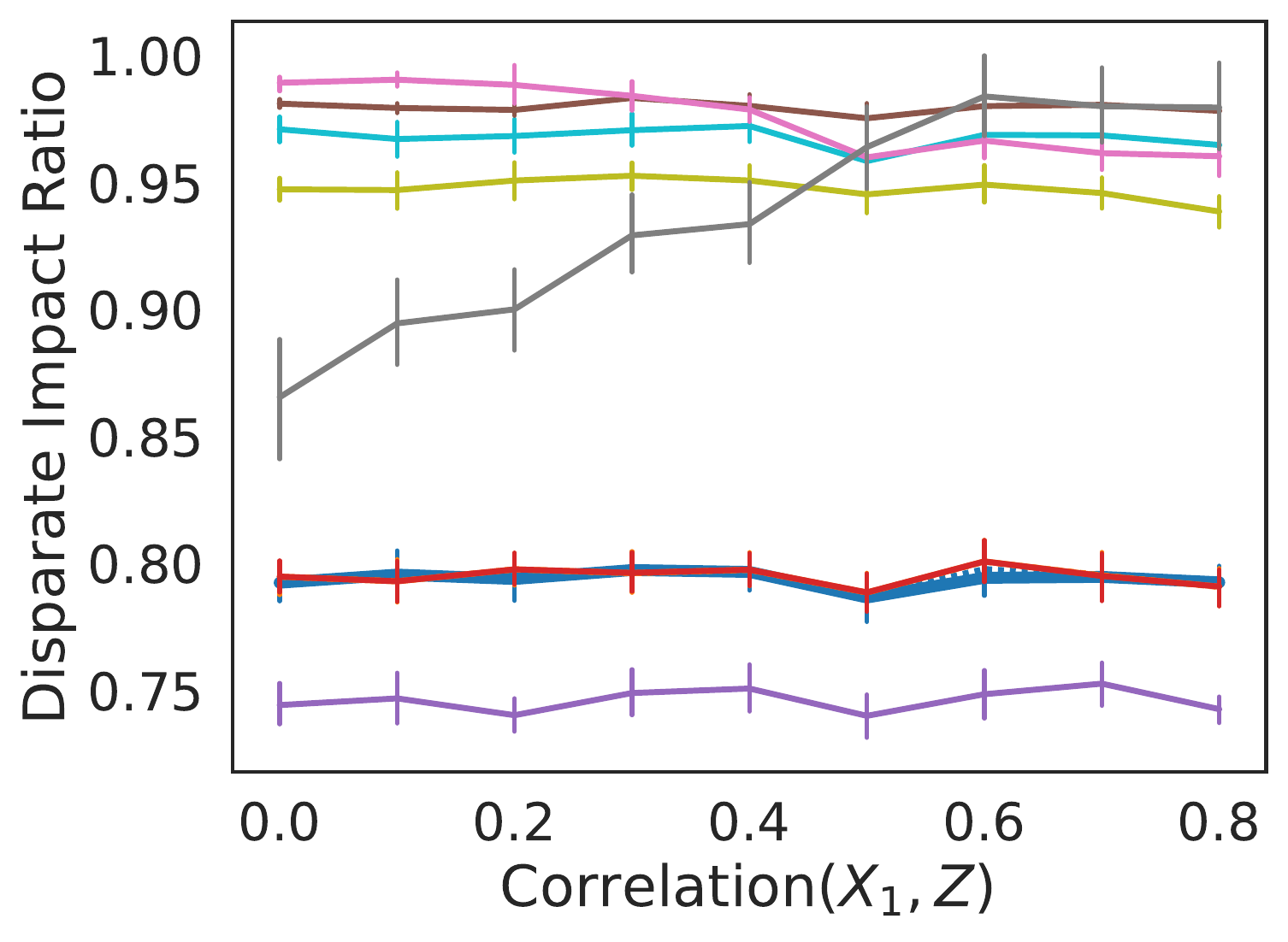}
\end{subfigure}
\begin{subfigure}{.24\linewidth}
  \centering
  \includegraphics[width=\linewidth]{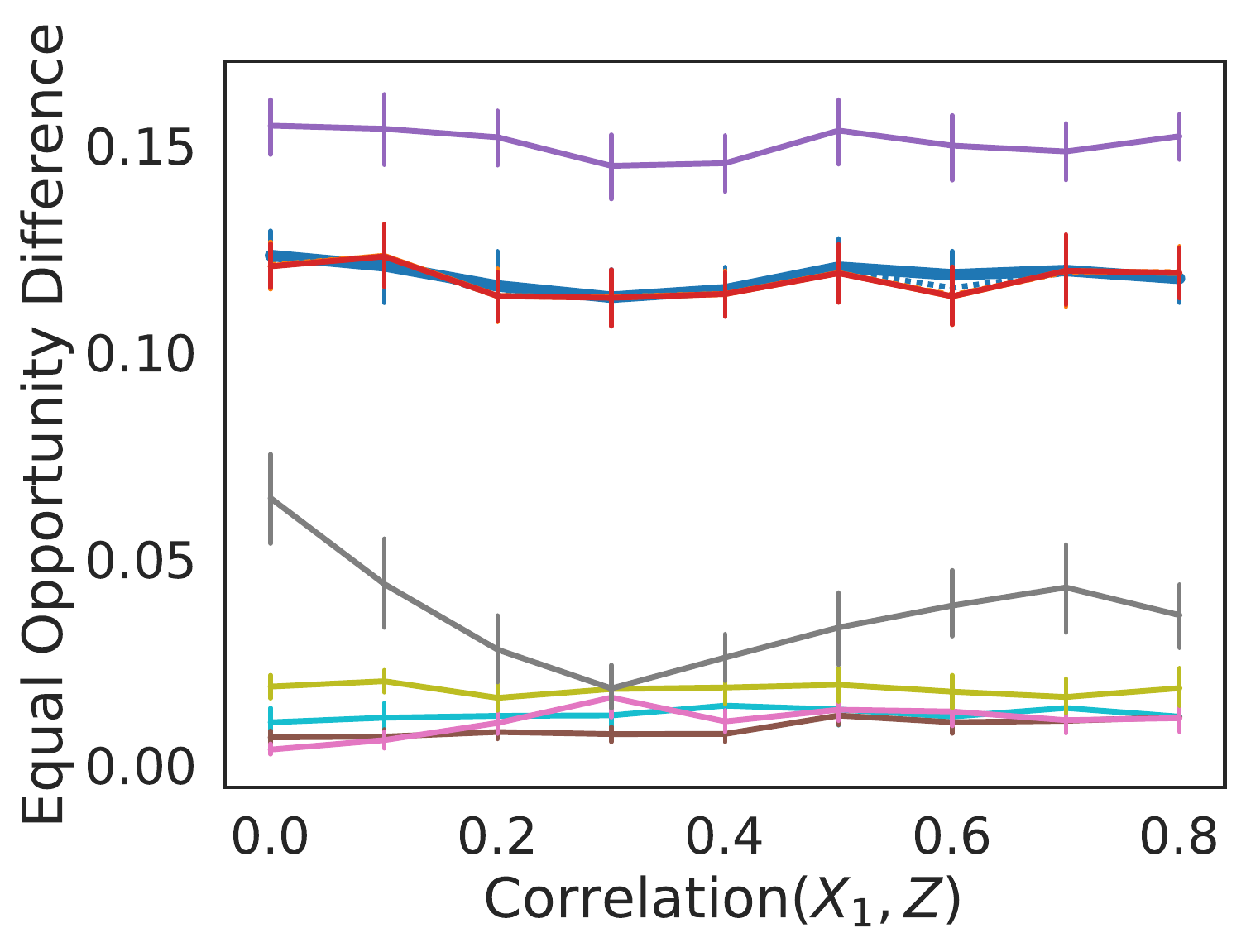}
\end{subfigure}
\begin{subfigure}{.24\linewidth}
  \centering
  \includegraphics[width=\linewidth]{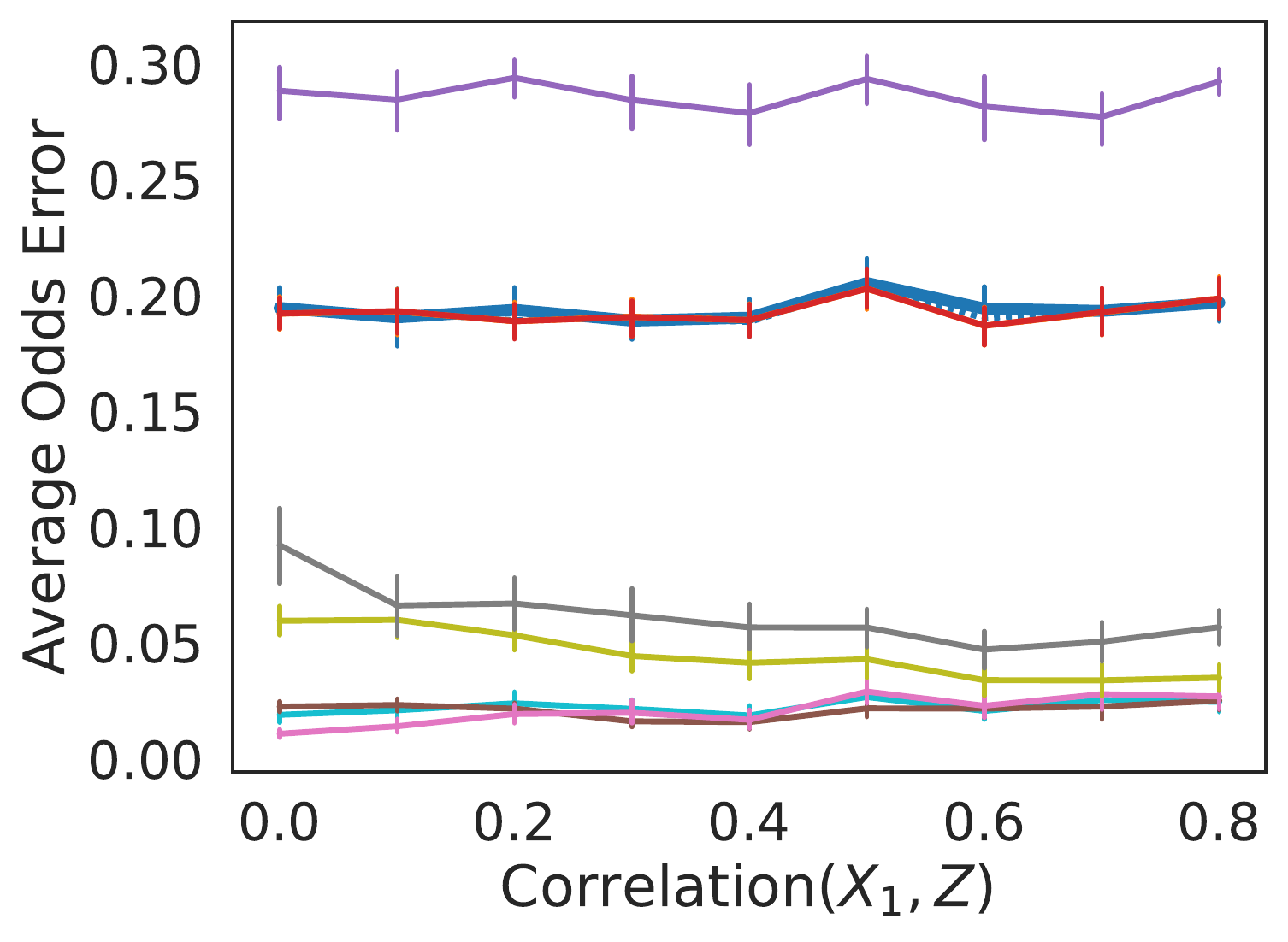}
\end{subfigure}
\begin{subfigure}{\linewidth}
  \centering
  \includegraphics[width=\linewidth]{figs/facct_results/scenario2_final/set3/legend3.pdf}
\end{subfigure}

\caption{Scenario B: Averaged absolute SHAP of $X_1$, $X_2$, and $Z$ and four fairness measures as we increase the correlation $r(X_1,Z)$. Error bars show 95\% confidence intervals.}
\label{fig:aif010sm}
\end{figure*}

\begin{figure*}[t]
\centering
\begin{subfigure}{.3\linewidth}
  \centering
  \includegraphics[width=\linewidth]{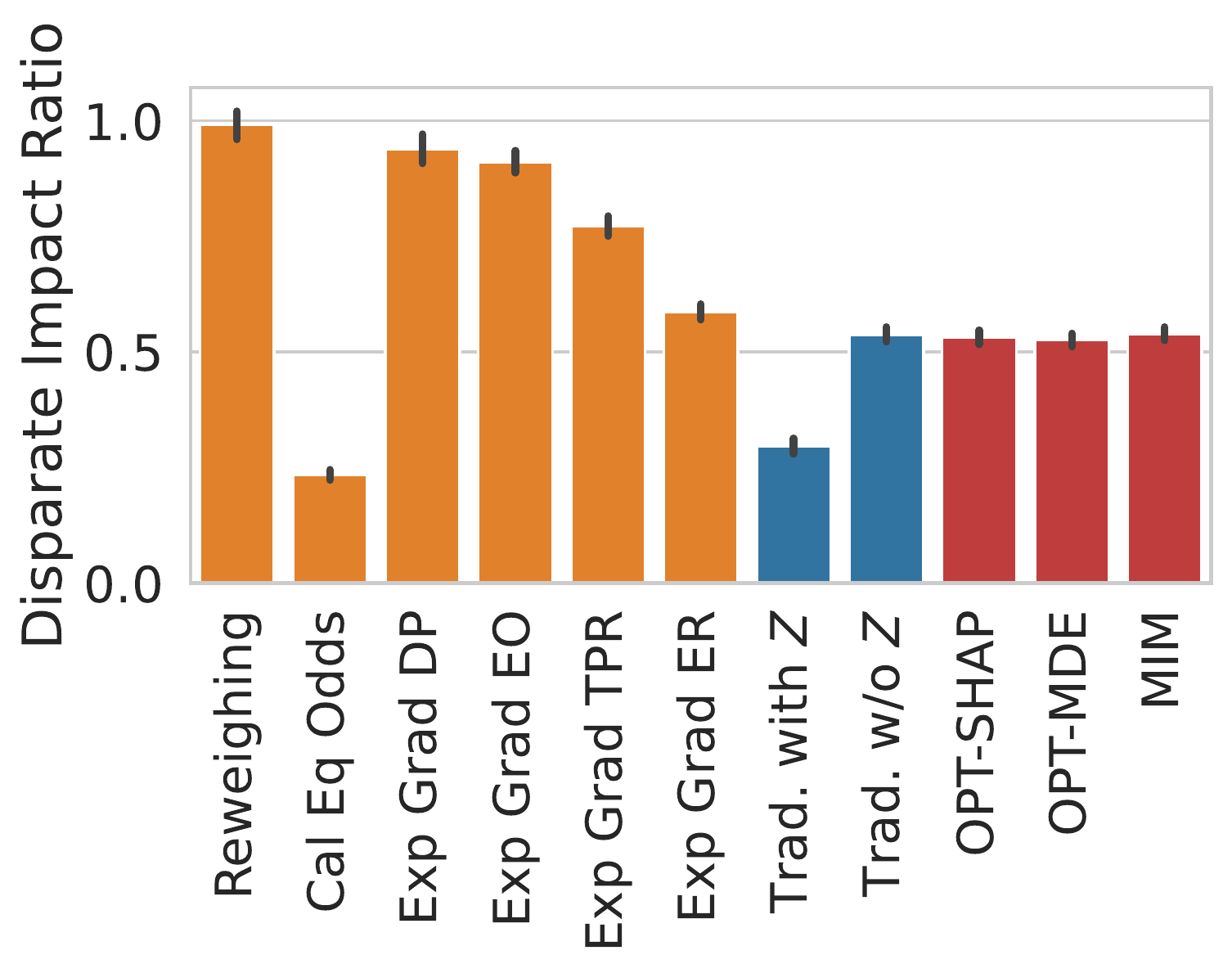}
\end{subfigure}
\begin{subfigure}{.3\linewidth}
  \centering
  \includegraphics[width=\linewidth]{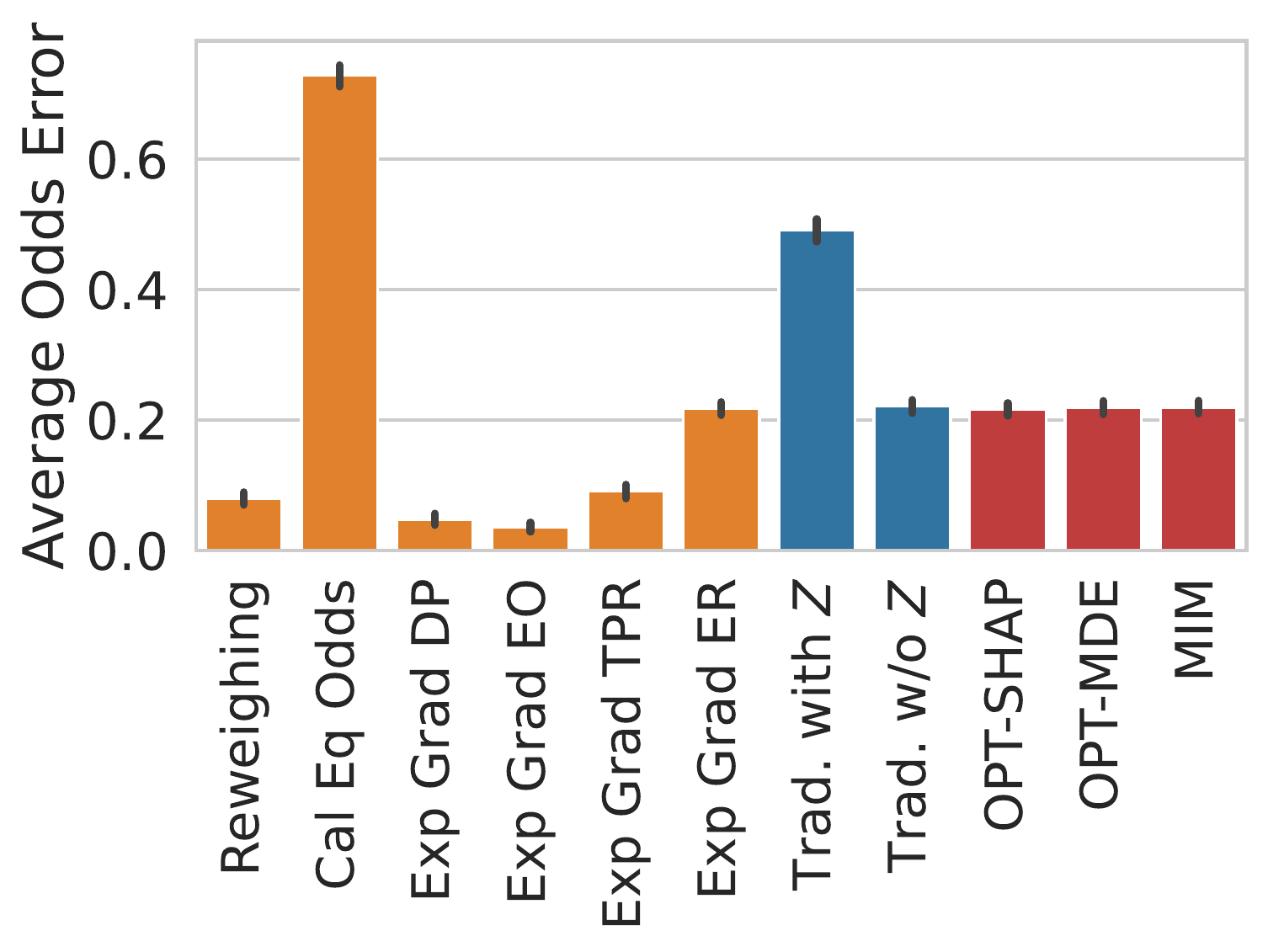}
\end{subfigure}
\begin{subfigure}{.3\linewidth}
  \centering
  \includegraphics[width=\linewidth]{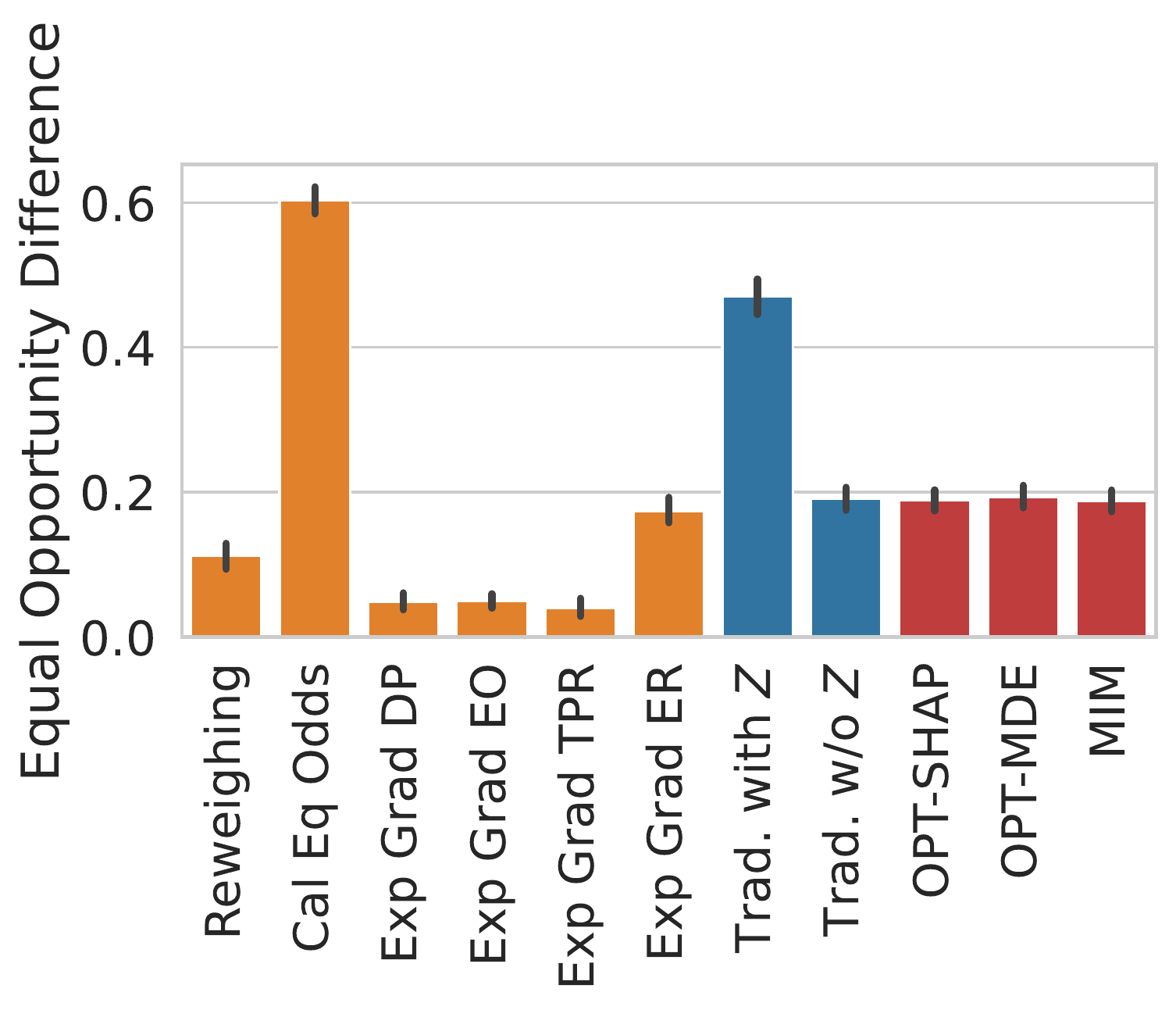}
\end{subfigure}
\caption{Disparate impact, equalized odds, and equal opportunity measures for the evaluated models on the COMPAS dataset. Error bars show 95\% confidence intervals.}
\label{fig:aifcompassm}
\end{figure*}

\begin{figure*}[ht]
\centering
\begin{subfigure}{.3\linewidth}
  \centering
  \includegraphics[width=\linewidth]{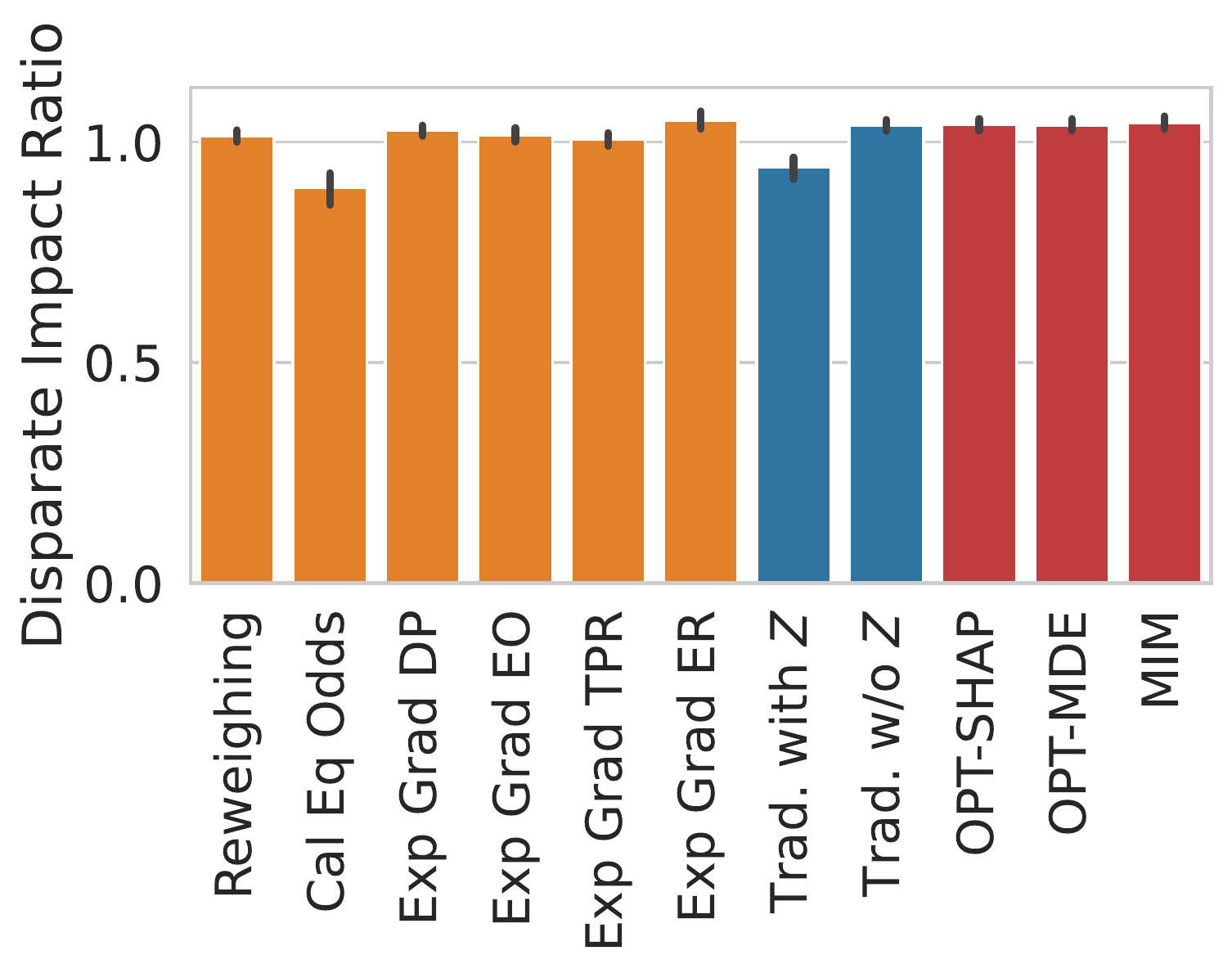}
\end{subfigure}
\begin{subfigure}{.3\linewidth}
  \centering
  \includegraphics[width=\linewidth]{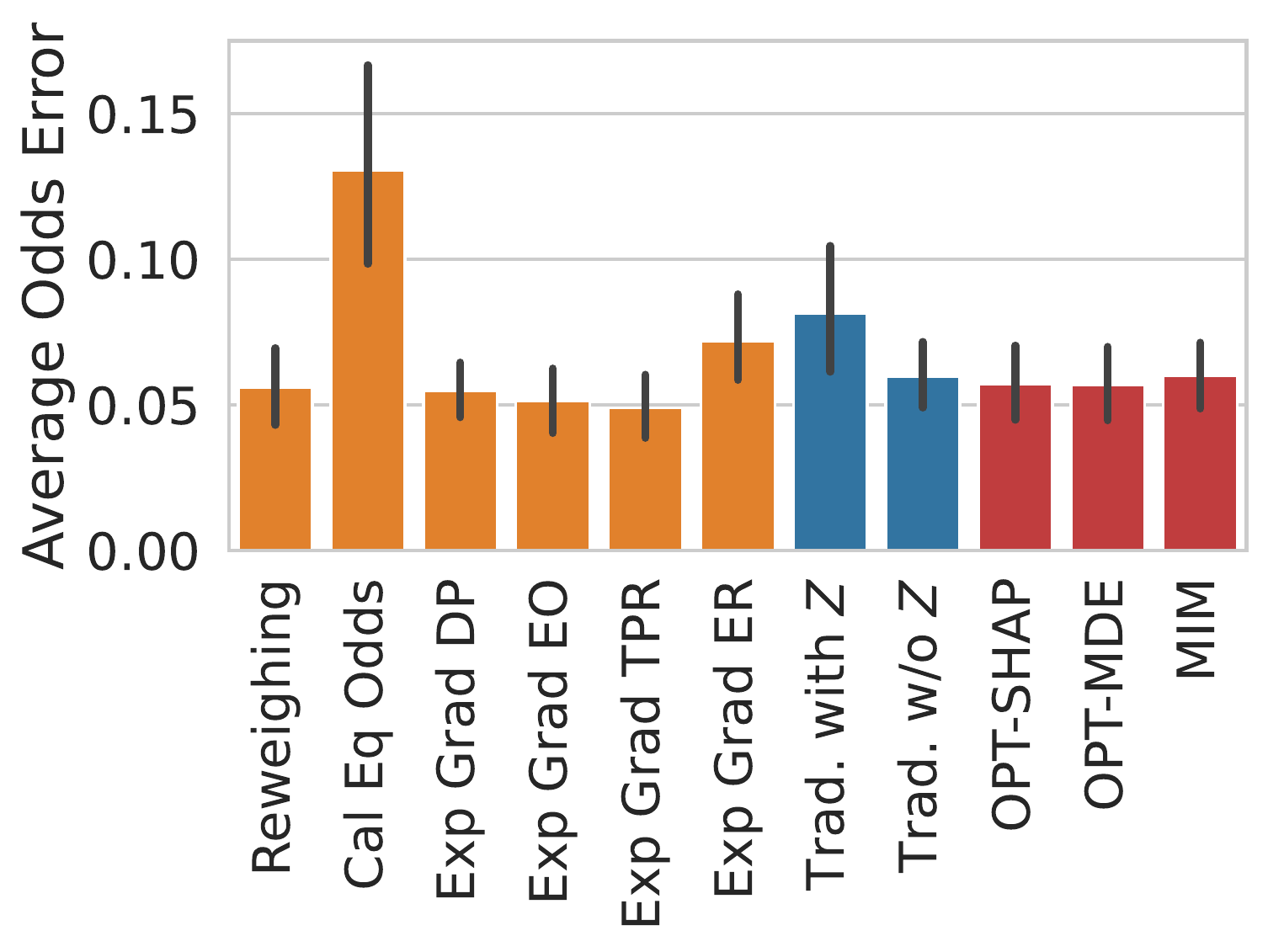}
\end{subfigure}
\begin{subfigure}{.3\linewidth}
  \centering
  \includegraphics[width=\linewidth]{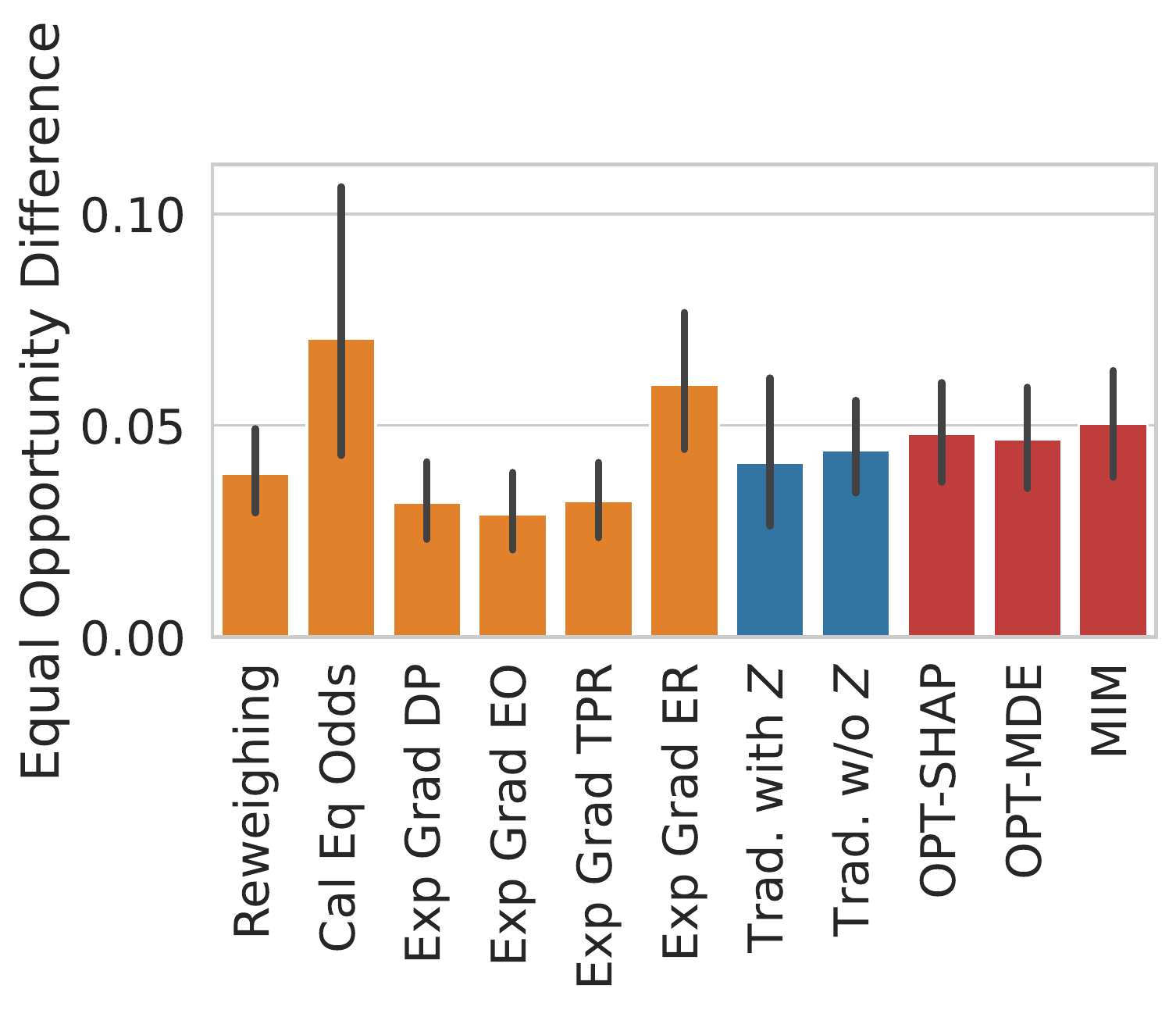}
\end{subfigure}
\caption{Disparate impact, equalized odds, and equal opportunity measures for the evaluated models on the German Credits dataset. Error bars show 95\% confidence intervals.}
\label{fig:aifgermansm}
\end{figure*}

\begin{figure*}[ht!]
\centering
\begin{subfigure}{.3\linewidth}
  \centering
  \includegraphics[width=\linewidth]{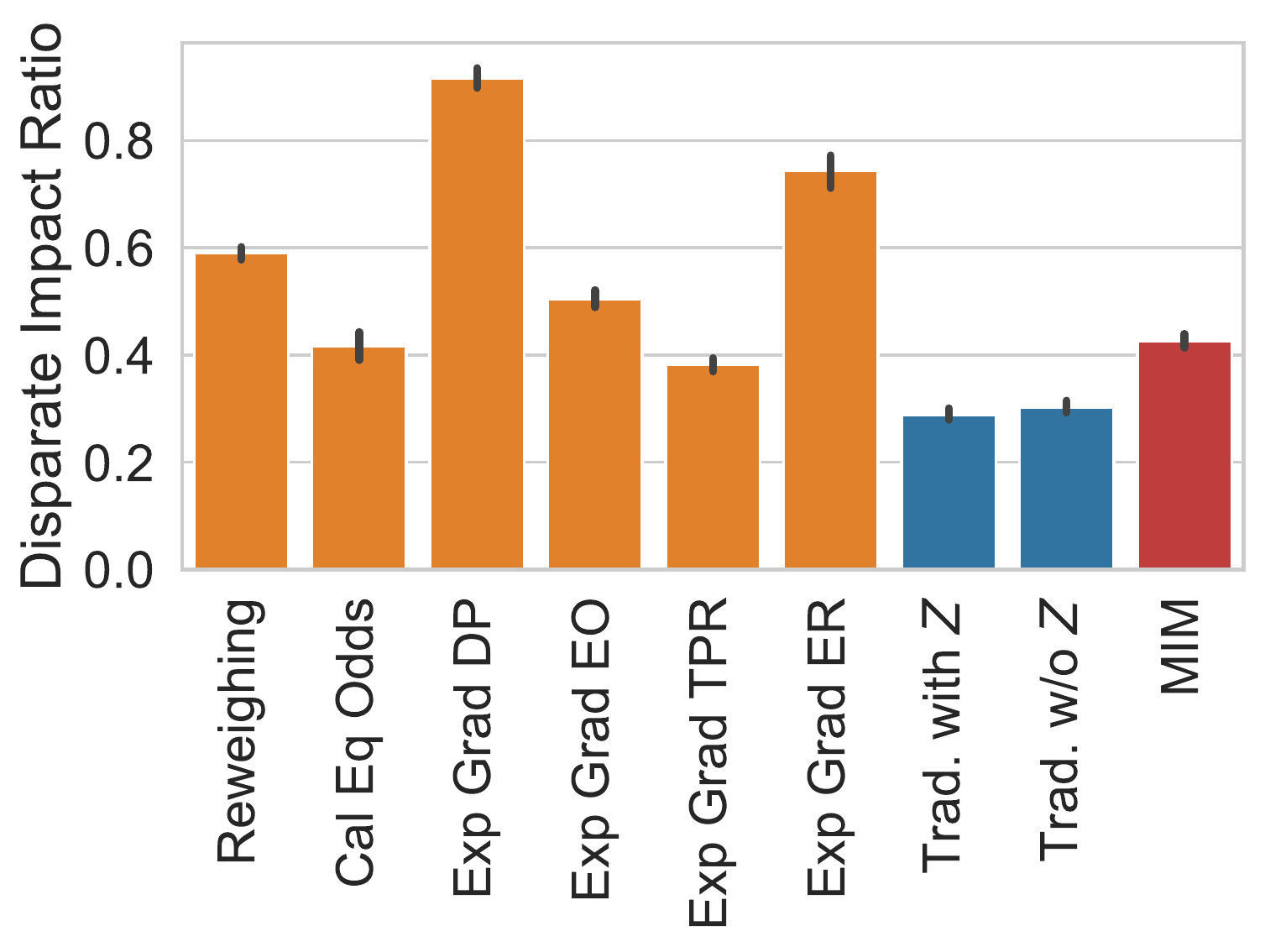}
\end{subfigure}
\begin{subfigure}{.3\linewidth}
  \centering
  \includegraphics[width=\linewidth]{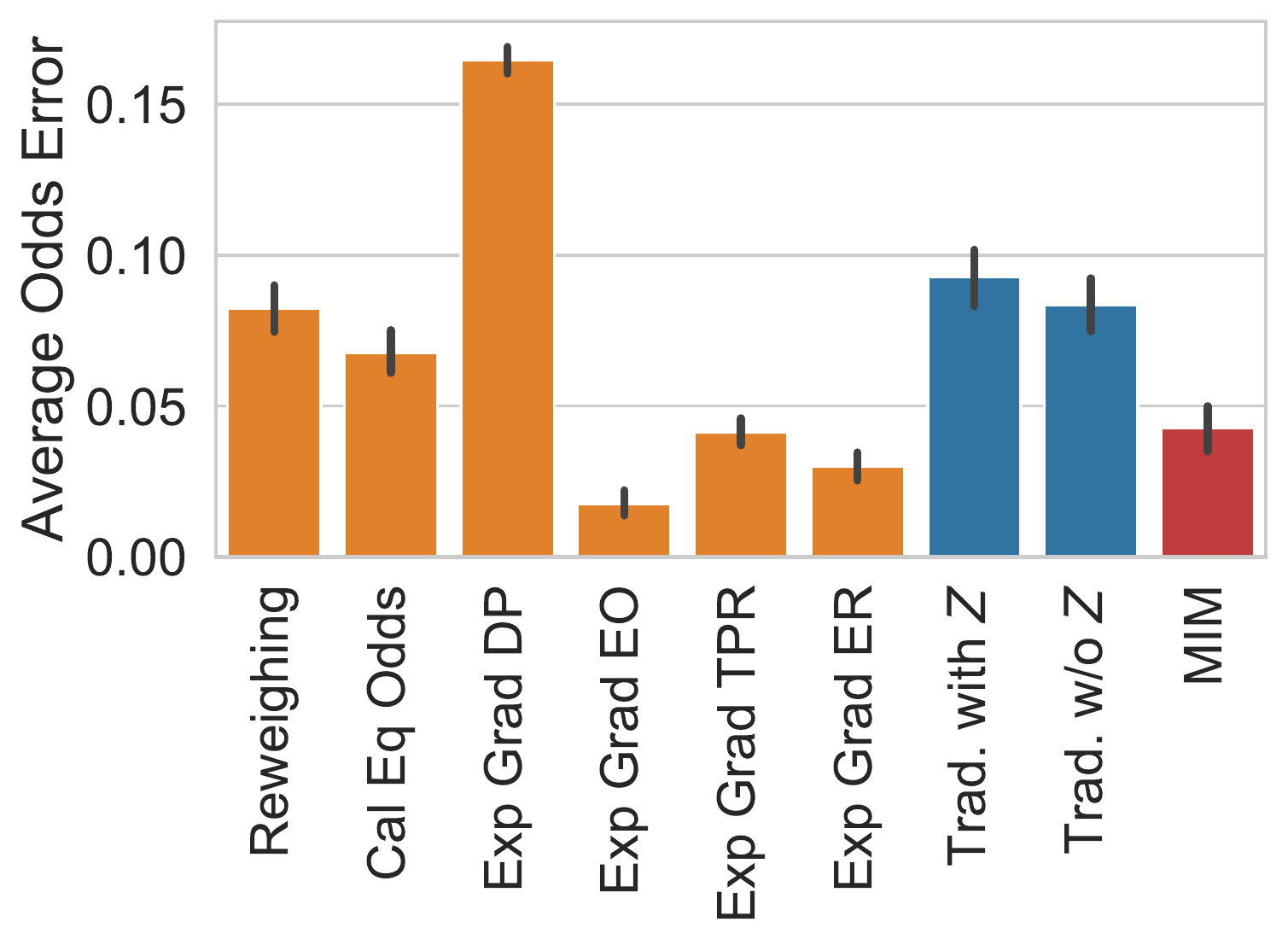}
\end{subfigure}
\begin{subfigure}{.3\linewidth}
  \centering
  \includegraphics[width=\linewidth]{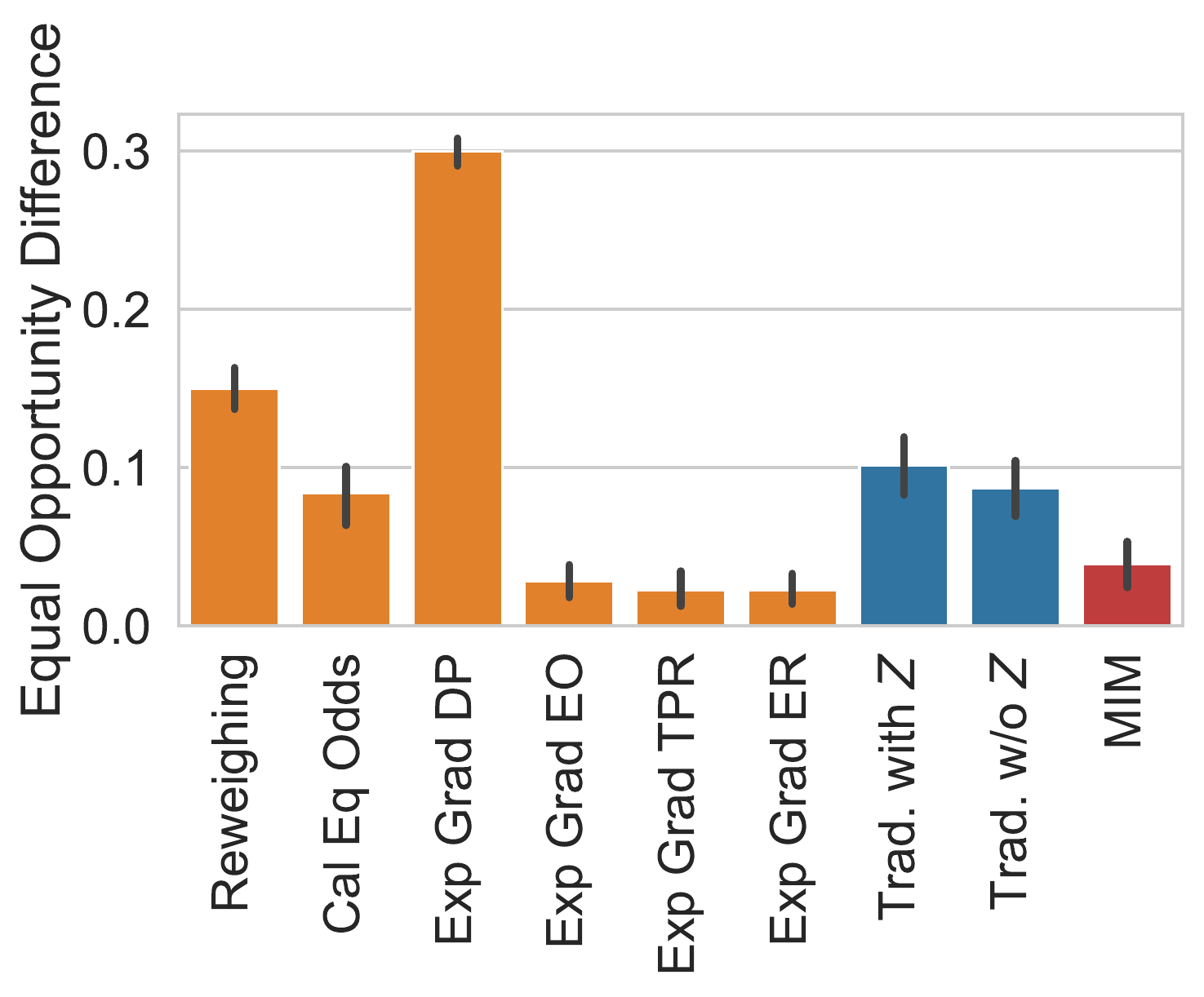}
\end{subfigure}

\caption{Disparate impact, equalized odds, and equal opportunity measures for the evaluated models on the Adult Census Income dataset. Error bars show 95\% confidence intervals.}
\label{fig:aifcensussm}
\end{figure*}

\section*{Appendix C: Marginal direct effect (MDE)}
In addition to the SHAP input influence measure, we measure the MDE ($\E_{\X_i,\X_i''} |\text{MDE}_{\Y}(X_i,X_i'')|$) on the same features for all datasets in Figures \ref{fig:aif111ate}-\ref{fig:aifgermanate}. All results for MDE are qualitative equivalent to that of SHAP for each dataset.

\begin{figure*}
\centering
\begin{subfigure}{.32\textwidth}
  \centering
  \includegraphics[width=\linewidth]{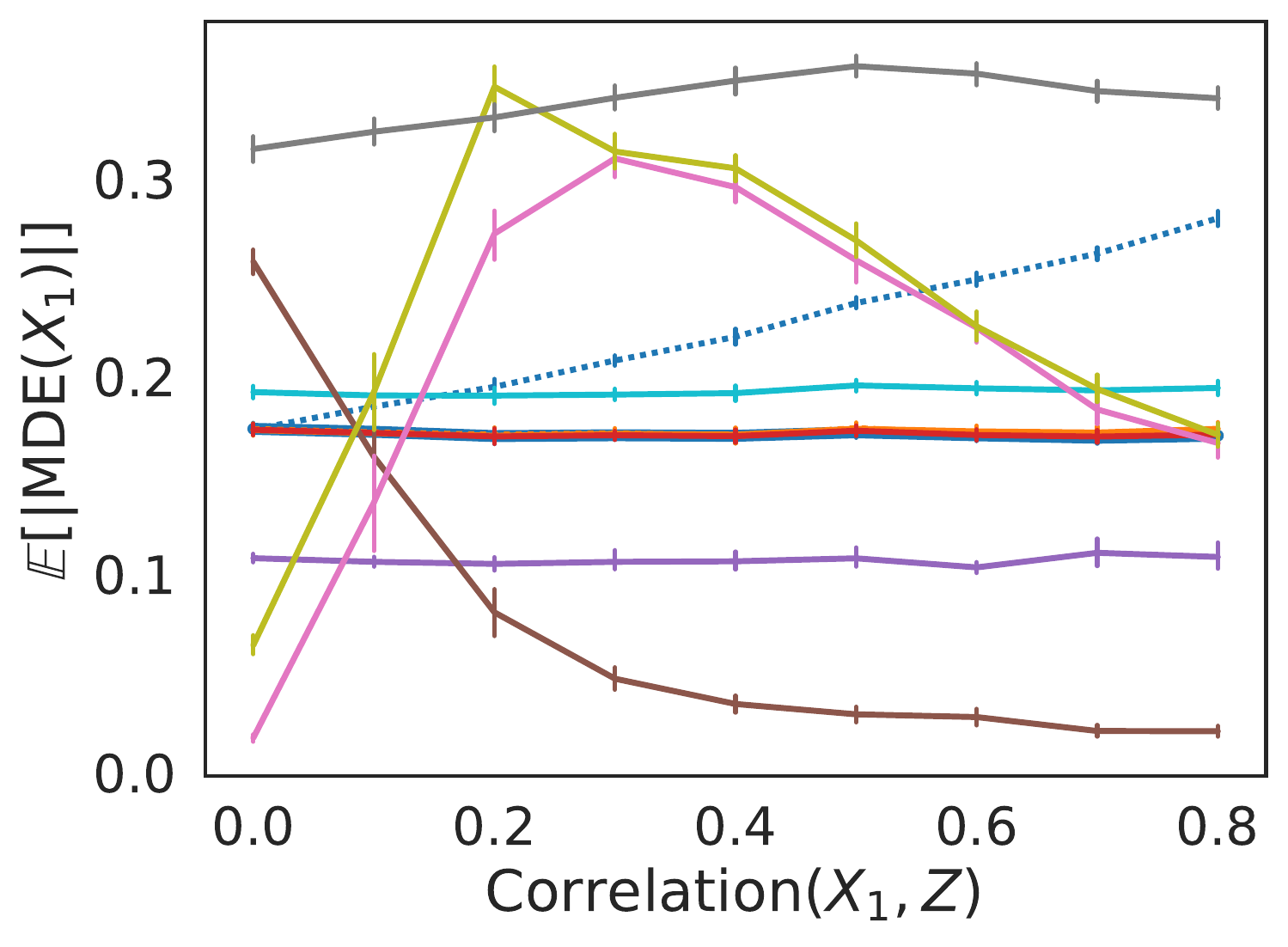}
\end{subfigure}
\begin{subfigure}{.32\textwidth}
  \centering
  \includegraphics[width=\linewidth]{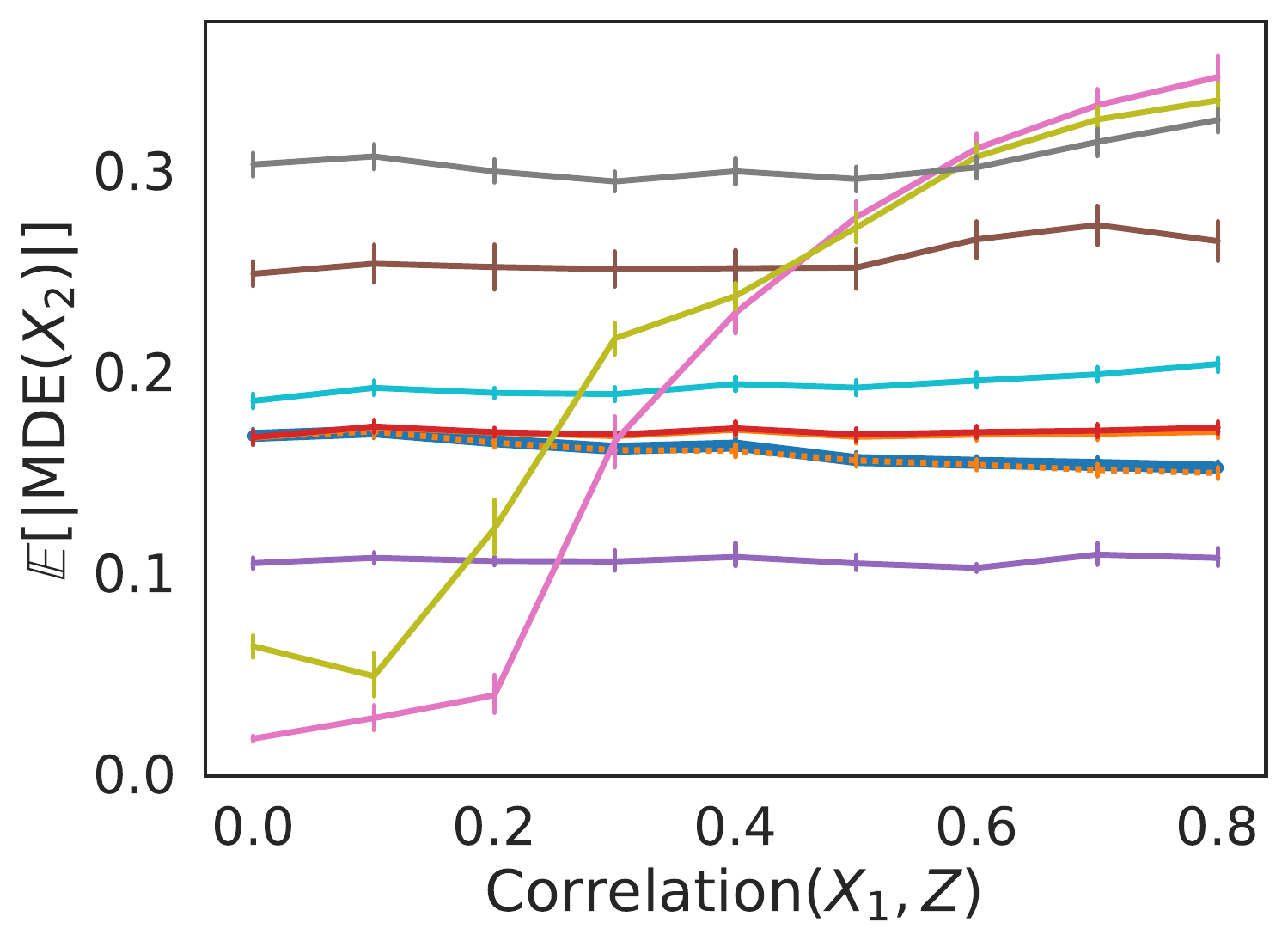}
\end{subfigure}
\begin{subfigure}{.32\textwidth}
  \centering
  \includegraphics[width=\linewidth]{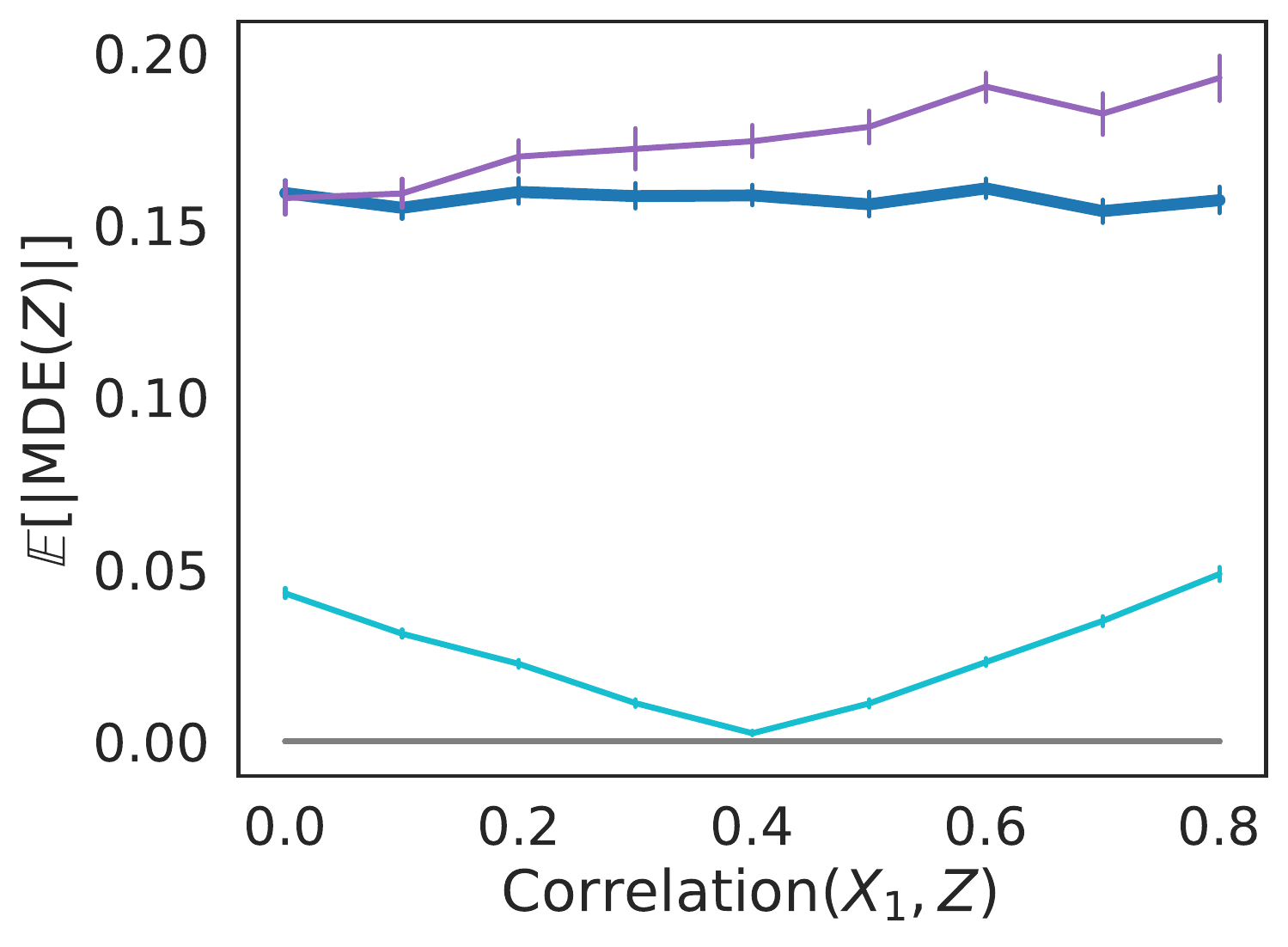}
\end{subfigure}
\begin{subfigure}{\linewidth}
  \centering
  \includegraphics[width=\linewidth]{figs/facct_results/scenario2_final/set3/legend3.pdf}
\end{subfigure}
\caption{Scenario A: MDE as we increase the correlation $r(X_1,Z)$. Error bars show 95\% confidence intervals.}
\label{fig:aif111ate}
\end{figure*}

\begin{figure*}
\centering
\begin{subfigure}{.32\textwidth}
  \centering
  \includegraphics[width=\linewidth]{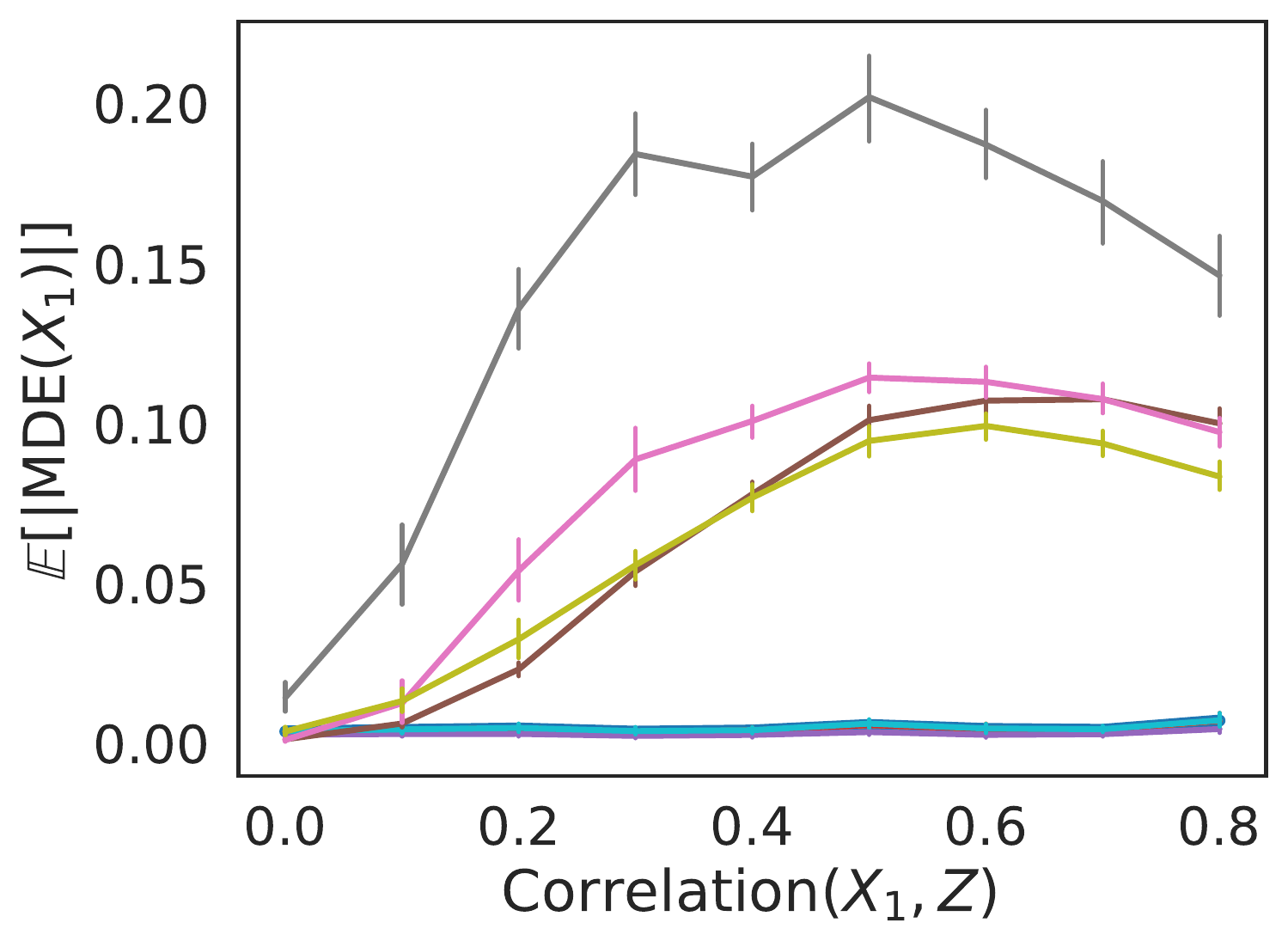}
\end{subfigure}
\begin{subfigure}{.32\textwidth}
  \centering
  \includegraphics[width=\linewidth]{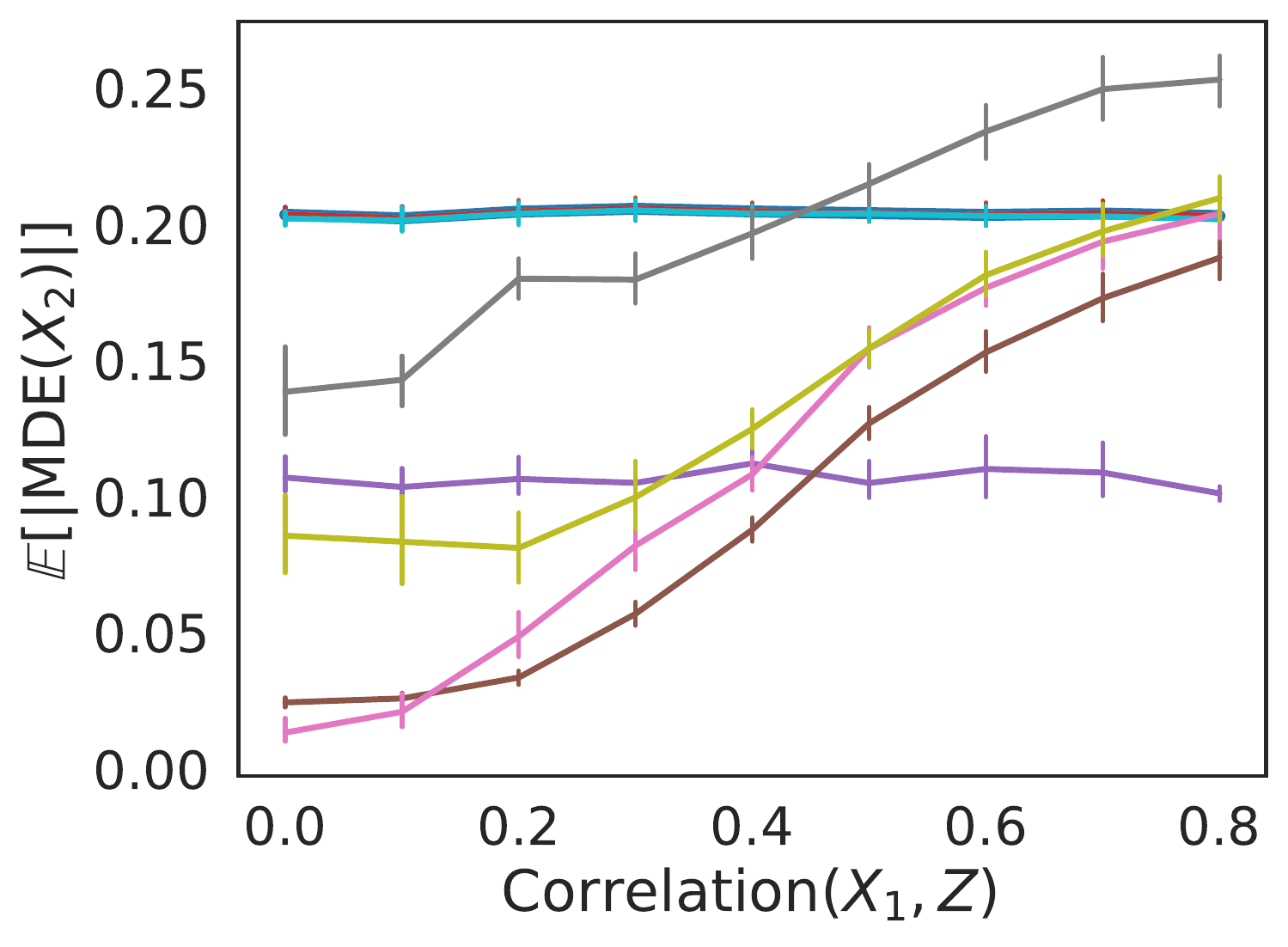}
\end{subfigure}
\begin{subfigure}{.32\textwidth}
  \centering
  \includegraphics[width=\linewidth]{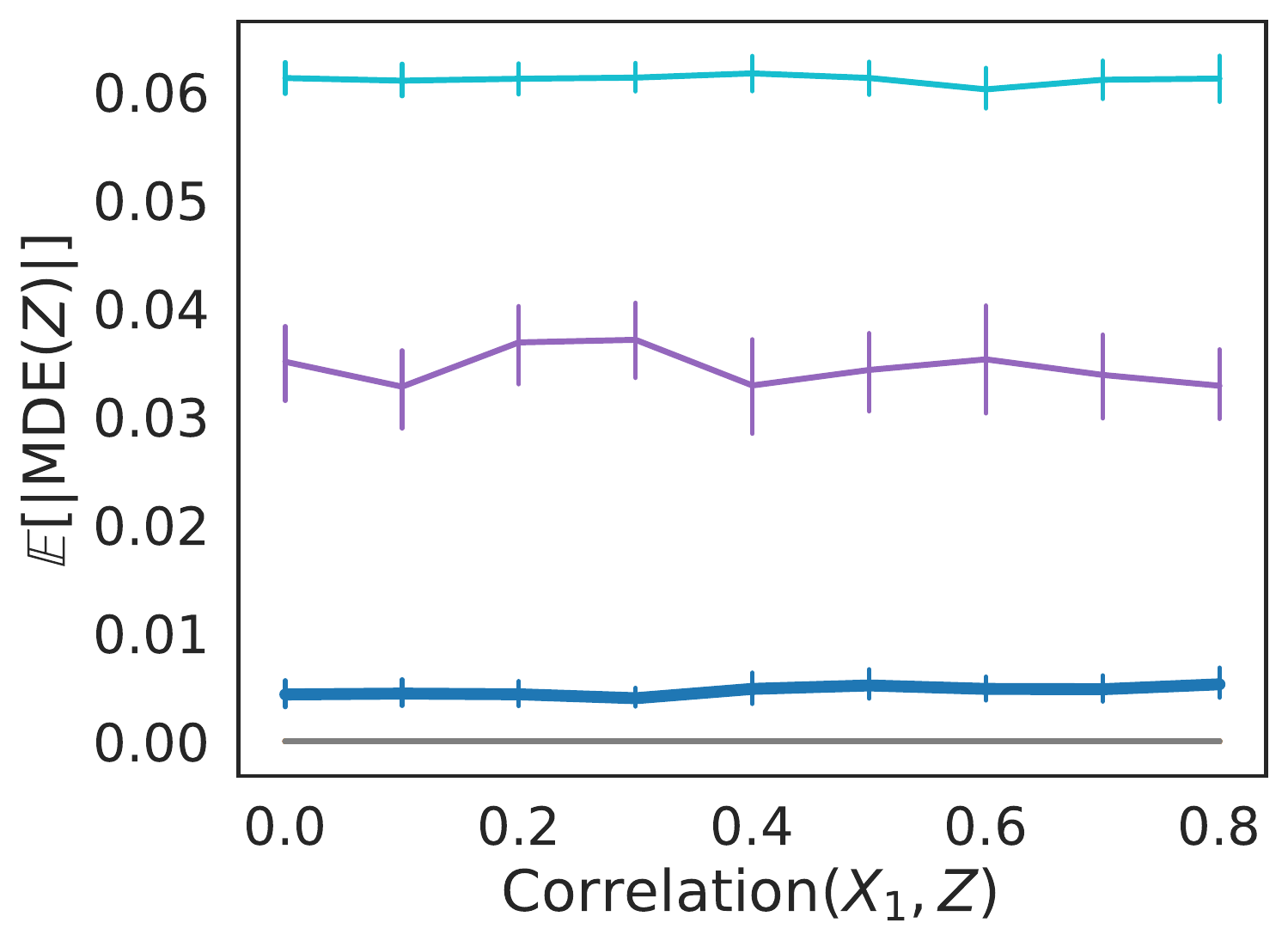}
\end{subfigure}
\begin{subfigure}{\linewidth}
  \centering
  \includegraphics[width=\linewidth]{figs/facct_results/scenario2_final/set3/legend3.pdf}
\end{subfigure}
\caption{Scenario B: MDE as we increase the correlation $r(X_1,Z)$. Error bars show 95\% confidence intervals.}
\label{fig:aif011ate}
\end{figure*}

\begin{figure*}
\centering
\begin{subfigure}{.32\textwidth}
  \centering
  \includegraphics[width=\linewidth]{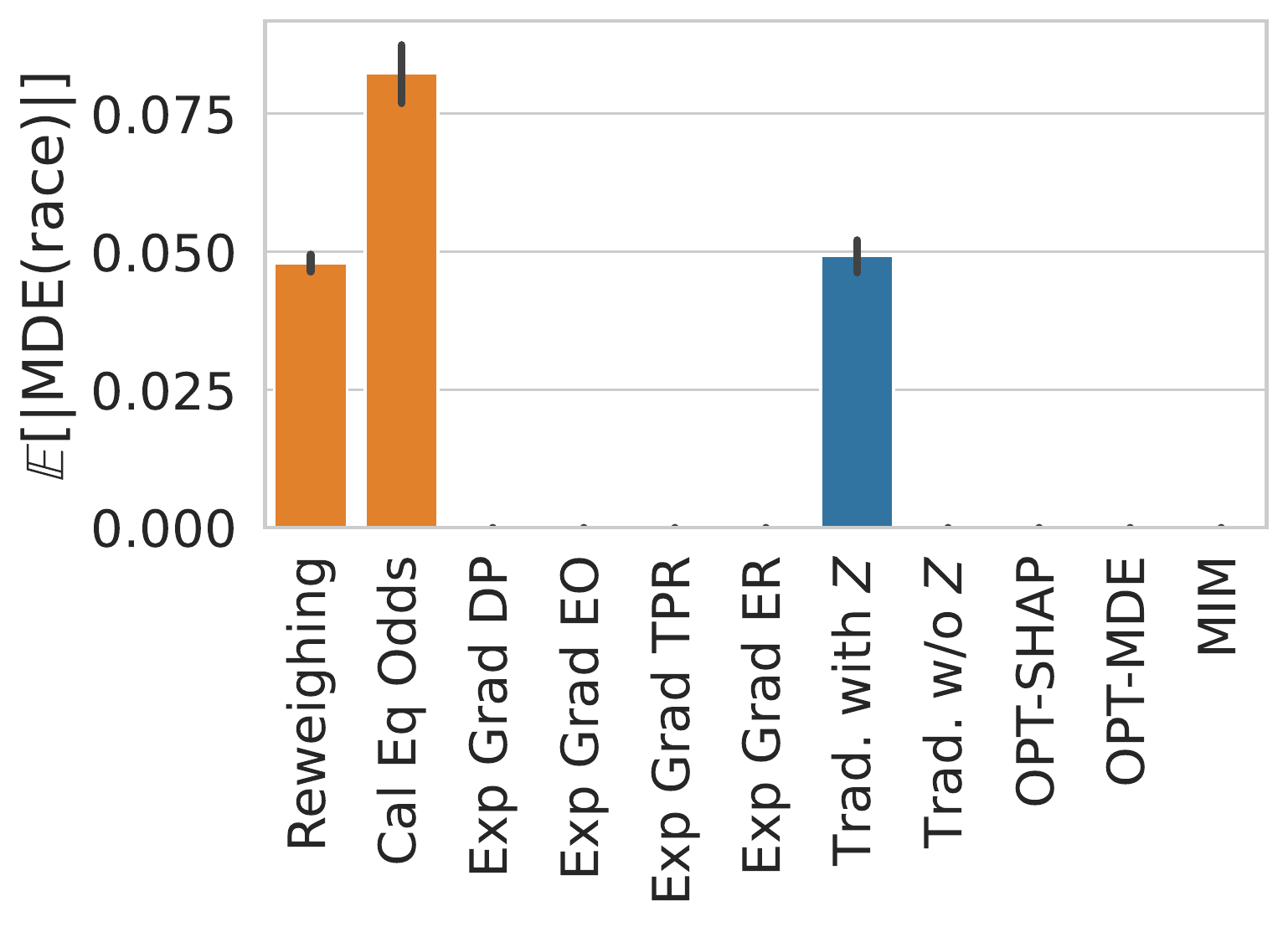}
\end{subfigure}
\begin{subfigure}{.32\textwidth}
  \centering
  \includegraphics[width=\linewidth]{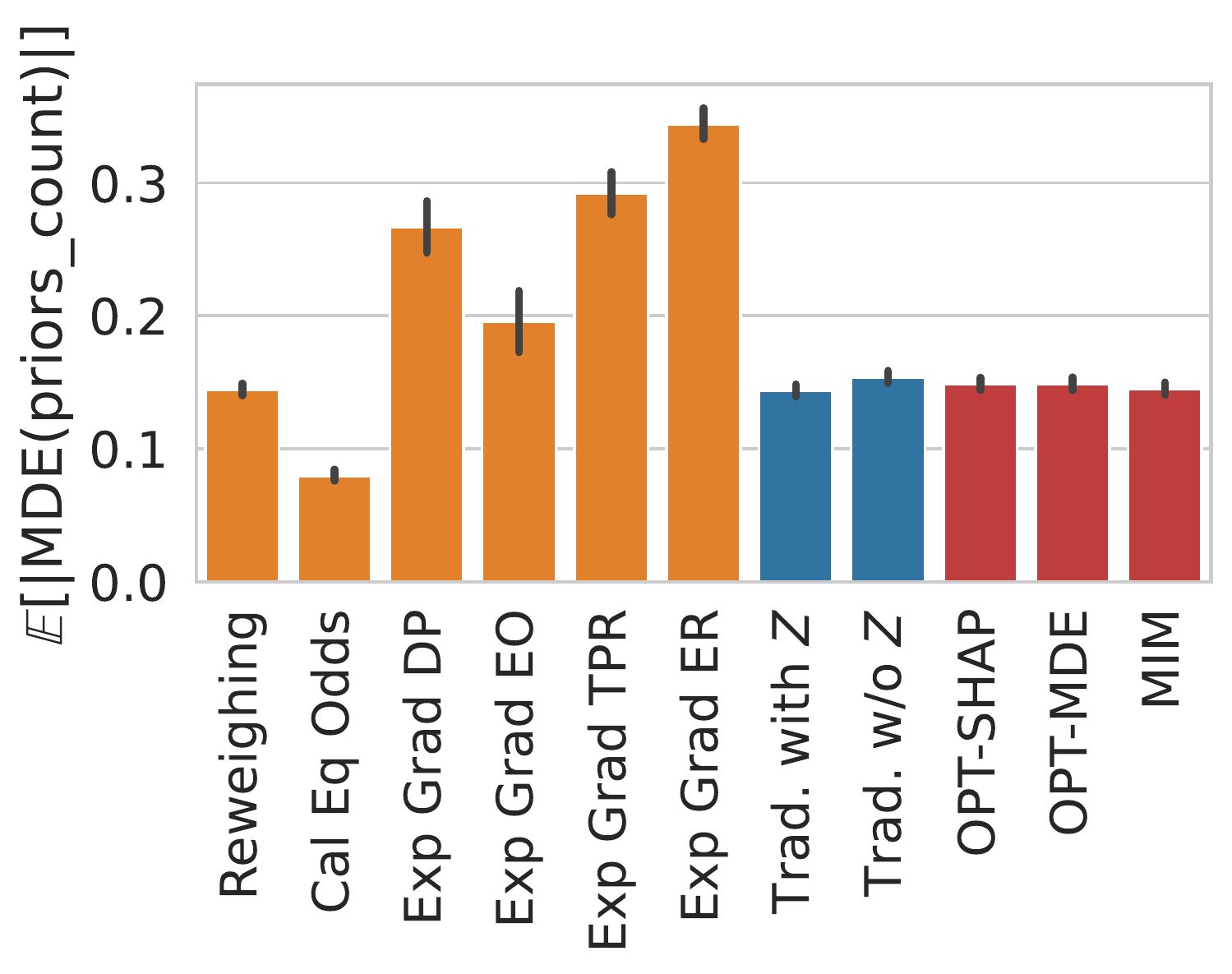}
\end{subfigure}
\begin{subfigure}{.32\textwidth}
  \centering
  \includegraphics[width=\linewidth]{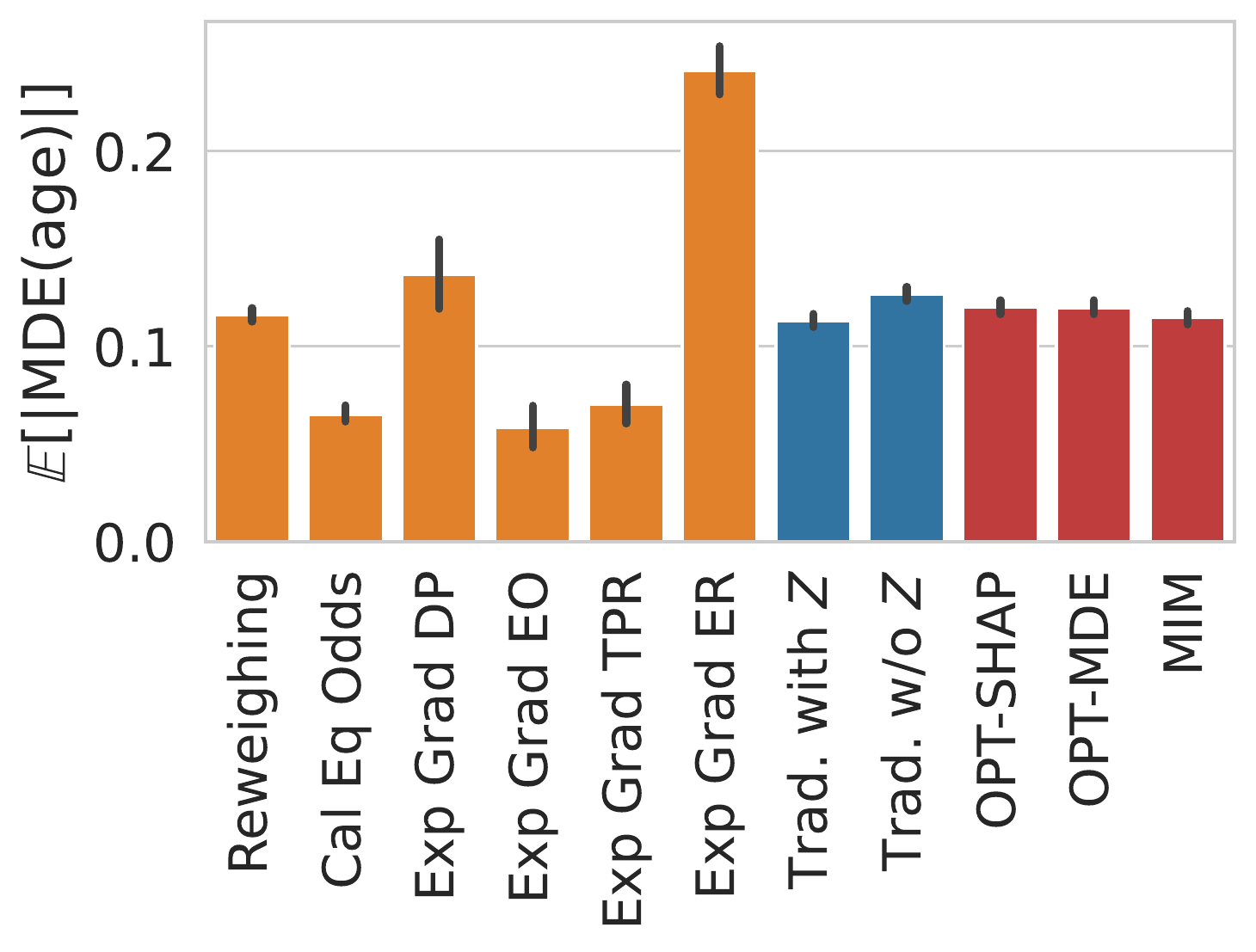}
\end{subfigure}

\caption{MDE for the protected attribute and two features most correlated with it for the evaluated models on the COMPAS dataset. Error bars show 95\% confidence intervals.}
\label{fig:aifcensusate}
\end{figure*}

\begin{figure*}
\centering
\begin{subfigure}{.32\textwidth}
  \centering
  \includegraphics[width=\linewidth]{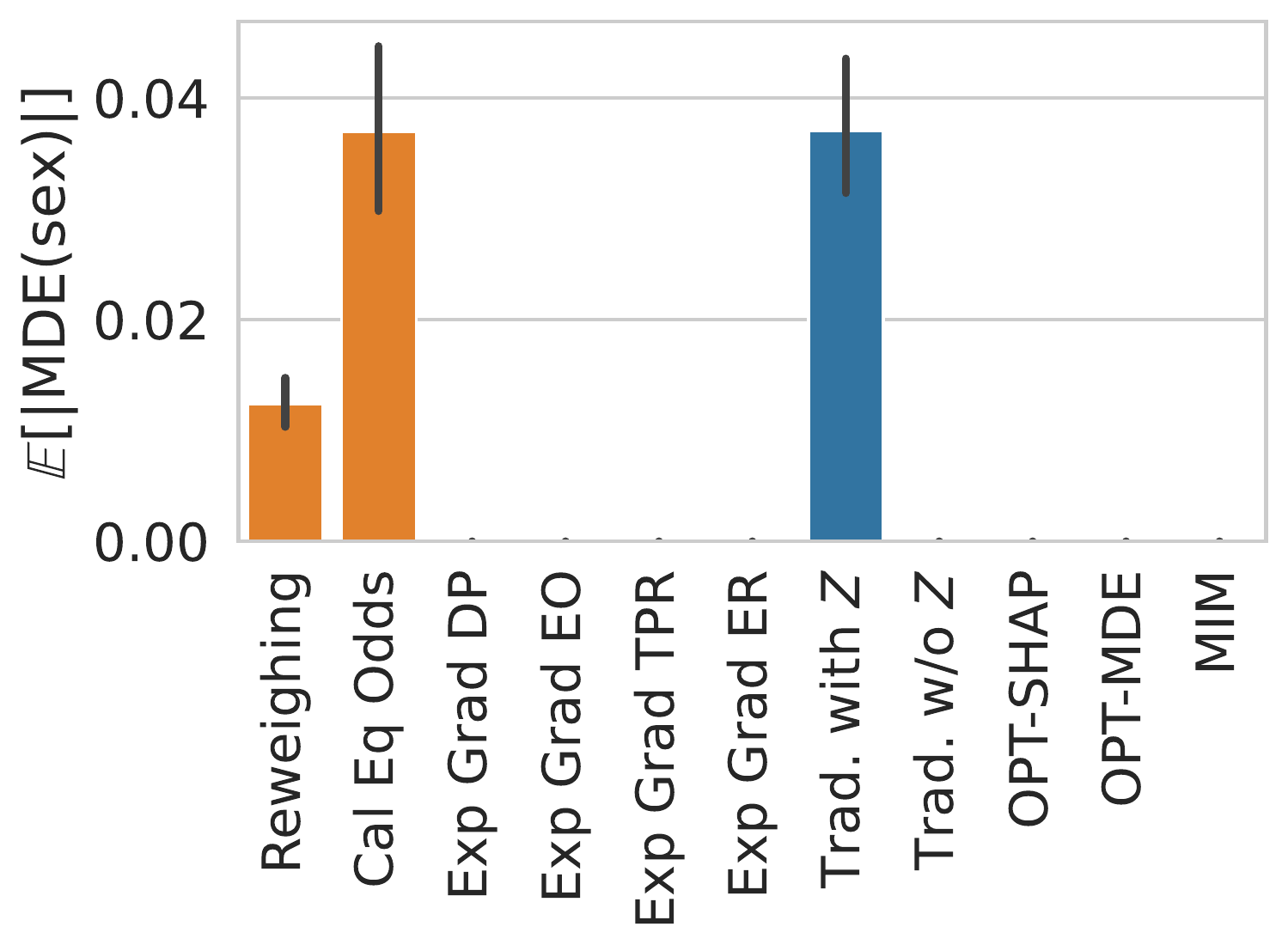}
\end{subfigure}
\begin{subfigure}{.32\textwidth}
  \centering
  \includegraphics[width=\linewidth]{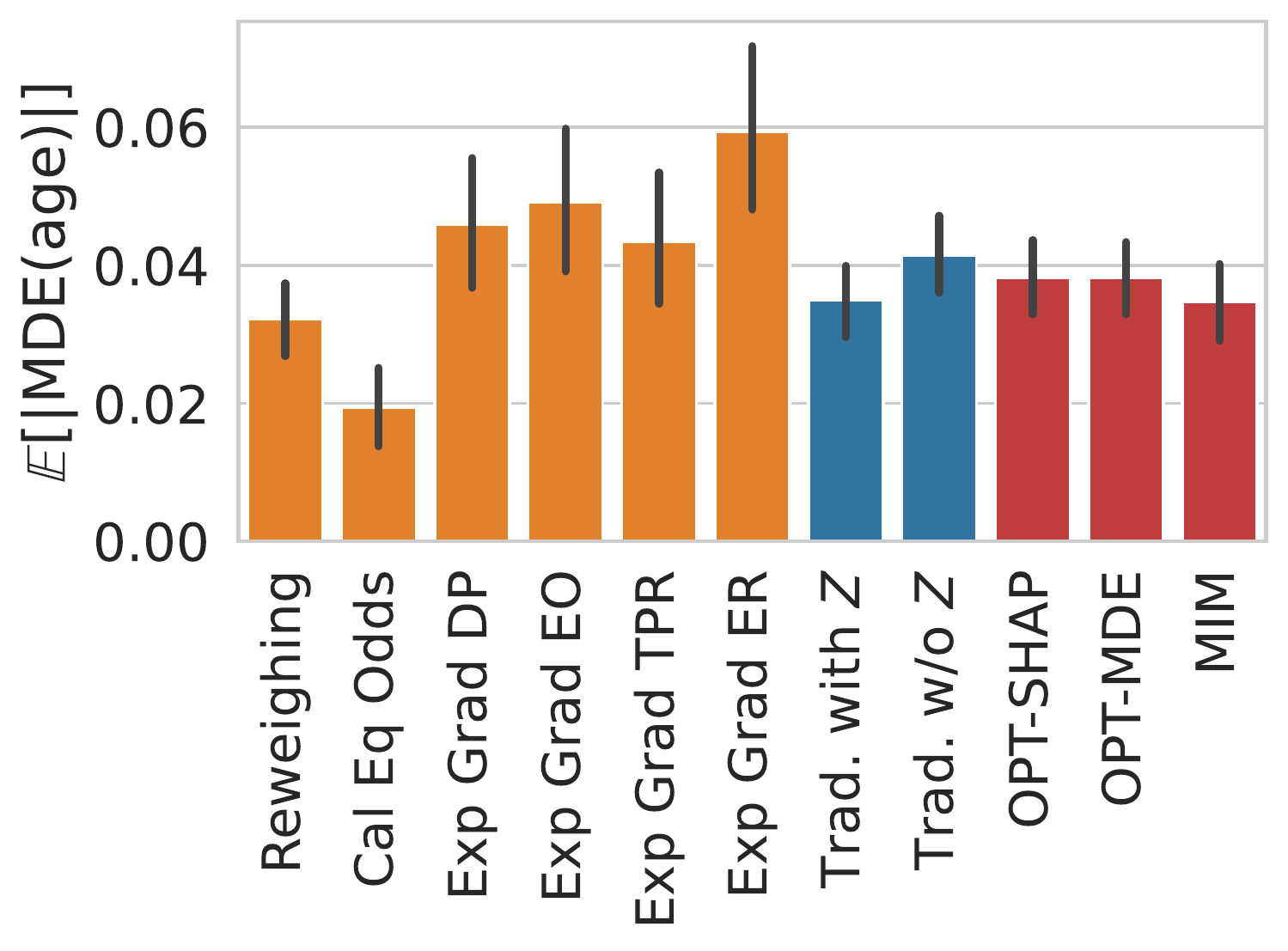}
\end{subfigure}
\begin{subfigure}{.32\textwidth}
  \centering
  \includegraphics[width=\linewidth]{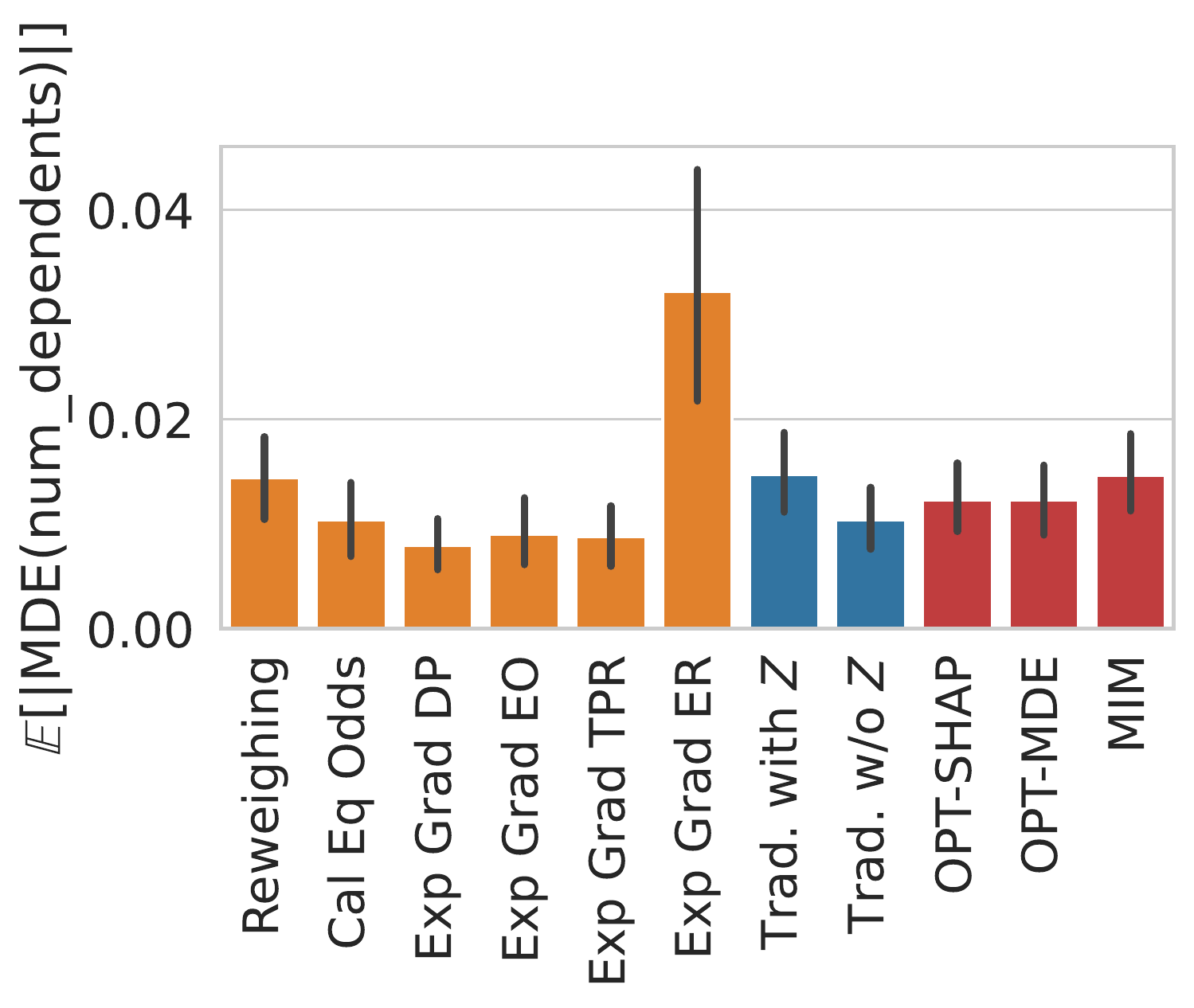}
\end{subfigure}

\caption{MDE for the protected attribute and two features most correlated with it for the evaluated models on the German Credit dataset. Error bars show 95\% confidence intervals.}
\label{fig:aifgermanate}
\end{figure*}

\begin{figure*}
\centering
\begin{subfigure}{.32\textwidth}
  \centering
  \includegraphics[width=\linewidth]{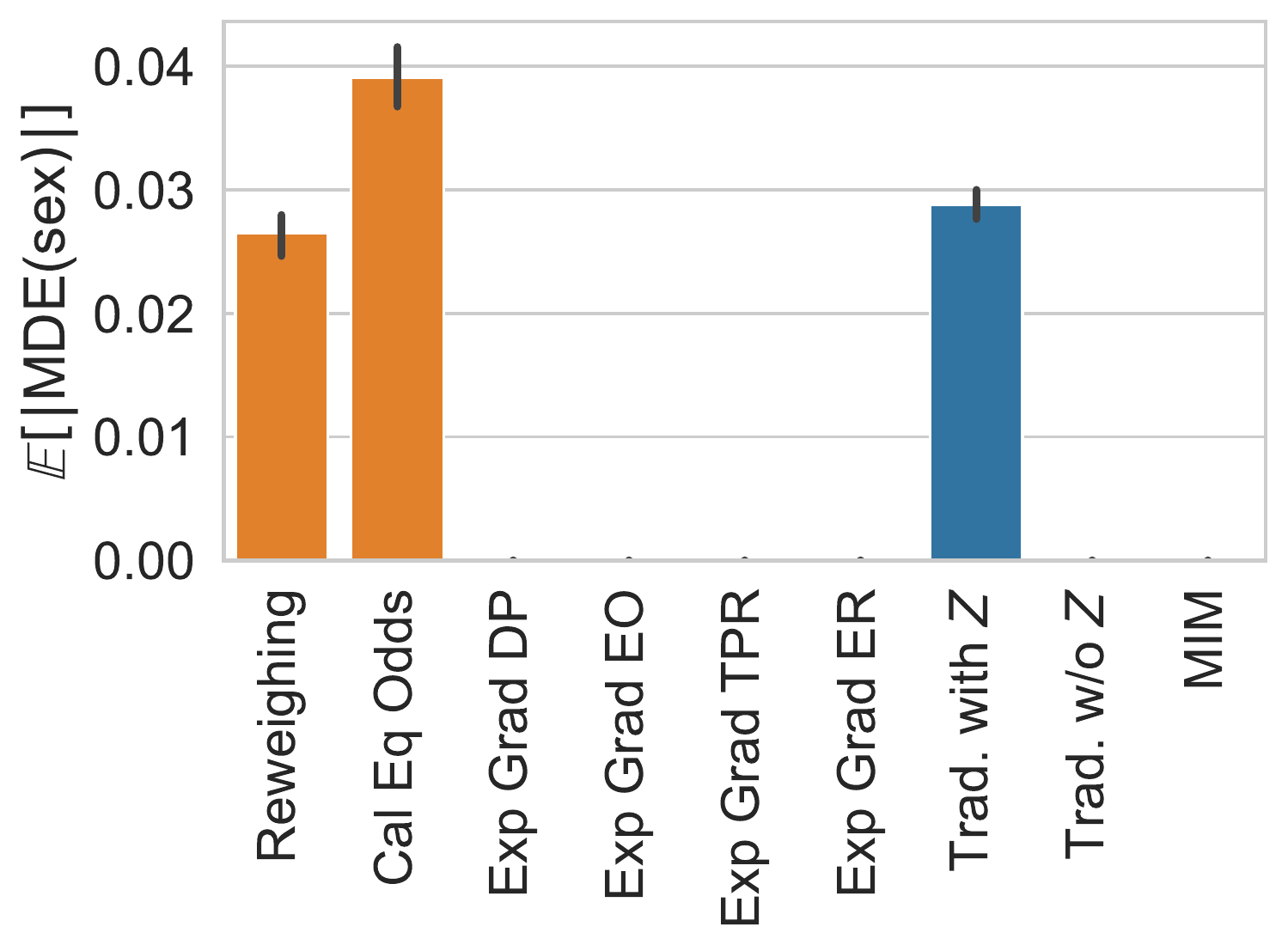}
\end{subfigure}
\begin{subfigure}{.32\textwidth}
  \centering
  \includegraphics[width=\linewidth]{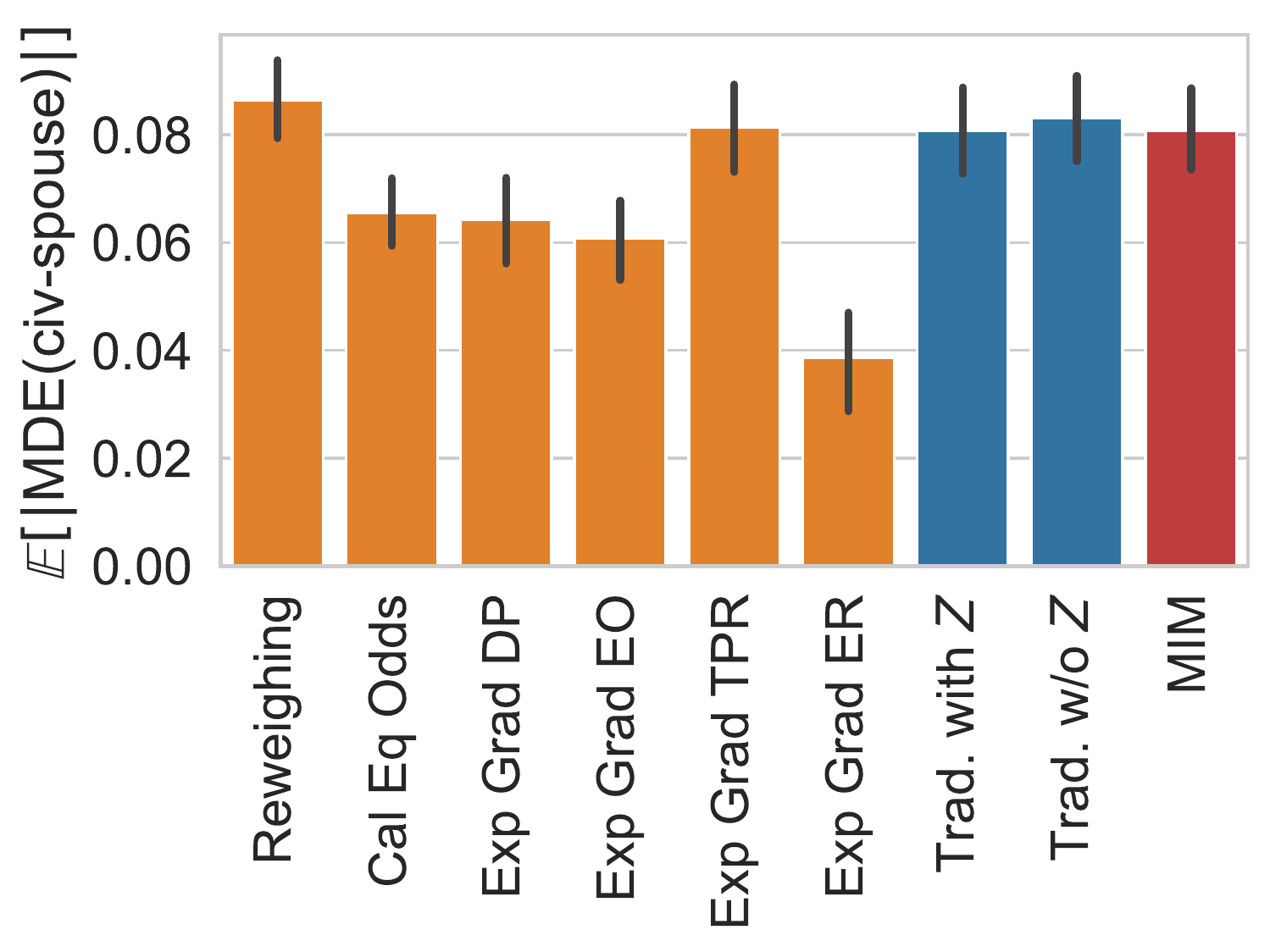}
\end{subfigure}
\begin{subfigure}{.32\textwidth}
  \centering
  \includegraphics[width=\linewidth]{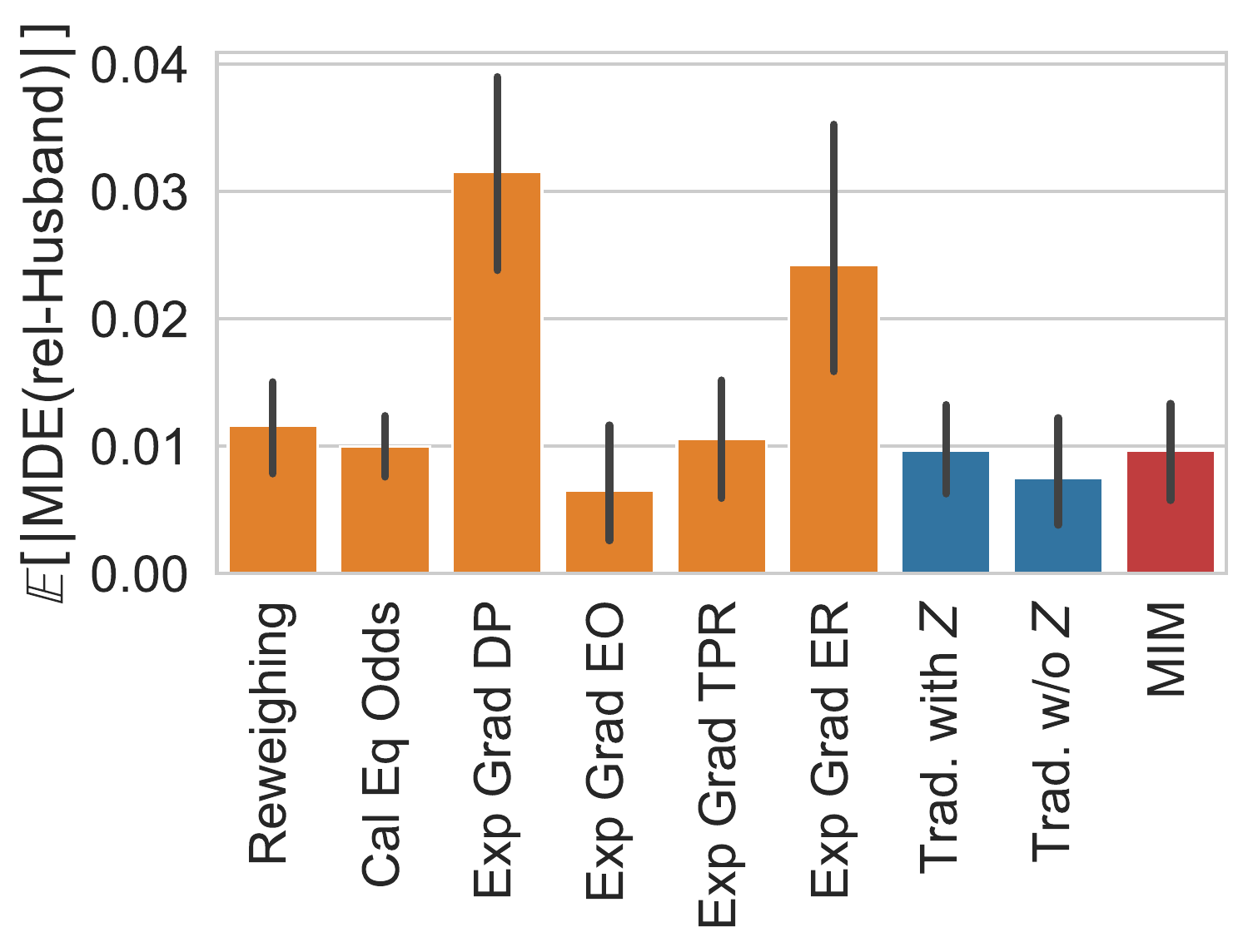}
\end{subfigure}

\caption{MDE for the protected attribute and two features most correlated with it for the evaluated models on the Adult Census Income dataset. Error bars show 95\% confidence intervals.}
\label{fig:aifgermanate}
\end{figure*}
\end{document}